\documentclass{article} 
\usepackage{iclr2027_conference,times}
\usepackage[T1]{fontenc} 

\usepackage{amsmath,amsfonts,bm}

\def\eqref#1{equation~\ref{#1}}

\def\1{\bm{1}}

\DeclareMathAlphabet{\mathsfit}{\encodingdefault}{\sfdefault}{m}{sl}
\SetMathAlphabet{\mathsfit}{bold}{\encodingdefault}{\sfdefault}{bx}{n}

\usepackage{hyperref}
\usepackage{url}

\usepackage{hyperref}       
\usepackage{url}            
\usepackage{booktabs}       
\usepackage{amsfonts}       
\usepackage{nicefrac}       
\usepackage{microtype}      
\usepackage{xcolor}         
\usepackage{amsmath,amssymb,amsthm}
\usepackage{multirow}
\usepackage[table]{xcolor}
\usepackage{graphicx}
\usepackage{booktabs}
\usepackage{multirow}
\usepackage{subcaption}
\usepackage{todonotes}

\usepackage{algorithm}
\usepackage{algpseudocode}

\definecolor{oursgray}{RGB}{240,240,248}
\newtheorem{theorem}{Theorem}
\newtheorem{proposition}[theorem]{Proposition}

\title{Correcting Guided Diffusion Trajectories\\ with Spectral Alignment}

\author{Gihoon Kim and Taesup Kim \\
Graduate School of Data Science, Seoul National University
}

\iclrfinalcopy
\usepackage{etoolbox}
\makeatletter
\patchcmd{\@maketitle}{Published as a conference paper at ICLR 2027}{Preprint}
  {}{\PackageError{arxiv}{Could not replace the ICLR publication header}{}}
\makeatother
\begin{document}

\maketitle

\begin{abstract}
The practical success of conditional image generation hinges on fine-grained differences in condition alignment and visual fidelity. Classifier-free guidance (CFG) is central to this success, but its lack of an explicit criterion makes it difficult to assess whether the guided trajectory is progressing as intended. To address this gap, we show that spectral alignment provides a principled criterion for understanding guidance behavior and improving guided diffusion sampling through adaptive correction. Our analysis identifies the spectra of intermediate states as an indicator of consistency with the expected spectral evolution of the forward process. Based on this observation, we introduce \textit{Spectral Correction Guidance}, a method that corrects deviations from an analytic reference spectrum during sampling. The proposed method is training-free and applicable across diffusion backbones and conditional generation tasks without modifying the underlying model. Experiments demonstrate consistent gains in preference-based metrics over baseline guidance methods in text-to-image generation and improved generation quality over CFG on ImageNet. These improvements persist across a range of guidance scales and with fewer denoising steps. Our analyses and ablations provide insight into guidance behavior and how the proposed method affects generation quality.
\end{abstract}

\section{Introduction} 
\label{sec:intro}

Denoising diffusion models~\citep{sohl2015deep,ho2020denoising,song2020denoising,song2020score,dhariwal2021diffusion} have shown strong generative capabilities across diverse domains~\citep{ho2022video,liu2023audioldm,chi2025diffusion}, including image synthesis~\citep{nichol2021improved,karras2022elucidating,peebles2023scalable}. They exhibit stable convergence during training while modeling the reverse of a forward noising process that perturbs clean data toward a noise distribution. The learned reverse process captures rich data distributions, enabling sample generation through denoising from high to low noise levels. Building on this capability, diffusion models support open-ended conditional generation tasks such as text-to-image synthesis~\citep{rombach2022high,saharia2022photorealistic,podell2023sdxl,chen2023pixart,chen2024pixart,xie2024sana,esser2024scaling}.

Classifier-free guidance (CFG)~\citep{ho2022classifier} has become a standard technique for steering sampling toward the target condition in diffusion models. As illustrated in Figures~\ref{fig:teaser_cond} and~\ref{fig:teaser_cfg}, conditional sampling without guidance often fails to capture the target condition, whereas CFG improves both condition alignment and perceptual quality. These benefits depend on how guidance is applied along the denoising trajectory~\citep{hong2023improving,karras2024guiding,sadat2024no}. However, standard CFG uses a fixed guidance scale, limiting its ability to provide appropriate guidance at each denoising step~\citep{kynkaanniemi2024applying}.

Recent guidance methods aim to enhance the benefits of CFG while mitigating its adverse effects. They adjust the guidance strength across timesteps~\citep{kynkaanniemi2024applying, malarz2025classifier, jin2025stage}, modify the direction of guidance with geometric constraints~\citep{sadat2024eliminating, chung2024cfg++, kwon2025tcfg, jin2025angle}, or modulate guidance within predefined frequency bands~\citep{zhang2024frecas,sadat2025guidance,song2026rethinking}. However, these benefits may not extend beyond the settings these methods are designed to address, such as high-guidance regimes. In Section~\ref{sec:t2i_generation}, experimental results show that these methods yield limited gains across a range of generation settings, where vanilla CFG remains a strong baseline.

In this paper, we argue that these limitations arise from reliance on constraints or suppression rules without a criterion to guide correction along the denoising trajectory. However, defining such a criterion is challenging in the conventional spatial domain, since coordinate-based representations encode local structures and reflect sample-specific characteristics. To address this difficulty, we identify alignment between the power spectra of guided intermediate states and the expected spectra under the forward process as a key indicator of effective guidance. This provides a complementary perspective, capturing global patterns of diffusion dynamics along the denoising trajectory.
    


Building on this observation, we propose \textit{Spectral Correction Guidance}, a training-free method for correcting guided denoising trajectories according to the spectral dynamics of the diffusion process. We derive a criterion from the power-law spectral statistics of natural images~\citep{ruderman1997origins}, yielding a distributional reference that is not tied to individual samples or specific failure patterns. Based on this criterion, our method corrects deviations from the expected spectral behavior in intermediate states. The proposed correction enables adaptive adjustments across frequencies and timesteps without relying on a prescribed correction pattern.


We further validate our method through extensive experiments on diverse conditional image generation tasks. In text-to-image generation, the proposed correction improves preference-based metrics across multiple backbones and guidance settings, indicating better text alignment and visual quality. In class-conditional generation, our method achieves lower FID~\citep{heusel2017gans} and higher Inception Score~\citep{salimans2016improved} than CFG, even with fewer sampling steps. These results demonstrate that correcting spectral deviations during sampling can improve generation quality without additional training. Ablation studies examine how correction strength and the timestep range over which correction is applied affect generation outputs. We also discuss how the proposed method relates to existing guidance approaches from a spectral perspective.


\begin{figure}[t]
    \centering
    \begin{subfigure}[b]{0.28\linewidth}
        \centering
        \caption{No CFG}
        \includegraphics[width=1.0\linewidth]{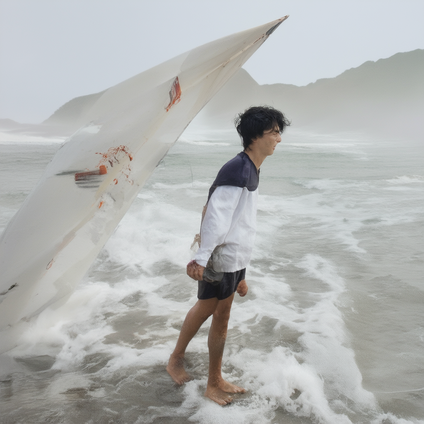}
        \vspace{0.15em}
        {\scriptsize Spectral distance = 0.014}
        \label{fig:teaser_cond}
    \end{subfigure}
    \hspace{0.015\linewidth}
    \begin{subfigure}[b]{0.28\linewidth}
        \centering
        \caption{CFG}
        \includegraphics[width=1.0\linewidth]{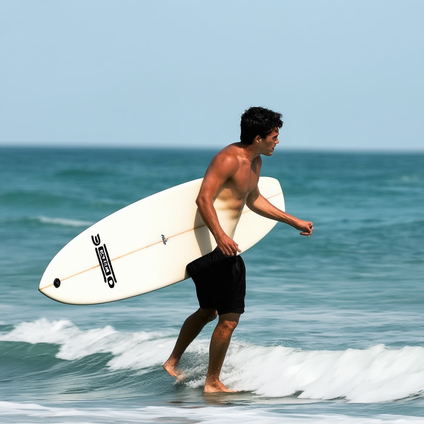}
        \vspace{0.15em}
        {\scriptsize Spectral distance = 0.006}
        \label{fig:teaser_cfg}
    \end{subfigure}
    \hspace{0.015\linewidth}
    \begin{subfigure}[b]{0.28\linewidth}
        \centering
        \caption{Ours}
        \includegraphics[width=1.0\linewidth]{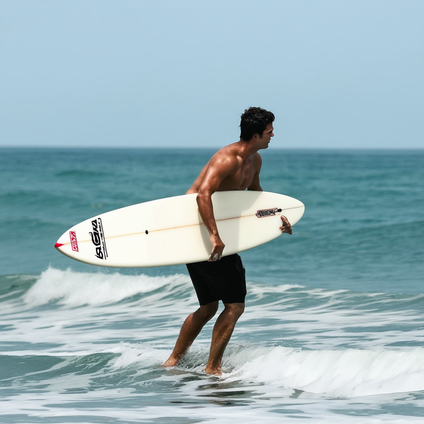}
        \vspace{0.15em}
        {\scriptsize Spectral distance = 0.002}
        \label{fig:teaser_ours}
    \end{subfigure}
    \vspace{-4pt}
    \caption{For the prompt ``\texttt{a male surfer in black shorts carrying a white board}'' on SD3.5~\citep{esser2024scaling}, sampling without CFG fails to form a reliable conditional image.
    CFG produces a more plausible image but introduces an artifact where the surfboard overlaps with the person.
    Our method better preserves the intended composition and improves visual coherence. It also yields the lowest spectral distance, which we analyze in Section~\ref{sec:spectral_dynamics}.}
    \label{fig:teaser}
    \vspace{-8pt}
\end{figure}

\vspace{-4pt}
\section{Background}
\label{sec:background}
\vspace{-4pt}
\paragraph{Denoising Diffusion Models.}



Diffusion models define a predefined forward noising process $q$ and a learned reverse generative process $p_\theta$~\citep{luo2022understanding, lai2025principles}.
The forward process is given by
$
q(z_{1:T}\mid z_0) := \prod_{t=1}^{T} q(z_t\mid z_{t-1}),
$
where $z$ denotes either a data sample or its latent representation.
Under an appropriate noise schedule, the state at time $t$ can be sampled from $q(z_t\mid z_0)$ as
\begin{equation}
z_t = \alpha_t z_0 + \sigma_t \epsilon,
\quad \epsilon \sim \mathcal{N}(0, I),
\label{eq:zt_def}
\end{equation}
where $\alpha_t$ and $\sigma_t$ are time-dependent scaling factors. The forward process in Eq.~\ref{eq:zt_def} induces the marginal distribution
\begin{equation}
q(z_t) = \int q(z_t \mid z_0)\, q(z_0)\, dz_0,
\label{eq:forward_marginal}
\end{equation}
where $q(z_0)$ denotes the data distribution. The learned reverse process $p_\theta$ parameterizes transitions $p_\theta(z_t\mid z_{t+1})$, which model the distribution of the less noisy state $z_t$ given the noisier state $z_{t+1}$. Accordingly, the reverse transition defines the marginal at time $t$ as
\begin{equation}
p_\theta(z_t) = \int p_\theta(z_t\mid z_{t+1})\, p_\theta(z_{t+1})\, dz_{t+1},
\label{eq}
\end{equation}
where $p_\theta(z_t)$ denotes the distribution induced by reverse sampling at time $t$. By matching the reverse marginals $p_\theta(z_t)$ to the corresponding forward marginals $q(z_t)$, the diffusion model approximates the data distribution $q(z_0)$. 



In practice, the generative process can be parameterized in different ways,
including noise prediction~\citep{ho2020denoising}, velocity prediction~\citep{liu2022flow,lipman2022flow}, and score-based modeling~\citep{song2019generative,song2020score}.
Across these variants, the models share the same objective of reversing the
noising process, as formalized by the above notation.

\paragraph{Classifier-free Guidance (CFG).}
CFG is introduced as an alternative to classifier guidance~\citep{dhariwal2021diffusion}, which uses the classifier gradient $\nabla_{z_t}\log p(c \mid z_t)$ to steer sampling toward the target condition $c$.
With a noise prediction model $\epsilon_\theta$, classifier guidance modifies the denoising prediction as
\begin{equation}
\epsilon_{\mathrm{CG}}(z_t,t,c)
=
\epsilon_\theta(z_t,t)
-
s \, \sigma_t \nabla_{z_t}\log p(c \mid z_t),
\label{eq:classifier_guidance}
\end{equation}
where $s$ is the classifier guidance scale.
From Bayes' rule, the classifier gradient can be expressed as
$\nabla_{z_t}\log p(c \mid z_t) = \nabla_{z_t}\log p(z_t \mid c) - \nabla_{z_t}\log p(z_t)$.
Substituting this identity into Eq.~\ref{eq:classifier_guidance}
and using Tweedie's formula~\citep{efron2011tweedie}
to relate score terms to noise predictions yields the standard CFG form
\begin{equation}
\epsilon_{\mathrm{CFG}}(z_t, t, c)
=
\epsilon_\theta(z_t, t, \emptyset)
+
w \bigl(\epsilon_\theta(z_t, t, c)
- \epsilon_\theta(z_t, t, \emptyset)\bigr),
\label{eq:cfg}
\end{equation}
where $w$ denotes the guidance scale and $\emptyset$ denotes the null condition.



Eq.~\ref{eq:cfg} has been widely adopted in modern diffusion models
due to its effectiveness in improving conditional generation quality.
Nevertheless, determining appropriate guidance remains challenging in practice,
as it remains difficult to assess whether the magnitude $w$ and the direction
$\epsilon_\theta(z_t,t,c)-\epsilon_\theta(z_t,t,\emptyset)$
of the guidance term are appropriate along the denoising trajectory.

\begin{figure}[t]
    \centering
    \begin{subfigure}[b]{0.30\linewidth}
        \centering
        \captionsetup{labelformat=empty}
        \caption{\quad (a) Spectral Distance}
        \label{fig:cfg_spectral_distance}
        \includegraphics[width=0.9\linewidth]{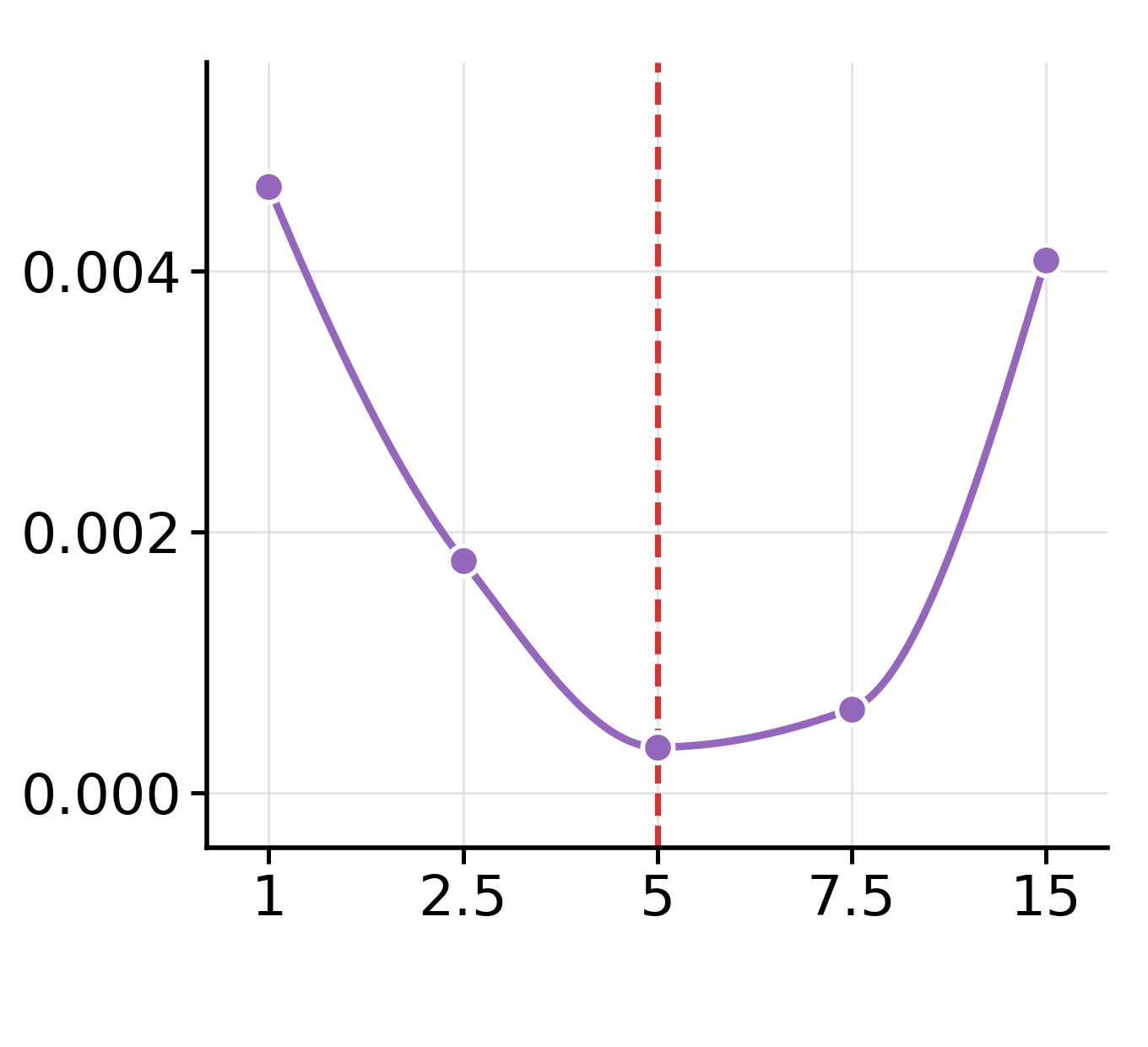}
    \end{subfigure}%
    \hfill
    \begin{subfigure}[b]{0.30\linewidth}
        \centering
        \captionsetup{labelformat=empty}
        \caption{\quad (b) FID}
        \label{fig:cfg_fid}
        \includegraphics[width=0.9\linewidth]{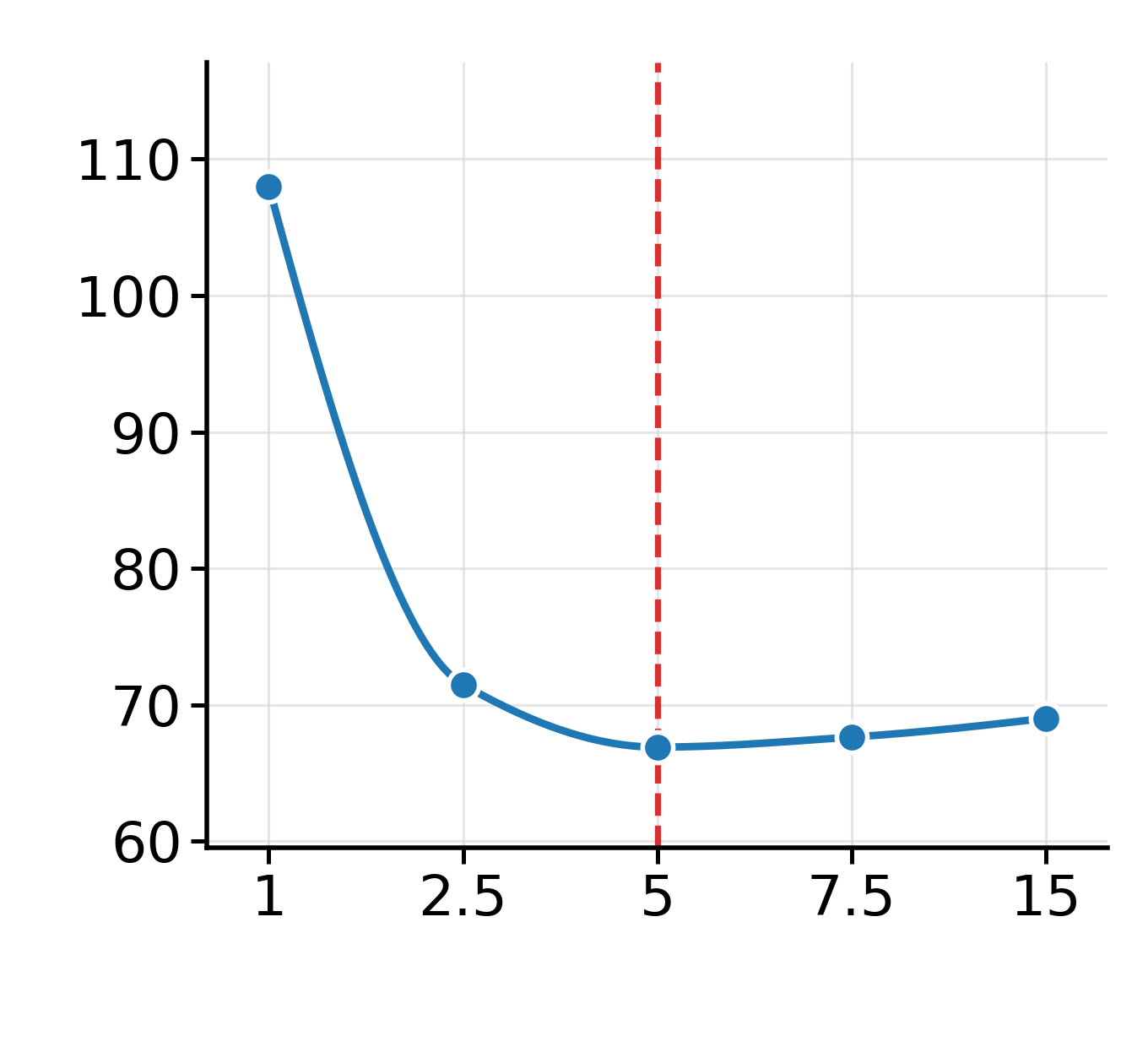}
    \end{subfigure}%
    \hfill
    \begin{subfigure}[b]{0.30\linewidth}
        \centering
        \captionsetup{labelformat=empty}
        \caption{\quad (c) IS}
        \label{fig:cfg_is}
        \includegraphics[width=0.9\linewidth]{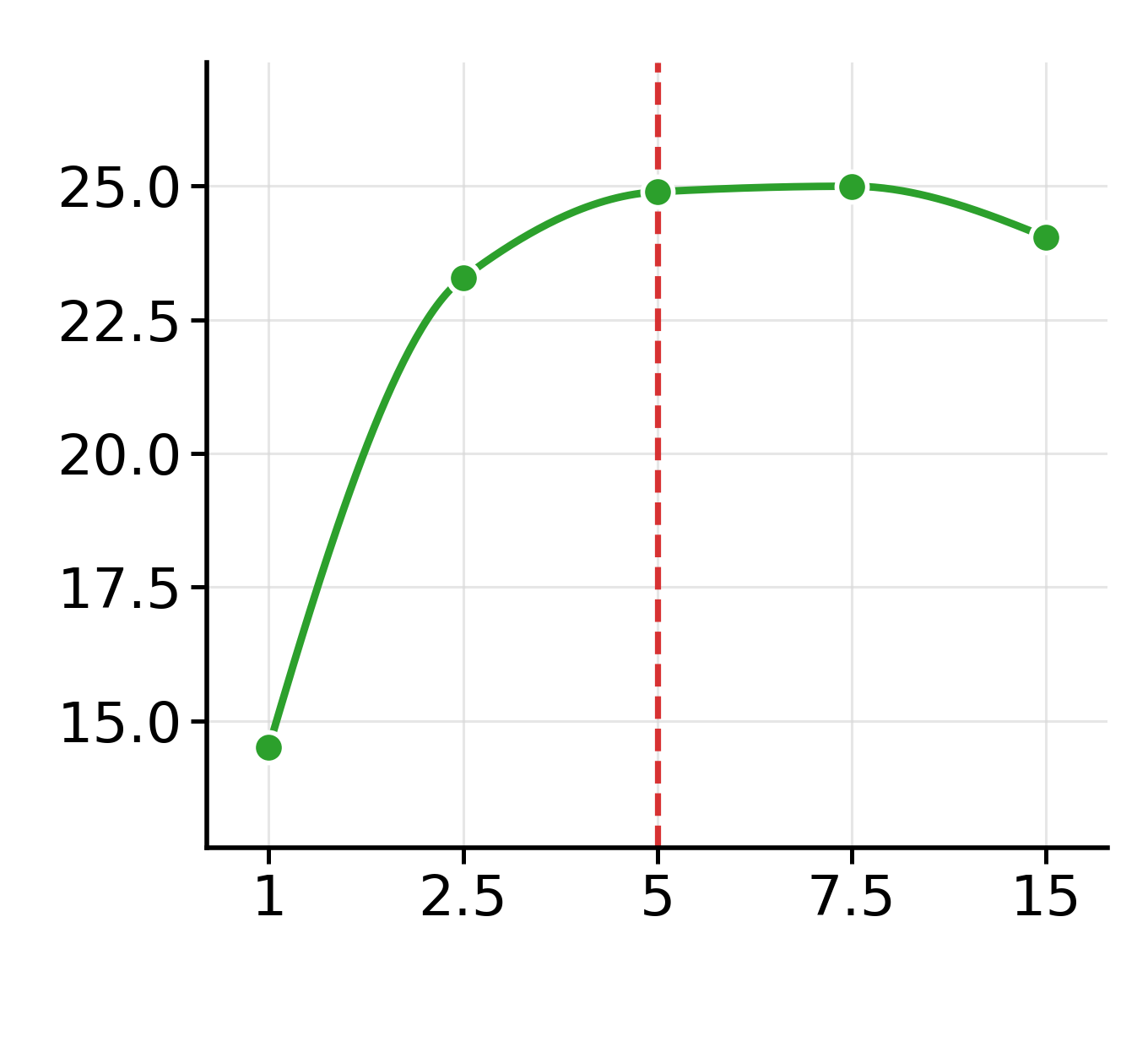}
    \end{subfigure}
    \vspace{-4pt}
    \caption{Effect of guidance scale on spectral distance defined in Eq.~\ref{eq:spectral_distance}, FID, and IS. The x-axis is the guidance scale and the y-axis is the metric value. Guidance scale values are uniformly spaced for visual clarity. The red dashed line marks the default guidance scale.}
    \label{fig:guidance_metric_correlation}
    \vspace{-6pt}
\end{figure}

\section{Spectral Dynamics in the Diffusion Process}
\label{sec:spectral_dynamics}

In this section, we examine how the guided diffusion trajectory evolves during sampling from a spectral perspective. To generate plausible samples from the data distribution, the reverse marginals $p_\theta(z_t)$ defined in Eq.~\ref{eq} should remain close to the corresponding forward marginals $q(z_t)$ in Eq.~\ref{eq:forward_marginal} across timesteps. We thus examine the spectral statistics of intermediate states to assess this marginal alignment. The power spectra of intermediate states provide a tractable summary of how their energy is distributed across frequency components. Agreement between their expected power spectra is a necessary condition for matching the forward and reverse marginals.

We compare the expected power spectra of forward and reverse states across guidance scales under the same noise schedule, as detailed in Appendix~\ref{app:implementation-details}. 
Let $\hat{z}_t=\mathcal{F}(z_t)$ be the Fourier representation of $z_t$, and define its power spectrum as $P_t(\omega)=|\hat{z}_t(\omega)|^2$,
where $\omega>0$ denotes the radial frequency index.
The expected power spectra of forward and reverse states are
$\mathbb{E}_{z_t \sim q(z_t)}[P_t(\omega)]$ and
$\mathbb{E}_{z_t \sim p_\theta(z_t)}[P_t(\omega)]$, respectively.
We quantify their discrepancy across timesteps and radial frequencies using the spectral distance defined as
\begin{equation}
\label{eq:spectral_distance}
D_{\mathrm{spec}}
=
\mathbb{E}_{t,\omega}
\left[
\left(
\log \mathbb{E}_{z_t\sim p_{\theta}(z_t)}[P_t(\omega)]
-
\log \mathbb{E}_{z_t\sim q(z_t)}[P_t(\omega)]
\right)^2
\right].
\end{equation}
Figure~\ref{fig:cfg_spectral_distance} shows the spectral distance at different guidance scales in Eq.~\ref{eq:cfg}.
Moderate guidance reduces this distance relative to conditional sampling without guidance at scale $1$.
However, excessive guidance increases the spectral distance, indicating that the reverse trajectory deviates further from the expected spectral behavior.
Figures~\ref{fig:cfg_fid} and~\ref{fig:cfg_is} further show that spectral discrepancy is correlated with generation quality. Lower spectral discrepancy corresponds to lower FID and higher IS, indicating stronger generative performance. These observations suggest that spectral alignment can indicate whether guided sampling remains consistent with the expected diffusion dynamics. However, such corresponding forward states are unavailable during sampling, making it difficult to correct the denoising trajectory across timesteps and noise schedules.

\begin{figure*}[t]
    \centering

    \raisebox{0.2em}{%
    \begin{subfigure}[b]{0.15\textwidth}
        \centering
        
        \begin{subfigure}[b]{\linewidth}
            \centering
            \makebox[\linewidth][c]{%
                \raisebox{1.75\height}{\rotatebox[origin=c]{90}{Ours}}\hspace{0.4em}%
                \includegraphics[width=0.9\linewidth]{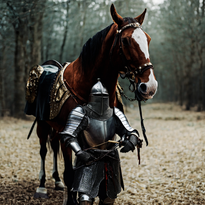}%
            }
            \label{fig:spectral_generated_ours}
        \end{subfigure}

        \vspace{-1em}

        \begin{subfigure}[b]{\linewidth}
            \centering
            \makebox[\linewidth][c]{%
                \raisebox{1.75\height}{\rotatebox[origin=c]{90}{CFG}}\hspace{0.4em}%
                \includegraphics[width=0.9\linewidth]{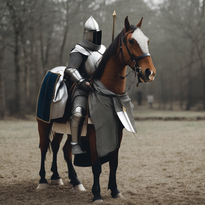}%
            }
            \label{fig:spectral_generated_cfg}
        \end{subfigure}
    \end{subfigure}
    }
    \hspace{0.006\textwidth}
    \begin{subfigure}[b]{0.26\textwidth}
        \centering
        \caption*{Baseline $S_t(\omega)$}
        \includegraphics[width=\linewidth]{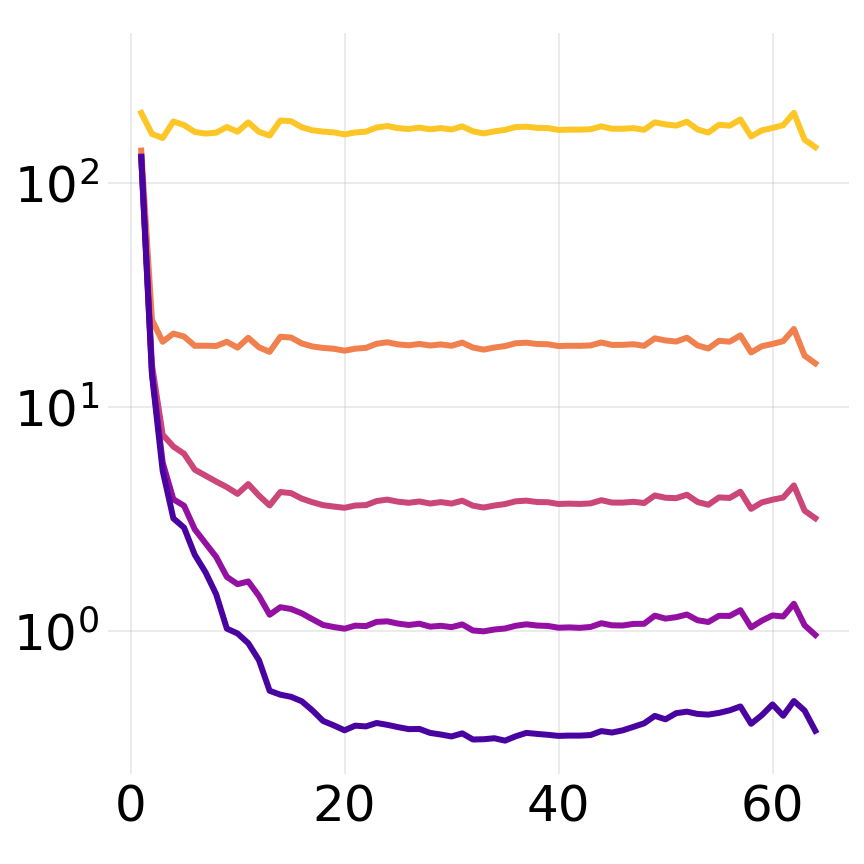}
        \label{fig:spectral_plot_cfg}
    \end{subfigure}
    \hspace{0.006\textwidth}
    \begin{subfigure}[b]{0.26\textwidth}
        \centering
        \caption*{Corrected $\tilde{S}_t(\omega)$}
        \includegraphics[width=\linewidth]{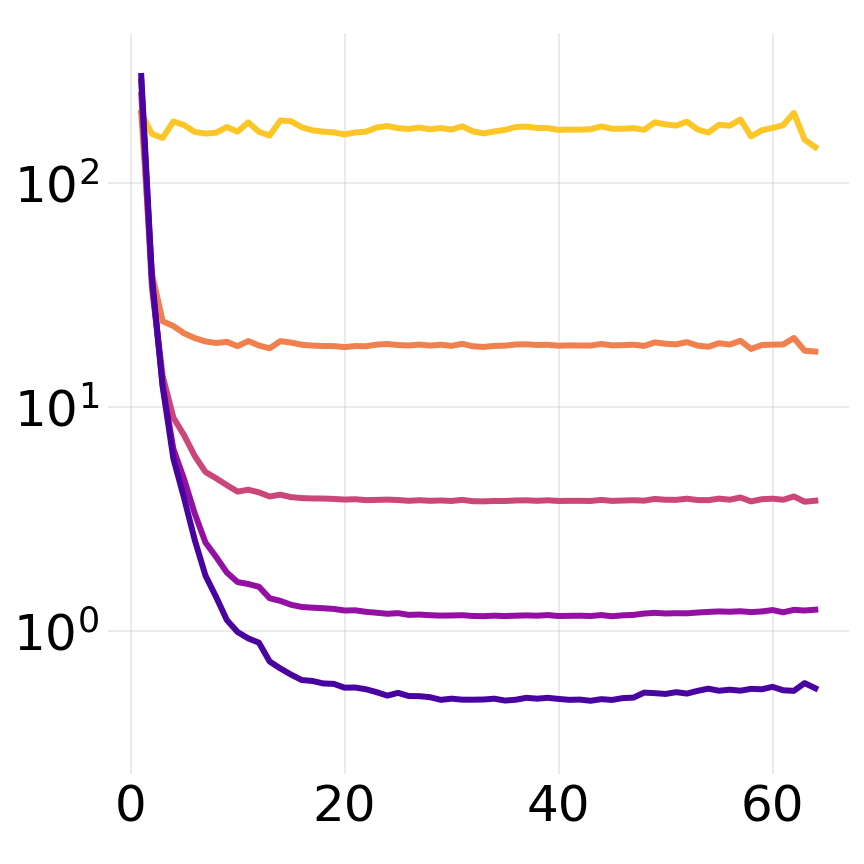}
        \label{fig:spectral_plot_ours}
    \end{subfigure}
    \hspace{0.006\textwidth}
    \begin{subfigure}[b]{0.26\textwidth}
        \centering
        \caption*{Target $T_t(\omega)$}
        \includegraphics[width=\linewidth]{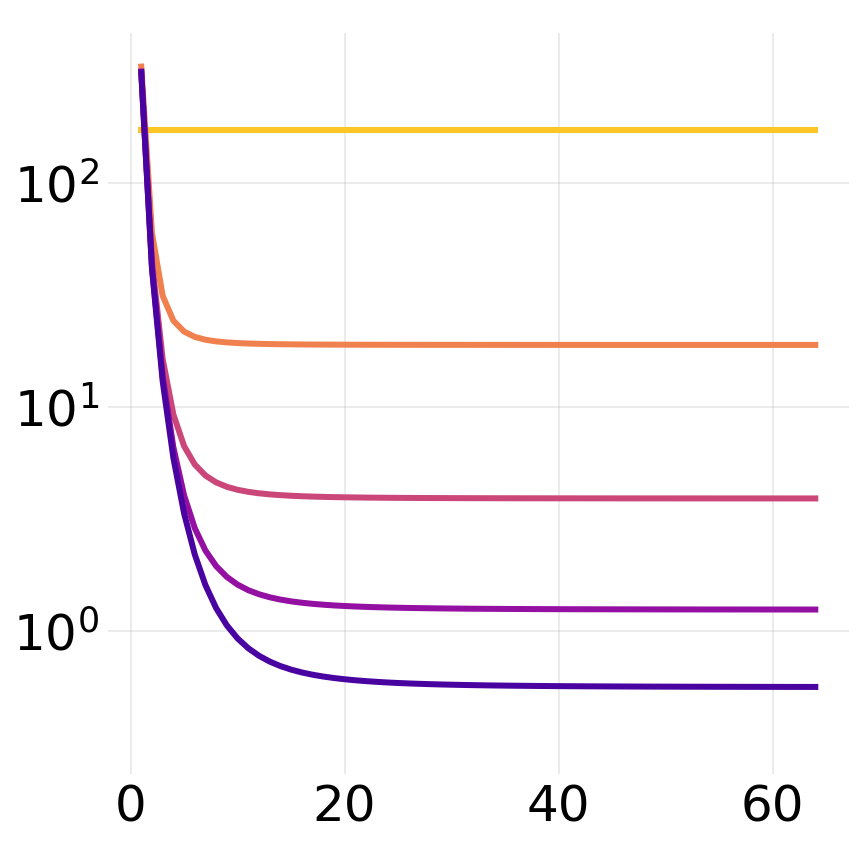}
        \label{fig:spectral_plot_ref}
    \end{subfigure}
    \vspace{-8pt}
    \caption{
Qualitative and spectral comparison for the prompt ``\texttt{a man holding a horse dressed like a knight}''.
Generated images are included for visual reference, and the power spectra are shown over denoising timesteps, with radial frequency bins on the $x$-axis and power on the $y$-axis.
The corrected spectrum $\tilde{S}_t(\omega)$ follows the analytic target $T_t(\omega)$ more closely than the baseline spectrum $S_t(\omega)$, reducing artifacts and improving visual quality in the generated image.
}
    \label{fig:spectral_qual_plot}
    \vspace{-4pt}
\end{figure*}

\vspace{-4pt}
\section{Spectral Correction for Classifier-Free Guidance}
\label{sec:method_spec}
\vspace{-4pt}

To enable correction without access to corresponding forward states, we derive an analytic reference from the statistical structure of natural image spectra.
Using the forward process in Eq.~\ref{eq:zt_def} and the spectral properties of Gaussian noise, we extend this reference to the expected power spectra of intermediate states along the forward diffusion trajectory.


\begin{proposition}
\label{prop:power_law_intermediate}
Following the power-law model of natural image spectra~\citep{ruderman1997origins},
suppose that the expected power spectrum of $z_0$ takes the form\footnote{We extend the standard power-law form with a constant offset $c$, where setting $c=0$ recovers the conventional form.}
\begin{equation}
\label{eq:Ruderman}
\mathbb{E}[P_0(\omega)] = C \omega^{-\beta} + c,
\quad \omega > 0,\; C, \beta > 0,\; c \ge 0.
\end{equation}
By definition, white Gaussian noise has a flat power spectrum,
\[
\mathbb{E}[P_\epsilon(\omega)] = c_\epsilon,
\quad \omega > 0,\; c_\epsilon \ge 0.
\]
For the forward process $z_t=\alpha_t z_0+\sigma_t\epsilon$ in Eq.~\ref{eq:zt_def},
$z_0 \perp \epsilon$ and $\mathbb{E}[\epsilon]=0$, so the cross term vanishes in expectation.
Therefore, the expected power spectrum of $z_t$ has the form
\begin{equation}
\label{eq:proposition}
\mathbb{E}[P_t(\omega)] = \lambda_t \omega^{-\beta} + \gamma_t,
\quad \omega > 0,\; \lambda_t, \gamma_t \ge 0,\; \beta > 0,
\end{equation}
where $\lambda_t$ and $\gamma_t$ are determined by
$\alpha_t,\sigma_t$ and the constants in Eq.~\ref{eq:Ruderman}.
\end{proposition}

Proposition~\ref{prop:power_law_intermediate} shows that the expected power spectrum at each timestep can be specified in an analytic functional form.
This formulation allows us to compute the reference spectrum for assessing spectral alignment without simulating the corresponding forward states.



Based on the criterion in Eq.~\ref{eq:proposition}, we formulate a spectral correction for guided denoising.
Let $S_t(\omega)$ denote the power spectrum of the current reverse state $z_t$,
and let $T_t(\omega)$ denote the expected power spectrum given by this criterion, which serves as the correction target at timestep $t$.
We correct the current spectrum by interpolating between $S_t(\omega)$ and $T_t(\omega)$ in the log-spectral domain:
\begin{equation}
\log \tilde{S}_t(\omega)
=
(1-\lambda)\log S_t(\omega)
+
\lambda \log T_t(\omega),
\label{eq:log_spectral_correction}
\end{equation}
where $\lambda\in[0,1]$ controls the correction strength. The use of log space accommodates the wide dynamic range of spectral power.
Eq.~\ref{eq:log_spectral_correction} can be equivalently written as
\begin{equation}
\tilde{S}_t(\omega)
=
S_t(\omega)
\left(\frac{T_t(\omega)}{S_t(\omega)}\right)^\lambda.
\label{eq:spectral_correction}
\end{equation}
Eq.~\ref{eq:spectral_correction} shows that the correction rescales the current spectrum by the spectral ratio $R_t(\omega)=T_t(\omega)/S_t(\omega)$ with strength $\lambda$. This provides an adaptive correction that increases power when $R_t(\omega)>1$ and decreases power when $R_t(\omega)<1$.

Algorithm~\ref{alg:spectrum_ratio_simple} summarizes how the correction is incorporated into the denoising process. The proposed correction preserves the phase of the current latent and modifies only its Fourier magnitude to match $\tilde{S}_t(\omega)$. This adjustment enables correction of spectral discrepancies while retaining the spatial information encoded in the Fourier phase. Figure~\ref{fig:spectral_qual_plot} visualizes this correction process, showing that our method aligns the corrected spectrum $\tilde{S}_t(\omega)$ with the target $T_t(\omega)$ in both overall energy level and frequency-dependent shape. Our method therefore employs the expected spectral evolution of the forward process as a reference for correcting guided denoising trajectories.

In practice, we estimate the coefficients in Eq.~\ref{eq:Ruderman} from the spectral statistics of real image samples~\citep{lin2014microsoft} rather than through model-specific performance tuning.
We perform this estimation in the $z_0$ space of each backbone without selecting samples to match its training data~\citep{deng2009imagenet,schuhmann2022laion}.
We compute radial power spectra by averaging two-dimensional Fourier power over radial frequency bins. Given the estimated coefficients, Eq.~\ref{eq:proposition} determines the target spectrum $T_t(\omega)$ at each timestep through the corresponding noise schedule. To retain fine-grained frequency components~\citep{karras2021alias}, we relax the correction in the late denoising stage as discussed in Section~\ref{sec:abl}. This procedure applies a common correction principle across diffusion models with different spectral statistics and noise schedules.



\vspace{-4pt}
\section{A Spectral View of Existing Guidance Methods} 
\label{sec:spectral_view_guidance}
\vspace{-4pt}
In this section, we discuss related guidance strategies.
From a spectral perspective, we also examine how these strategies control distinct aspects of the guidance signal during denoising.

\vspace{-7pt}
\paragraph{Scheduled guidance scaling.} Scheduled guidance scaling replaces the constant guidance scale $w$ with a predefined timestep-dependent schedule. \citet{kynkaanniemi2024applying} restrict guidance to selected denoising intervals, while \citet{malarz2025classifier} and \citet{jin2025stage} adjust the guidance strength according to predefined functional forms.
By Parseval's theorem, adjusting the guidance scale at each timestep also modulates the Fourier-domain energy of the guidance term. In this sense, our method provides fine-grained scaling through spectral correction at each timestep and frequency.
\vspace{-7pt}

\paragraph{Geometry-aware guidance.} Geometry-aware guidance methods constrain the magnitude~\citep{jin2025angle} or direction~\citep{sadat2024eliminating} of guided updates, or adjust sampling to reduce deviations from the data manifold~\citep{chung2024cfg++,kwon2025tcfg}.
Our method shares the goal of keeping guided trajectories consistent with the data distribution, but approaches it through spectral alignment with the forward marginals encountered during training.
\vspace{-7pt}

\paragraph{Frequency-based guidance.} Frequency-based methods modulate guidance using predefined frequency partitions~\citep{zhang2024frecas}, applying different guidance strengths to separate bands~\citep{sadat2025guidance} or suppressing selected frequencies~\citep{song2026rethinking}. Although these methods also consider spectral information, their adjustments follow band-specific rules. In contrast, our method adapts the correction at each frequency based on spectral deviations from an analytic reference.

\begin{table*}[t]
\centering
\caption{Quantitative comparison of guidance methods on text-to-image diffusion models, including HPSv3. Our method achieves the strongest preference-based performance across models. Among existing guidance strategies, standard CFG remains the most competitive baseline in most cases.}
\vspace{-4pt}
\label{tab:t2i_cfg_ours_comparison}
\setlength{\tabcolsep}{7pt}
\renewcommand{\arraystretch}{1.15}
\resizebox{0.93\textwidth}{!}{
\begin{tabular}{llcccccccc}
\toprule
Model & $w$ & Method & HPSv3$\uparrow$ & ImageReward$\uparrow$ & PickScore$\uparrow$ & HPSv2$\uparrow$ & CLIP-T$\uparrow$ & FID$\downarrow$ & sFID$\downarrow$ \\
\midrule\midrule
\multirow{5}{*}{SDXL}
& \multirow{5}{*}{5.0}
& \cellcolor{gray!20}Ours
& \cellcolor{gray!20}\textbf{8.437}
& \cellcolor{gray!20}\textbf{0.731}
& \cellcolor{gray!20}\textbf{22.43}
& \cellcolor{gray!20}\textbf{0.279}
& \cellcolor{gray!20}\underline{0.264}
& \cellcolor{gray!20}\textbf{17.02}
& \cellcolor{gray!20}\textbf{47.94}\\
& & CFG
& \underline{8.096} & \underline{0.700} & \underline{22.41} & \underline{0.274} & \textbf{0.265} & 17.62 & 49.27 \\
& & CFG++
& 7.962 & 0.684 & 22.36 & 0.273 & \textbf{0.265} & \underline{17.46} & \underline{49.11} \\
& & TV-CFG
& 6.952 & 0.514 & 22.03 & 0.263 & \underline{0.264} & 18.94 & 50.14 \\
& & LF-CFG
& 7.045 & 0.566 & 21.95 & 0.267 & 0.262 & 18.00 & 52.75 \\
\midrule
\multirow{5}{*}{PixArt-$\alpha$}
& \multirow{5}{*}{4.5}
& \cellcolor{gray!20}Ours
& \cellcolor{gray!20}\textbf{9.957}
& \cellcolor{gray!20}\textbf{0.876}
& \cellcolor{gray!20}\textbf{22.63}
& \cellcolor{gray!20}\textbf{0.299}
& \cellcolor{gray!20}\underline{0.260}
& \cellcolor{gray!20}\textbf{27.95}
& \cellcolor{gray!20}66.35 \\
& & CFG
& \underline{9.571} & 0.845 & \underline{22.55} & \underline{0.298} & \underline{0.260} & 28.30 & 66.43 \\
& & CFG++
& 6.664 & 0.346 & 21.69 & 0.263 & 0.245 & \underline{28.27} & \underline{56.85} \\
& & TV-CFG
& 9.510 & \underline{0.864} & 22.43 & \underline{0.298} & \textbf{0.261} & 29.97 & 66.61 \\
& & LF-CFG
& 3.371 & 0.120 & 20.97 & 0.246 & 0.245 & 33.77 & \textbf{54.13} \\
\midrule
\multirow{5}{*}{SD3.5}
& \multirow{5}{*}{4.5}
& \cellcolor{gray!20}Ours
& \cellcolor{gray!20}\textbf{9.471}
& \cellcolor{gray!20}\textbf{0.986}
& \cellcolor{gray!20}\textbf{22.44}
& \cellcolor{gray!20}\textbf{0.284}
& \cellcolor{gray!20}\underline{0.268}
& \cellcolor{gray!20}20.24
& \cellcolor{gray!20}49.07\\
& & CFG
& \underline{9.406} & \underline{0.979} & \underline{22.41} & \textbf{0.284} & \textbf{0.269} & 20.31 & 48.79 \\
& & CFG++
& -4.826 & -1.082 & 19.70 & 0.174 & 0.229 & 74.49 & 66.35 \\
& & TV-CFG
& 8.618 & 0.812 & 22.06 & \underline{0.275} & 0.266 & \textbf{15.98} & \underline{44.22} \\
& & LF-CFG
& 8.631 & 0.889 & 22.13 & \underline{0.275} & \underline{0.268} & \underline{19.93} & \textbf{12.61} \\
\bottomrule
\end{tabular}
\vspace{-4pt}
}
\end{table*}

\vspace{-4pt}
\section{Experiments and Results}
\label{sec:exp_res}
\vspace{-4pt}
\paragraph{Experimental setup.}
We evaluate our method on text-to-image and class-conditional image generation.
For text-to-image generation, we compare against standard CFG~\citep{ho2022classifier} and representative methods from Section~\ref{sec:spectral_view_guidance}: TV-CFG~\citep{jin2025stage}, CFG++~\citep{chung2024cfg++}, and LF-CFG~\citep{song2026rethinking}.
We use 10K prompts sampled from the COCO validation set~\citep{lin2014microsoft}.
To assess generality across backbones, we conduct experiments with SDXL~\citep{podell2023sdxl}, PixArt-$\alpha$~\citep{chen2023pixart}, and SD3.5~\citep{esser2024scaling}.
We evaluate each model at its default guidance scale and at a higher guidance scale.
For class-conditional generation, we use DiT-XL/2~\citep{peebles2023scalable} on ImageNet~\citep{deng2009imagenet} and evaluate robustness across different numbers of sampling steps.
We generate 50 samples per ImageNet class, resulting in 50K samples in total.
Unless otherwise specified, all methods use the same prompts, labels, sampling budget, and guidance scales, and follow the recommended settings for each method.



\paragraph{Evaluation metrics.}
For text-to-image generation, we assess human preference using HPSv3~\citep{Ma_2025_ICCV}, ImageReward~\citep{xu2023imagereward}, PickScore~\citep{kirstain2023pick}, and HPSv2~\citep{wu2023human}, text-image alignment using CLIP-T~\citep{radford2021learning}, and distributional matching using FID~\citep{heusel2017gans} and sFID~\citep{nash2021generating}.
We complement these metrics with a human preference study.
For class-conditional generation, we follow the ADM evaluation protocol~\citep{dhariwal2021diffusion} and report FID, sFID, Inception Score (IS)~\citep{salimans2016improved}, Precision, and Recall to evaluate generation quality and distributional coverage across sampling budgets measured by the number of function evaluations (NFEs).

\begin{figure*}[!htbp]
    \centering
    \setlength{\tabcolsep}{2pt}
    \renewcommand{\arraystretch}{1.0}
    \scriptsize

    \begin{tabular}{c ccccc}
        & (a) Ours & (b) CFG & (c) CFG++ & (d) TV-CFG & (e) LF-CFG \\[0.25em]

        \multirow{4}{*}{\rotatebox{90}{\textbf{SDXL}}}
        & \includegraphics[width=0.18\linewidth]{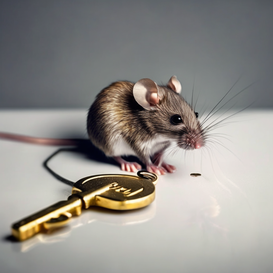}
        & \includegraphics[width=0.18\linewidth]{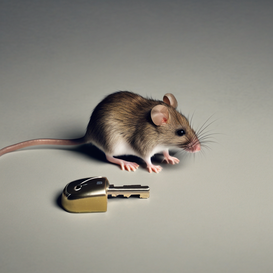}
        & \includegraphics[width=0.18\linewidth]{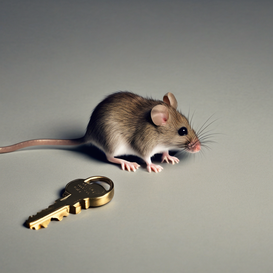}
        & \includegraphics[width=0.18\linewidth]{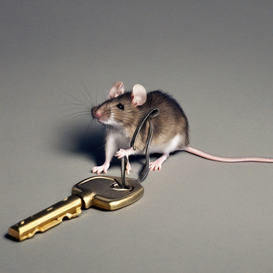}
        & \includegraphics[width=0.18\linewidth]{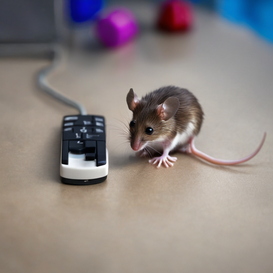} \\[-0.1em]
        & \multicolumn{5}{c}{\scriptsize \textbf{Prompt:} ``A mouse is sitting on a surface next to a key.''} \\[0.45em]
        & \includegraphics[width=0.18\linewidth]{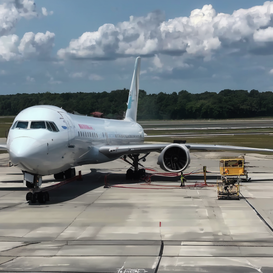}
        & \includegraphics[width=0.18\linewidth]{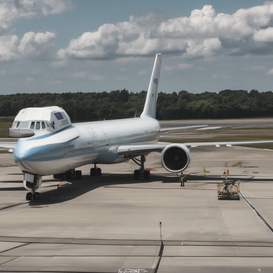}
        & \includegraphics[width=0.18\linewidth]{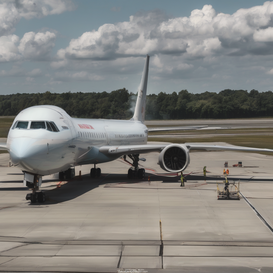}
        & \includegraphics[width=0.18\linewidth]{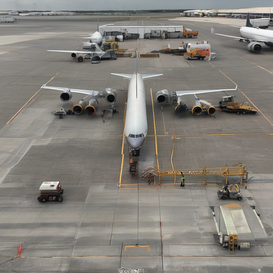}
        & \includegraphics[width=0.18\linewidth]{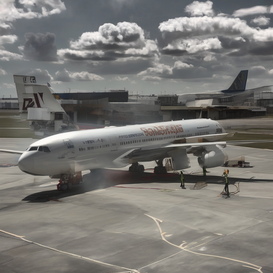} \\[-0.1em]
        & \multicolumn{5}{c}{\scriptsize \textbf{Prompt:} ``An older passenger jet is being readied for takeoff.''} \\[0.8em]

        \multirow{4}{*}{\rotatebox{90}{\textbf{PixArt-$\alpha$}}}
        & \includegraphics[width=0.18\linewidth]{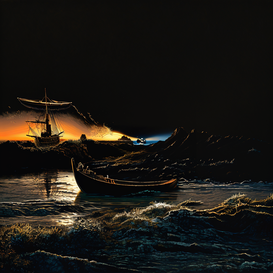}
        & \includegraphics[width=0.18\linewidth]{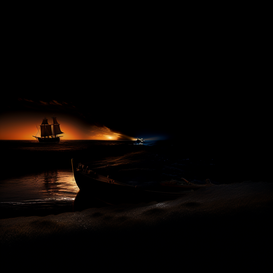}
        & \includegraphics[width=0.18\linewidth]{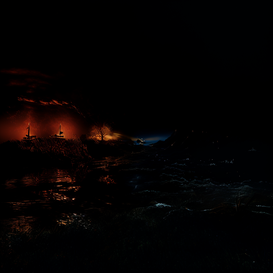}
        & \includegraphics[width=0.18\linewidth]{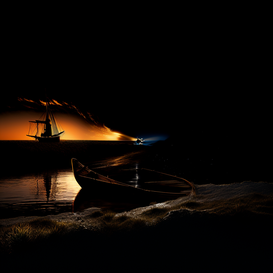}
        & \includegraphics[width=0.18\linewidth]{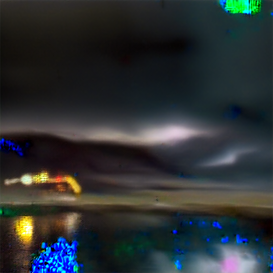} \\[-0.1em]
        & \multicolumn{5}{c}{\scriptsize \textbf{Prompt:} ``The bow of a ship on land with another on the edge of the water.''} \\[0.45em]
        & \includegraphics[width=0.18\linewidth]{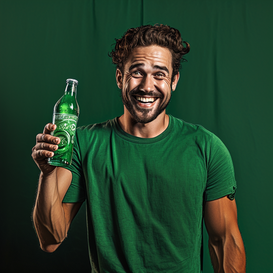}
        & \includegraphics[width=0.18\linewidth]{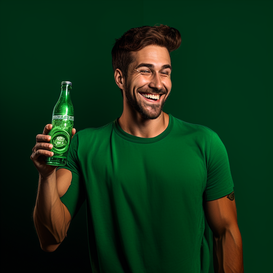}
        & \includegraphics[width=0.18\linewidth]{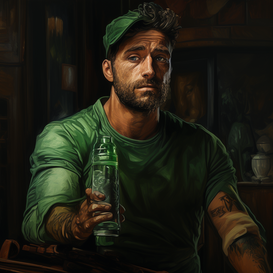}
        & \includegraphics[width=0.18\linewidth]{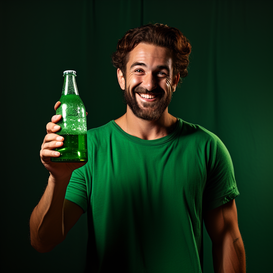}
        & \includegraphics[width=0.18\linewidth]{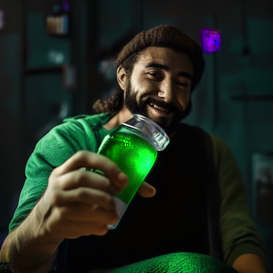} \\[-0.1em]
        & \multicolumn{5}{c}{\scriptsize \textbf{Prompt:} ``A young man holding up a green bottle as he drinks.''} \\[0.8em]

        \multirow{4}{*}{\rotatebox{90}{\textbf{SD3.5}}}
        & \includegraphics[width=0.18\linewidth]{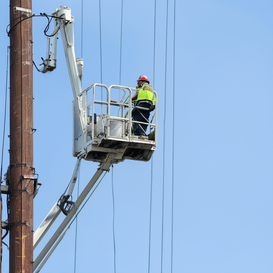}
        & \includegraphics[width=0.18\linewidth]{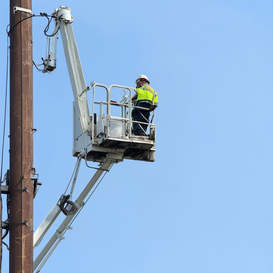}
        & \includegraphics[width=0.18\linewidth]{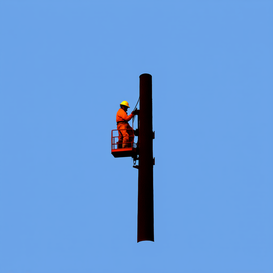}
        & \includegraphics[width=0.18\linewidth]{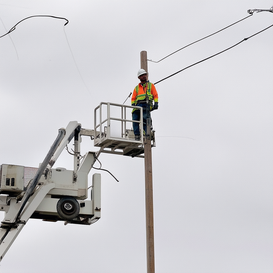}
        & \includegraphics[width=0.18\linewidth]{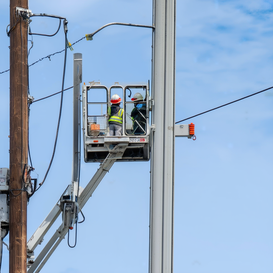} \\[-0.1em]
        & \multicolumn{5}{c}{\scriptsize \textbf{Prompt:} ``A Utility Pole Worker standing inside of the boom of a bucket truck.''} \\[0.45em]
        & \includegraphics[width=0.18\linewidth]{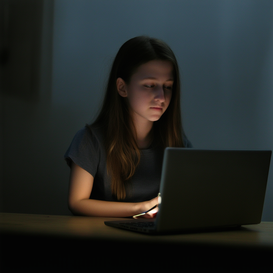}
        & \includegraphics[width=0.18\linewidth]{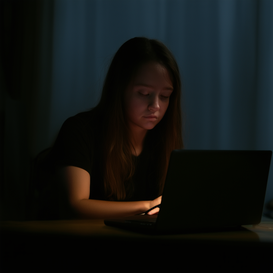}
        & \includegraphics[width=0.18\linewidth]{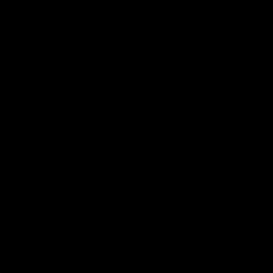}
        & \includegraphics[width=0.18\linewidth]{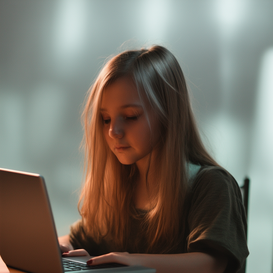}
        & \includegraphics[width=0.18\linewidth]{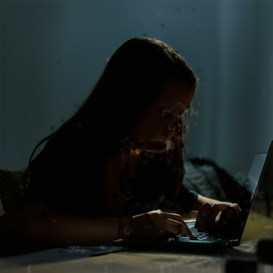} \\[-0.1em]
        & \multicolumn{5}{c}{\scriptsize \textbf{Prompt:} ``A girl working on a laptop in dim light.''}
    \end{tabular}
    \vspace{-0.2em}
    \caption{Qualitative comparison of guidance methods across text-to-image diffusion models.
    Our method better preserves the intended composition and improves visual coherence.}
    \label{fig:qualitative_all_models}


    \begingroup
    \setlength{\tabcolsep}{3pt}
    \renewcommand{\arraystretch}{0.92}
    \scriptsize
    \captionof{table}{Quantitative comparison between CFG and our method at higher guidance scales. Our method improves preference-based metrics across all three backbones.}
    \label{tab:model_high_guidance}
    \vspace{-0.15em}
    \resizebox{0.72\linewidth}{!}{
    \begin{tabular}{lllcccccc}
    \toprule
    Model & $w$ & Method & ImageReward$\uparrow$ & PickScore$\uparrow$ & HPSv2$\uparrow$ & CLIP-T$\uparrow$ & FID$\downarrow$ & sFID$\downarrow$ \\
    \midrule\midrule
    \multirow{2}{*}{SDXL}
    & \multirow{2}{*}{10.0}
    & Ours\cellcolor{gray!20}
    & \cellcolor{gray!20}\textbf{0.860}
    & \cellcolor{gray!20}\textbf{22.61}
    & \cellcolor{gray!20}\textbf{0.290}
    & \cellcolor{gray!20}\textbf{0.268}
    & \cellcolor{gray!20}\textbf{19.01}
    & \cellcolor{gray!20}\textbf{48.88} \\
    & & CFG
    & 0.842 & 22.60 & 0.288 & \textbf{0.268} & 19.49 & 50.03 \\
    \midrule
    \multirow{2}{*}{PixArt-$\alpha$}
    & \multirow{2}{*}{9.0}
    & Ours\cellcolor{gray!20}
    & \cellcolor{gray!20}\textbf{0.871}
    & \cellcolor{gray!20}\textbf{22.41}
    & \cellcolor{gray!20}\textbf{0.298}
    & \cellcolor{gray!20}\textbf{0.261}
    & \cellcolor{gray!20}\textbf{32.00}
    & \cellcolor{gray!20}72.09 \\
    & & CFG
    & 0.844 & 22.37 & 0.296 & \textbf{0.261} & 32.19 & \textbf{71.59} \\
    \midrule
    \multirow{2}{*}{SD3.5}
    & \multirow{2}{*}{14.0}
    & Ours\cellcolor{gray!20}
    & \cellcolor{gray!20}\textbf{1.055}
    & \cellcolor{gray!20}\textbf{22.28}
    & \cellcolor{gray!20}\textbf{0.292}
    & \cellcolor{gray!20}\textbf{0.270}
    & \cellcolor{gray!20}\textbf{22.71}
    & \cellcolor{gray!20}\textbf{12.93} \\
    & & CFG
    & 0.657 & 21.70 & 0.266 & 0.267 & 24.17 & 21.86 \\
    \bottomrule
    \end{tabular}
    }
    \endgroup
\end{figure*}


\vspace{-4pt}
\subsection{Experimental Results}
\label{sec:t2i_generation}
\vspace{-4pt}
\paragraph{Text-to-image generation.}
Table~\ref{tab:t2i_cfg_ours_comparison} presents quantitative results comparing our method with guidance baselines under their standard guidance settings. Standard CFG outperforms the other baseline methods on most preference-based metrics across the three backbones, indicating limited gains from existing guidance strategies at standard guidance scales. Our method improves on this strong baseline, achieving the highest HPSv3, ImageReward, and PickScore across all three models, along with the best or joint-best HPSv2 scores.
Even on SD3.5, where standard CFG already achieves strong preference scores, our method yields additional gains. For distributional metrics, our method achieves the lowest FID and sFID on SDXL and the lowest FID on PixArt-$\alpha$. Although TV-CFG and LF-CFG achieve lower FID and sFID on SD3.5, this comes with lower preference scores.



Figure~\ref{fig:qualitative_all_models} provides qualitative support for the trends in Table~\ref{tab:t2i_cfg_ours_comparison}.
In the SDXL examples, most methods produce plausible images, although LF-CFG shows weaker prompt alignment.
CFG introduces a visible artifact around the front part of the airplane, which is alleviated by our method and CFG++.
For PixArt-$\alpha$, although other guidance variants provide some improvements, our method shows clearer gains in overall appearance and brightness balance.
In the SD3.5 examples, CFG++ produces over-smoothed images, consistent with its lower quantitative scores.
The remaining methods achieve similar overall quality, but our method produces more distinct fine-grained details.


Furthermore, we evaluate our method under different guidance regimes. Table~\ref{tab:model_high_guidance} demonstrates that our method improves over CFG at higher guidance scales. The improvement is clear for SD3.5, where our method maintains strong image quality and alignment while CFG becomes more sensitive to the increased guidance scale. As shown in Figure~\ref{fig:cfg_ours_sd35_high}, the larger guidance scale can lead to over-saturated images or weaker prompt following, whereas our method mitigates these effects and recovers more faithful compositions. Table~\ref{tab:scale_abl} also shows that our method improves under lower guidance, while existing guidance methods show limited gains.
These results indicate that the proposed correction provides stable guidance control across models and guidance scales. We provide a human preference study, multi-seed experiments and runtime measurements in the appendix.

\begin{figure*}[t]
    \centering
    \setlength{\tabcolsep}{2pt}
    \renewcommand{\arraystretch}{1.0}

    \begin{tabular}{cccc}
        CFG & Ours & CFG & Ours \\[-0.1em]
        \includegraphics[width=0.24\linewidth]{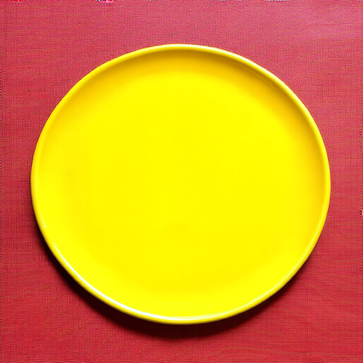} &
        \includegraphics[width=0.24\linewidth]{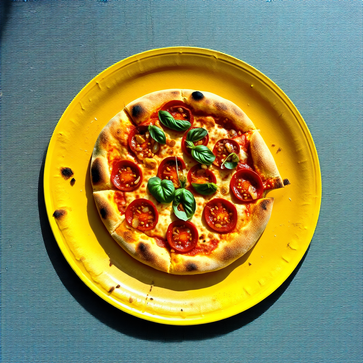} &
        \includegraphics[width=0.24\linewidth]{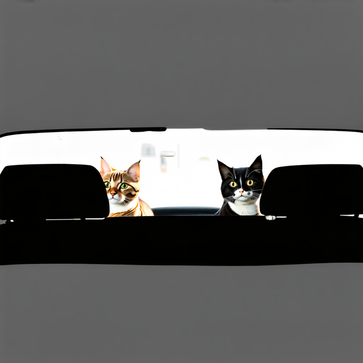} &
        \includegraphics[width=0.24\linewidth]{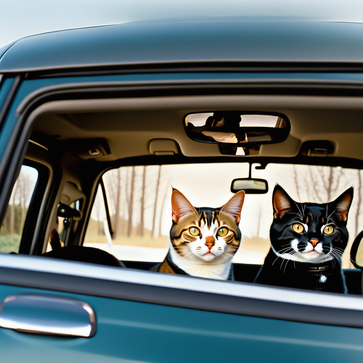} \\[-0.1em]

        \multicolumn{2}{c}{\scriptsize \textbf{Prompt:} ``\texttt{A small loaded pizza on a yellow plate}''} &
        \multicolumn{2}{c}{\scriptsize \textbf{Prompt:} ``\texttt{a pair of cats riding inside a car}''}
    \end{tabular}

    \caption{Qualitative comparison between CFG and our method on SD3.5 at a high guidance scale of $w=14$. Our method corrects over-saturated and over-simplified outputs produced by CFG.}
    \label{fig:cfg_ours_sd35_high}
    \vspace{-10pt}
\end{figure*}

\begin{figure*}[h]
\centering

\begin{minipage}[t]{0.50\textwidth}
    \vspace{0pt}
    \centering
    \resizebox{\linewidth}{!}{
    \setlength{\tabcolsep}{0.5pt}
    \renewcommand{\arraystretch}{0.95}
    \begin{tabular}{c@{\hspace{4pt}}ccc}
        & Chameleon & Jellyfish & Space shuttle \\[0.15em]

        \raisebox{1.0\height}{\rotatebox{90}{Ours}} &
        \includegraphics[width=0.32\linewidth]{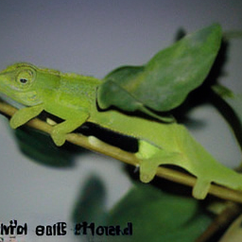} &
        \includegraphics[width=0.32\linewidth]{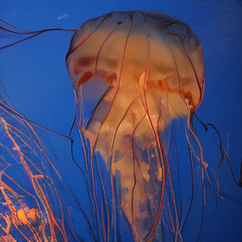} &
        \includegraphics[width=0.32\linewidth]{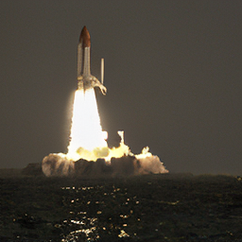} \\[0.15em]
        
        \raisebox{1.0\height}{\rotatebox{90}{CFG}} &
        \includegraphics[width=0.32\linewidth]{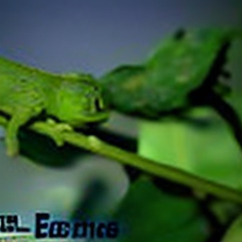} &
        \includegraphics[width=0.32\linewidth]{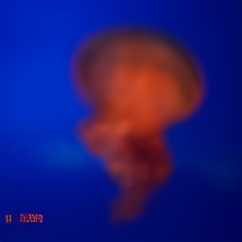} &
        \includegraphics[width=0.32\linewidth]{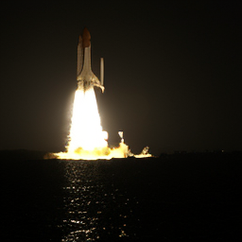}
    \end{tabular}
    }
    \vspace{-0.7em}
    \captionof{figure}{Qualitative examples for ImageNet class-conditional generation with DiT-XL/2.}
    \label{fig:class_cond_qual}
\end{minipage}
\hfill
\begin{minipage}[t]{0.48\textwidth}
    \vspace{0pt}
    \centering
    \captionof{table}{Class-conditional ImageNet generation results with DiT-XL/2. Results include the default 250-step setting and reduced sampling budgets measured by the number of function evaluations (NFE).}
    \label{tab:nfe_ablation}
    \resizebox{\linewidth}{!}{
    \begin{tabular}{llccccc}
    \toprule
    NFE & Method & FID$\downarrow$ & sFID$\downarrow$ & Precision$\uparrow$ & Recall$\uparrow$ & IS$\uparrow$ \\
    \midrule\midrule

    \multirow{2}{*}{250}
    & Ours\cellcolor{gray!20}
    & \cellcolor{gray!20}\textbf{2.62}
    & \cellcolor{gray!20}\textbf{4.92}
    & \cellcolor{gray!20}\textbf{0.842}
    & \cellcolor{gray!20}0.550
    & \cellcolor{gray!20}\textbf{62.92} \\
    & CFG
    & 2.89
    & 7.42
    & 0.823
    & \textbf{0.565}
    & 61.77 \\
    \midrule

    \multirow{2}{*}{100}
    & Ours\cellcolor{gray!20}
    & \cellcolor{gray!20}\textbf{2.78}
    & \cellcolor{gray!20}\textbf{5.65}
    & \cellcolor{gray!20}\textbf{0.843}
    & \cellcolor{gray!20}0.553
    & \cellcolor{gray!20}\textbf{63.13} \\
    & CFG
    & 3.07
    & 7.64
    & 0.827
    & \textbf{0.557}
    & 61.84 \\
    \midrule

    \multirow{2}{*}{50}
    & Ours\cellcolor{gray!20}
    & \cellcolor{gray!20}\textbf{2.88}
    & \cellcolor{gray!20}\textbf{6.29}
    & \cellcolor{gray!20}\textbf{0.840}
    & \cellcolor{gray!20}0.552
    & \cellcolor{gray!20}\textbf{62.52} \\
    & CFG
    & 3.09
    & 7.54
    & 0.825
    & \textbf{0.554}
    & 61.54 \\
    \midrule

    \multirow{2}{*}{25}
    & Ours\cellcolor{gray!20}
    & \cellcolor{gray!20}\textbf{3.07}
    & \cellcolor{gray!20}\textbf{6.78}
    & \cellcolor{gray!20}\textbf{0.830}
    & \cellcolor{gray!20}0.545
    & \cellcolor{gray!20}\textbf{61.47} \\
    & CFG
    & 3.20
    & 7.43
    & 0.824
    & \textbf{0.548}
    & 60.79 \\
    \bottomrule
    \end{tabular}
    }
\end{minipage}
\vspace{-10pt}
\end{figure*}

\paragraph{Class-conditional image generation.}
To evaluate our method across conditional generation tasks, we apply the proposed correction to class-conditional generation with DiT-XL/2 on ImageNet. As shown in Table~\ref{tab:nfe_ablation}, our method achieves strong performance with the default 250 denoising steps. In particular, it reduces sFID from 7.42 to 4.92 and improves FID, Precision, and IS. Figure~\ref{fig:class_cond_qual} shows that these quantitative gains are reflected in image quality. In the chameleon example, the proposed correction leads to a clearer depiction of the target class. The improvement is also visible in the jellyfish and space shuttle examples, where our method produces brighter and sharper images than the blurrier or darker CFG samples.

The proposed method maintains strong performance across sampling budgets, with NFEs ranging from 250 to 25. Even with 50 NFEs, our method achieves an FID of 2.88 and an sFID of 6.29, matching the FID and improving upon the sFID of CFG with 250 NFEs. At 25 NFEs, our method remains competitive with CFG using 100 NFEs. Recall is slightly lower than CFG across budgets, indicating a trade-off between coverage and the observed gains in precision and fidelity. These results demonstrate that the proposed correction improves generation quality across sampling budgets and extends to conditional generation tasks beyond text-to-image synthesis.

\vspace{-7pt}
\subsection{Ablation Studies}
\label{sec:abl}
\vspace{-7pt}
Here, we study how correction strength affects generation results and examine how to make effective use of the proposed correction in the denoising process.

\begin{figure*}[t]
    \centering
    \begin{subfigure}[t]{0.23\linewidth}
        \centering
        \caption{CFG}
        \label{fig:lambda_ablation_cfg}
        \includegraphics[width=\linewidth]{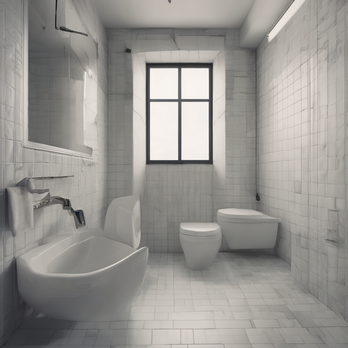}
    \end{subfigure}
    \hfill
    \begin{subfigure}[t]{0.23\linewidth}
        \centering
        \caption{$\lambda=0.005$}
        \label{fig:lambda_ablation_0005}
        \includegraphics[width=\linewidth]{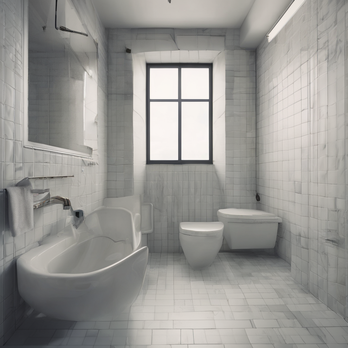}
    \end{subfigure}
    \hfill
    \begin{subfigure}[t]{0.23\linewidth}
        \centering
        \caption{$\lambda=0.05$}
        \label{fig:lambda_ablation_005}
        \includegraphics[width=\linewidth]{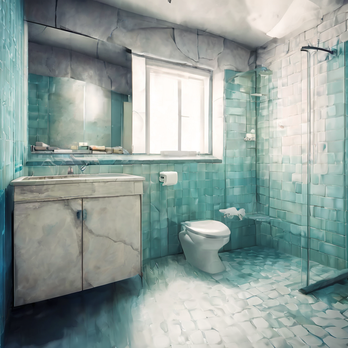}
    \end{subfigure}
    \hfill
    \begin{subfigure}[t]{0.23\linewidth}
        \centering
        \caption{$\lambda=0.5$}
        \label{fig:lambda_ablation_05}
        \includegraphics[width=\linewidth]{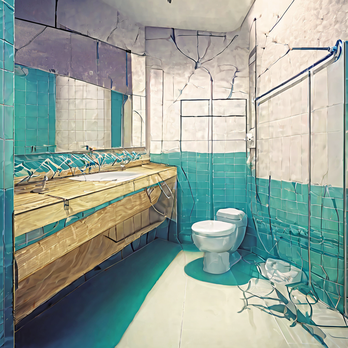}
    \end{subfigure}

    \vspace{0.3em}
    {\small \textbf{Prompt:} ``\texttt{A bathroom that is empty and not in use}''}
    \vspace{-4pt}
    \caption{Ablation on the correction strength $\lambda$.
Increasing $\lambda$ improves visual fidelity and prompt alignment up to a point, but excessive correction introduces artifacts.}
    \label{fig:lambda_ablation}
    \vspace{-8pt}
\end{figure*}

\paragraph{Effect of correction strength.}
The ablation shows that an appropriate correction strength improves visual fidelity and prompt alignment.
As shown in Figure~\ref{fig:lambda_ablation_0005}, a small correction leaves the output close to standard guidance with little visible effect, suggesting that weak spectral modification has a limited effect on the guided trajectory. At moderate strengths, Figure~\ref{fig:lambda_ablation_005} shows richer colors and clearer details, indicating that stronger correction better restores spectral components affected by guidance.
However, excessive correction introduces artifacts and degrades perceptual quality as shown in Figure~\ref{fig:lambda_ablation_05}. These results suggest that the correction strength controls the balance between correcting deviations from the expected spectrum and preserving individual sample characteristics.

\begin{figure*}[t]
    \centering
    \setlength{\tabcolsep}{2pt}
    \renewcommand{\arraystretch}{1.0}
    \begin{tabular}{cccccc}
        $\rho=0.25$ & $\rho=0.5$ & $\rho=0.75$ &
        $\rho=0.25$ & $\rho=0.5$ & $\rho=0.75$ \\[0.25em]

        \includegraphics[width=0.15\linewidth]{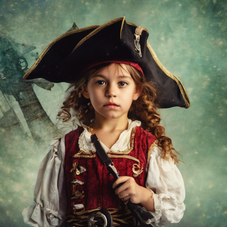} &
        \includegraphics[width=0.15\linewidth]{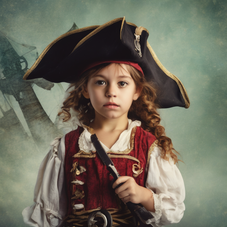} &
        \includegraphics[width=0.15\linewidth]{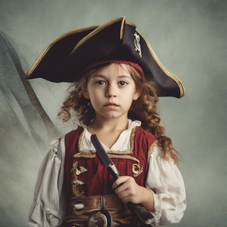} &
        \includegraphics[width=0.15\linewidth]{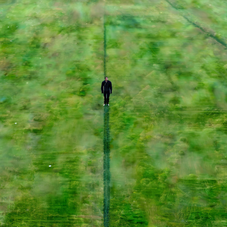} &
        \includegraphics[width=0.15\linewidth]{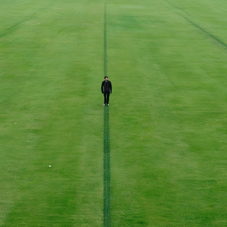} &
        \includegraphics[width=0.15\linewidth]{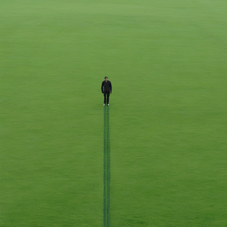} \\[-0.1em]

        \multicolumn{3}{c}{\scriptsize \textbf{Prompt:} ``\texttt{A young girl pretending to be a pirate}''} &
        \multicolumn{3}{c}{\scriptsize \textbf{Prompt:} ``\texttt{a person standing on a large green field}''}
    \end{tabular}
    \vspace{-4pt}
    \caption{Ablation on the relaxation parameter $\rho$. Smaller $\rho$ applies correction for longer during denoising, while larger $\rho$ relaxes it earlier. Intermediate relaxation yields the best balance between spectral correction and artifact reduction.}
    \label{fig:rho_ablation}
    \vspace{-10pt}
\end{figure*}

\vspace{-7pt}
\paragraph{Correction relaxation.}
The previous ablation shows that increasing correction strength improves fidelity and prompt alignment but can introduce artifacts.
We thus study how to relax the correction across the denoising process to retain its benefits while reducing artifacts.
Motivated by the role of late denoising steps in fine-grained detail formation~\citep{karras2021alias}, we vary the relaxation parameter $\rho$, which controls when the frequency-domain constraint is relaxed during denoising.
Figure~\ref{fig:rho_ablation} shows that applying correction until the end can retain artifacts, whereas applying it only in the early stage weakens its effect.
An intermediate relaxation point provides the best balance between spectral correction and natural detail formation. We therefore adopt this setting in our final method. 
The relaxation analysis clarifies where spectral correction improves generation quality and where its continued application can degrade quality~\citep{kynkaanniemi2024applying, jin2025stage}.


\vspace{-7pt}
\section{Conclusion}
\vspace{-7pt}


In this paper, we demonstrated that spectral alignment can serve as a practical criterion for understanding guidance behavior and improving guided diffusion sampling through principled correction. This perspective provides a basis for assessing whether guided trajectories remain consistent with the spectral dynamics induced by the forward process. We derived an analytic reference from natural image statistics and used it to correct spectral deviations along guided denoising trajectories. Building on this criterion, Spectral Correction Guidance is a training-free method applicable to diverse diffusion backbones and conditional generation tasks without modifying the underlying model. 
Experiments show that the proposed method improves generation quality across model architectures, guidance strengths, and reduced sampling budgets, while prior guidance strategies provide limited gains in common guidance settings. This work highlights how diffusion dynamics and domain-specific statistics can be combined to establish a principled criterion for trajectory correction.

\bibliography{iclr2027_conference}

@article{lipman2022flow,
  title={Flow matching for generative modeling},
  author={Lipman, Yaron and Chen, Ricky TQ and Ben-Hamu, Heli and Nickel, Maximilian and Le, Matt},
  journal={arXiv preprint arXiv:2210.02747},
  year={2022}
}

@article{liu2022flow,
  title={Flow straight and fast: Learning to generate and transfer data with rectified flow},
  author={Liu, Xingchao and Gong, Chengyue and Liu, Qiang},
  journal={arXiv preprint arXiv:2209.03003},
  year={2022}
}

@article{ho2020denoising,
  title={Denoising diffusion probabilistic models},
  author={Ho, Jonathan and Jain, Ajay and Abbeel, Pieter},
  journal={Advances in neural information processing systems},
  volume={33},
  pages={6840--6851},
  year={2020}
}

@article{efron2011tweedie,
  title={Tweedie’s formula and selection bias},
  author={Efron, Bradley},
  journal={Journal of the American Statistical Association},
  volume={106},
  number={496},
  pages={1602--1614},
  year={2011},
  publisher={Taylor \& Francis}
}

@article{song2020denoising,
  title={Denoising diffusion implicit models},
  author={Song, Jiaming and Meng, Chenlin and Ermon, Stefano},
  journal={arXiv preprint arXiv:2010.02502},
  year={2020}
}

@article{song2019generative,
  title={Generative modeling by estimating gradients of the data distribution},
  author={Song, Yang and Ermon, Stefano},
  journal={Advances in neural information processing systems},
  volume={32},
  year={2019}
}

@article{luo2022understanding,
  title={Understanding diffusion models: A unified perspective},
  author={Luo, Calvin},
  journal={arXiv preprint arXiv:2208.11970},
  year={2022}
}

@article{lai2025principles,
  title={The principles of diffusion models},
  author={Lai, Chieh-Hsin and Song, Yang and Kim, Dongjun and Mitsufuji, Yuki and Ermon, Stefano},
  journal={arXiv preprint arXiv:2510.21890},
  year={2025}
}

@article{song2020score,
  title={Score-based generative modeling through stochastic differential equations},
  author={Song, Yang and Sohl-Dickstein, Jascha and Kingma, Diederik P and Kumar, Abhishek and Ermon, Stefano and Poole, Ben},
  journal={arXiv preprint arXiv:2011.13456},
  year={2020}
}

@inproceedings{sohl2015deep,
  title={Deep unsupervised learning using nonequilibrium thermodynamics},
  author={Sohl-Dickstein, Jascha and Weiss, Eric and Maheswaranathan, Niru and Ganguli, Surya},
  booktitle={International conference on machine learning},
  pages={2256--2265},
  year={2015},
  organization={pmlr}
}

@article{dhariwal2021diffusion,
  title={Diffusion models beat gans on image synthesis},
  author={Dhariwal, Prafulla and Nichol, Alexander},
  journal={Advances in neural information processing systems},
  volume={34},
  pages={8780--8794},
  year={2021}
}

@inproceedings{nichol2021improved,
  title={Improved denoising diffusion probabilistic models},
  author={Nichol, Alexander Quinn and Dhariwal, Prafulla},
  booktitle={International conference on machine learning},
  pages={8162--8171},
  year={2021},
  organization={PMLR}
}

@article{karras2022elucidating,
  title={Elucidating the design space of diffusion-based generative models},
  author={Karras, Tero and Aittala, Miika and Aila, Timo and Laine, Samuli},
  journal={Advances in neural information processing systems},
  volume={35},
  pages={26565--26577},
  year={2022}
}

@inproceedings{peebles2023scalable,
  title={Scalable diffusion models with transformers},
  author={Peebles, William and Xie, Saining},
  booktitle={Proceedings of the IEEE/CVF international conference on computer vision},
  pages={4195--4205},
  year={2023}
}

@article{liu2023audioldm,
  title={Audioldm: Text-to-audio generation with latent diffusion models},
  author={Liu, Haohe and Chen, Zehua and Yuan, Yi and Mei, Xinhao and Liu, Xubo and Mandic, Danilo and Wang, Wenwu and Plumbley, Mark D},
  journal={arXiv preprint arXiv:2301.12503},
  year={2023}
}

@article{ho2022video,
  title={Video diffusion models},
  author={Ho, Jonathan and Salimans, Tim and Gritsenko, Alexey and Chan, William and Norouzi, Mohammad and Fleet, David J},
  journal={Advances in neural information processing systems},
  volume={35},
  pages={8633--8646},
  year={2022}
}

@article{chi2025diffusion,
  title={Diffusion policy: Visuomotor policy learning via action diffusion},
  author={Chi, Cheng and Xu, Zhenjia and Feng, Siyuan and Cousineau, Eric and Du, Yilun and Burchfiel, Benjamin and Tedrake, Russ and Song, Shuran},
  journal={The International Journal of Robotics Research},
  volume={44},
  number={10-11},
  pages={1684--1704},
  year={2025},
  publisher={Sage Publications Sage UK: London, England}
}

@article{saharia2022photorealistic,
  title={Photorealistic text-to-image diffusion models with deep language understanding},
  author={Saharia, Chitwan and Chan, William and Saxena, Saurabh and Li, Lala and Whang, Jay and Denton, Emily L and Ghasemipour, Kamyar and Gontijo Lopes, Raphael and Karagol Ayan, Burcu and Salimans, Tim and others},
  journal={Advances in neural information processing systems},
  volume={35},
  pages={36479--36494},
  year={2022}
}

@inproceedings{rombach2022high,
  title={High-resolution image synthesis with latent diffusion models},
  author={Rombach, Robin and Blattmann, Andreas and Lorenz, Dominik and Esser, Patrick and Ommer, Bj{\"o}rn},
  booktitle={Proceedings of the IEEE/CVF conference on computer vision and pattern recognition},
  pages={10684--10695},
  year={2022}
}

@article{podell2023sdxl,
  title={Sdxl: Improving latent diffusion models for high-resolution image synthesis},
  author={Podell, Dustin and English, Zion and Lacey, Kyle and Blattmann, Andreas and Dockhorn, Tim and M{\"u}ller, Jonas and Penna, Joe and Rombach, Robin},
  journal={arXiv preprint arXiv:2307.01952},
  year={2023}
}

@inproceedings{esser2024scaling,
  title={Scaling rectified flow transformers for high-resolution image synthesis},
  author={Esser, Patrick and Kulal, Sumith and Blattmann, Andreas and Entezari, Rahim and M{\"u}ller, Jonas and Saini, Harry and Levi, Yam and Lorenz, Dominik and Sauer, Axel and Boesel, Frederic and others},
  booktitle={Forty-first international conference on machine learning},
  year={2024}
}

@article{chen2023pixart,
  title={{PixArt}-$\alpha$: Fast Training of Diffusion Transformer for Photorealistic Text-to-Image Synthesis},
  author={Chen, Junsong and Yu, Jincheng and Ge, Chongjian and Yao, Lewei and Xie, Enze and Wu, Yue and Wang, Zhongdao and Kwok, James and Luo, Ping and Lu, Huchuan and others},
  journal={arXiv preprint arXiv:2310.00426},
  year={2023}
}

@inproceedings{chen2024pixart,
  title={Pixart-$\sigma$: Weak-to-strong training of diffusion transformer for 4k text-to-image generation},
  author={Chen, Junsong and Ge, Chongjian and Xie, Enze and Wu, Yue and Yao, Lewei and Ren, Xiaozhe and Wang, Zhongdao and Luo, Ping and Lu, Huchuan and Li, Zhenguo},
  booktitle={European Conference on Computer Vision},
  pages={74--91},
  year={2024},
  organization={Springer}
}

@article{xie2024sana,
  title={Sana: Efficient high-resolution image synthesis with linear diffusion transformers},
  author={Xie, Enze and Chen, Junsong and Chen, Junyu and Cai, Han and Tang, Haotian and Lin, Yujun and Zhang, Zhekai and Li, Muyang and Zhu, Ligeng and Lu, Yao and others},
  journal={arXiv preprint arXiv:2410.10629},
  year={2024}
}

@article{ho2022classifier,
  title={Classifier-free diffusion guidance},
  author={Ho, Jonathan and Salimans, Tim},
  journal={arXiv preprint arXiv:2207.12598},
  year={2022}
}

@inproceedings{sadat2024eliminating,
  title={Eliminating oversaturation and artifacts of high guidance scales in diffusion models},
  author={Sadat, Seyedmorteza and Hilliges, Otmar and Weber, Romann M},
  booktitle={The Thirteenth International Conference on Learning Representations},
  year={2025}
}

@article{kynkaanniemi2024applying,
  title={Applying guidance in a limited interval improves sample and distribution quality in diffusion models},
  author={Kynk{\"a}{\"a}nniemi, Tuomas and Aittala, Miika and Karras, Tero and Laine, Samuli and Aila, Timo and Lehtinen, Jaakko},
  journal={Advances in Neural Information Processing Systems},
  volume={37},
  pages={122458--122483},
  year={2024}
}

@article{karras2024guiding,
  title={Guiding a diffusion model with a bad version of itself},
  author={Karras, Tero and Aittala, Miika and Kynk{\"a}{\"a}nniemi, Tuomas and Lehtinen, Jaakko and Aila, Timo and Laine, Samuli},
  journal={Advances in Neural Information Processing Systems},
  volume={37},
  pages={52996--53021},
  year={2024}
}

@inproceedings{hong2023improving,
  title={Improving sample quality of diffusion models using self-attention guidance},
  author={Hong, Susung and Lee, Gyuseong and Jang, Wooseok and Kim, Seungryong},
  booktitle={Proceedings of the IEEE/CVF International Conference on Computer Vision},
  pages={7462--7471},
  year={2023}
}

@article{sadat2024no,
  title={No training, no problem: Rethinking classifier-free guidance for diffusion models},
  author={Sadat, Seyedmorteza and Kansy, Manuel and Hilliges, Otmar and Weber, Romann M},
  journal={arXiv preprint arXiv:2407.02687},
  year={2024}
}

@inproceedings{song2026rethinking,
  title={Rethinking oversaturation in classifier-free guidance via low frequency},
  author={Song, Kaiyu and Tong, Baoshun and Lai, Hanjiang},
  booktitle={ICASSP 2026-2026 IEEE International Conference on Acoustics, Speech and Signal Processing (ICASSP)},
  pages={10107--10111},
  year={2026},
  organization={IEEE}
}

@article{sadat2025guidance,
  title={Guidance in the frequency domain enables high-fidelity sampling at low cfg scales},
  author={Sadat, Seyedmorteza and Vontobel, Tobias and Salehi, Farnood and Weber, Romann M},
  journal={arXiv preprint arXiv:2506.19713},
  year={2025}
}

@article{chung2024cfg++,
  title={Cfg++: Manifold-constrained classifier free guidance for diffusion models},
  author={Chung, Hyungjin and Kim, Jeongsol and Park, Geon Yeong and Nam, Hyelin and Ye, Jong Chul},
  journal={arXiv preprint arXiv:2406.08070},
  year={2024}
}

@inproceedings{kwon2025tcfg,
  title={Tcfg: Tangential damping classifier-free guidance},
  author={Kwon, Mingi and Kim, Shin seong and Jeong, Jaeseok and Hsiao, Yi Ting and Uh, Youngjung},
  booktitle={Proceedings of the Computer Vision and Pattern Recognition Conference},
  pages={2620--2629},
  year={2025}
}

@inproceedings{jin2025angle,
  title={Angle Domain Guidance: Latent Diffusion Requires Rotation Rather Than Extrapolation},
  author={Jin, Cheng and Xiao, Zhenyu and Liu, Chutao and Gu, Yuantao},
  booktitle={International Conference on Machine Learning},
  year={2025},
  organization={PMLR}
}

@article{jin2025stage,
  title={Stage-wise dynamics of classifier-free guidance in diffusion models},
  author={Jin, Cheng and Shi, Qitan and Gu, Yuantao},
  journal={arXiv preprint arXiv:2509.22007},
  year={2025}
}

@article{malarz2025classifier,
  title={Classifier-free Guidance with Adaptive Scaling},
  author={Malarz, Dawid and Kasymov, Artur and Zi{\k{e}}ba, Maciej and Tabor, Jacek and Spurek, Przemys{\l}aw},
  journal={arXiv preprint arXiv:2502.10574},
  year={2025}
}

@article{ruderman1997origins,
  title   = {Origins of scaling in natural images},
  author  = {Ruderman, Daniel L.},
  journal = {Vision Research},
  year    = {1997}
}

@inproceedings{heusel2017gans,
  title     = {GANs Trained by a Two Time-Scale Update Rule Converge to a Local Nash Equilibrium},
  author    = {Heusel, Martin and Ramsauer, Hubert and Unterthiner, Thomas and Nessler, Bernhard and Hochreiter, Sepp},
  booktitle = {Advances in Neural Information Processing Systems},
  year      = {2017}
}

@inproceedings{salimans2016improved,
  title     = {Improved Techniques for Training GANs},
  author    = {Salimans, Tim and Goodfellow, Ian and Zaremba, Wojciech and Cheung, Vicki and Radford, Alec and Chen, Xi},
  booktitle = {Advances in Neural Information Processing Systems},
  year      = {2016}
}

@inproceedings{xu2023imagereward,
  title={ImageReward: learning and evaluating human preferences for text-to-image generation},
  author={Xu, Jiazheng and Liu, Xiao and Wu, Yuchen and Tong, Yuxuan and Li, Qinkai and Ding, Ming and Tang, Jie and Dong, Yuxiao},
  booktitle={Proceedings of the 37th International Conference on Neural Information Processing Systems},
  pages={15903--15935},
  year={2023}
}

@article{kirstain2023pick,
  title={Pick-a-pic: An open dataset of user preferences for text-to-image generation},
  author={Kirstain, Yuval and Polyak, Adam and Singer, Uriel and Matiana, Shahbuland and Penna, Joe and Levy, Omer},
  journal={Advances in neural information processing systems},
  volume={36},
  pages={36652--36663},
  year={2023}
}

@article{wu2023human,
  title={Human Preference Score v2: A Solid Benchmark for Evaluating Human Preferences of Text-to-Image Synthesis},
  author={Wu, Xiaoshi and Hao, Yiming and Sun, Keqiang and Chen, Yixiong and Zhu, Feng and Zhao, Rui and Li, Hongsheng},
  journal={arXiv preprint arXiv:2306.09341},
  year={2023}
}

@inproceedings{radford2021learning,
  title={Learning transferable visual models from natural language supervision},
  author={Radford, Alec and Kim, Jong Wook and Hallacy, Chris and Ramesh, Aditya and Goh, Gabriel and Agarwal, Sandhini and Sastry, Girish and Askell, Amanda and Mishkin, Pamela and Clark, Jack and others},
  booktitle={International conference on machine learning},
  pages={8748--8763},
  year={2021},
  organization={PmLR}
}

@article{nash2021generating,
  title={Generating images with sparse representations},
  author={Nash, Charlie and Menick, Jacob and Dieleman, Sander and Battaglia, Peter W},
  journal={arXiv preprint arXiv:2103.03841},
  year={2021}
}

@inproceedings{lin2014microsoft,
  title={Microsoft coco: Common objects in context},
  author={Lin, Tsung-Yi and Maire, Michael and Belongie, Serge and Hays, James and Perona, Pietro and Ramanan, Deva and Doll{\'a}r, Piotr and Zitnick, C Lawrence},
  booktitle={European conference on computer vision},
  pages={740--755},
  year={2014},
  organization={Springer}
}

@inproceedings{deng2009imagenet,
  title={Imagenet: A large-scale hierarchical image database},
  author={Deng, Jia and Dong, Wei and Socher, Richard and Li, Li-Jia and Li, Kai and Fei-Fei, Li},
  booktitle={2009 IEEE conference on computer vision and pattern recognition},
  pages={248--255},
  year={2009},
  organization={Ieee}
}

@article{schuhmann2022laion,
  title={Laion-5b: An open large-scale dataset for training next generation image-text models},
  author={Schuhmann, Christoph and Beaumont, Romain and Vencu, Richard and Gordon, Cade and Wightman, Ross and Cherti, Mehdi and Coombes, Theo and Katta, Aarush and Mullis, Clayton and Wortsman, Mitchell and others},
  journal={Advances in neural information processing systems},
  volume={35},
  pages={25278--25294},
  year={2022}
}

@article{karras2021alias,
  title={Alias-free generative adversarial networks},
  author={Karras, Tero and Aittala, Miika and Laine, Samuli and H{\"a}rk{\"o}nen, Erik and Hellsten, Janne and Lehtinen, Jaakko and Aila, Timo},
  journal={Advances in neural information processing systems},
  volume={34},
  pages={852--863},
  year={2021}
}

@inproceedings{papalampidi2025dynamic,
  title={Dynamic classifier-free diffusion guidance via online feedback},
  author={Papalampidi, Pinelopi and Wiles, Olivia and Ktena, Ira and Shtedritski, Aleksandar and Bugliarello, Emanuele and Kajic, Ivana and Albuquerque, Isabela and Nematzadeh, Aida},
  booktitle={International Conference on Learning Representations},
  year={2026}
}

@inproceedings{galashov2026learn,
  title={Learn to Guide Your Diffusion Model},
  author={Galashov, Alexandre and Pokle, Ashwini and Doucet, Arnaud and Gretton, Arthur and Delbracio, Mauricio and De Bortoli, Valentin},
  booktitle={International Conference on Learning Representations},
  year={2026}
}

@article{mallat1989theory,
  title={A Theory for Multiresolution Signal Decomposition: The Wavelet Representation},
  author={Mallat, SG},
  journal={IEEE Transactions on Pattern Analysis and Machine Intelligence},
  volume={11},
  number={7},
  pages={674--693},
  year={1989},
  publisher={IEEE Computer Society Washington, DC, USA}
}

@article{li2025back,
  title={Back to basics: Let denoising generative models denoise},
  author={Li, Tianhong and He, Kaiming},
  journal={arXiv preprint arXiv:2511.13720},
  year={2025}
}

@inproceedings{ma2024sit,
  title={Sit: Exploring flow and diffusion-based generative models with scalable interpolant transformers},
  author={Ma, Nanye and Goldstein, Mark and Albergo, Michael S and Boffi, Nicholas M and Vanden-Eijnden, Eric and Xie, Saining},
  booktitle={European Conference on Computer Vision},
  pages={23--40},
  year={2024},
  organization={Springer}
}

@article{zhang2024frecas,
  title={Frecas: Efficient higher-resolution image generation via frequency-aware cascaded sampling},
  author={Zhang, Zhengqiang and Li, Ruihuang and Zhang, Lei},
  journal={arXiv preprint arXiv:2410.18410},
  year={2024}
}

@InProceedings{Ma_2025_ICCV,
    author    = {Ma, Yuhang and Wu, Xiaoshi and Sun, Keqiang and Li, Hongsheng},
    title     = {HPSv3: Towards Wide-Spectrum Human Preference Score},
    booktitle = {Proceedings of the IEEE/CVF International Conference on Computer Vision (ICCV)},
    month     = {October},
    year      = {2025},
    pages     = {15086-15095}
}
\bibliographystyle{iclr2027_conference}

\newpage
\appendix
\section*{Appendix}

\section{Reproducibility Statement}
Section~\ref{sec:method_spec} presents the formulation of the proposed method, and Algorithm~\ref{alg:spectrum_ratio_simple} summarizes the correction procedure.
Appendix~\ref{app:proof-proposition-1} provides the proof of Proposition~\ref{prop:power_law_intermediate}.
Appendix~\ref{app:implementation-details} details spectral coefficient estimation, hyperparameter settings, experimental configurations, and computational resources to facilitate reproduction of the results.

\section{LLM Usage Statement}
We used generative AI tools in auxiliary roles to support and polish the writing of the manuscript and to assist with implementing and debugging experimental code. All AI-assisted outputs were reviewed by the authors.

\section{Ethics Statement}
\label{app:societal-impacts}

Generative modeling can support creative and practical applications, including image editing, design, and content creation.
At the same time, improvements in generation quality can increase the risk of misuse, including harmful synthetic content.
We therefore emphasize the need for responsible deployment policies.

\section{Limitations and Future Work}
\label{app:limitations}

The proposed method demonstrates that spectral alignment based on power spectra can provide a useful criterion for guiding diffusion trajectories and improving generation quality.
However, this criterion only captures marginal frequency-domain statistics and does not fully characterize the generated distribution.
Matching power spectra can align the overall spectral energy across frequencies, but it does not account for phase information or spatial structure.
Future work could extend spectral correction beyond power spectra by using wavelet representations~\citep{mallat1989theory} to capture localized frequency structure and by incorporating constraints that account for phase and individual sample structure.
Another limitation is that our correction is applied with a shared rule across samples and conditions.
In practice, the amount and type of guidance needed may depend on the input condition, prompt, or class.
Guidance strategies that adapt to the input condition~\citep{galashov2026learn,papalampidi2025dynamic} suggest a promising direction for extending spectral correction to account for differences across samples and conditions.
Combining spectral alignment with condition-aware guidance control could enable more flexible adjustment of the denoising trajectory.


\section{Additional Implementation Details}
\label{app:implementation-details}

\paragraph{Baseline implementation details.}
We implement each baseline according to its recommended setting whenever applicable.\footnote{
Backbone models: stabilityai/stable-diffusion-xl-base-1.0, PixArt-alpha/PixArt-XL-2-1024-MS, stabilityai/stable-diffusion-3.5-medium, facebook/DiT-XL-2-256, JiT-L/16, and SiT-XL/2-256.
}
For TV-CFG, we set the time-varying guidance schedule such that it peaks at the midpoint of the denoising process and has the same average guidance scale as the CFG guidance scale in the corresponding experiment.
For LF-CFG, we follow the original guideline and implement the low-frequency correction using upsampling and downsampling with an 8$\times$ scaling ratio and a low-frequency attenuation factor of $\rho=0.5$.
For CFG++, we use a guidance scale of $w/12.5$ based on the approximate scale correspondence reported in prior work~\citep{chung2024cfg++,jin2025angle}. Prompts are shown exactly as used during generation, including typographical errors in the original COCO captions.





\paragraph{Implementation details of our method.}
We apply the proposed correction at every denoising step, as described in Algorithm~\ref{alg:spectrum_ratio_simple}.
For text-to-image generation, we use 64 radial bins.
The correction strength is set to $\lambda=0.05$ for SDXL and PixArt-$\alpha$, and $\lambda=0.1$ for SD3.5.
The relaxation factor is set to $\rho=0.5$ for SDXL and SD3.5. For PixArt-$\alpha$ at the higher guidance scale, we use $\lambda=0.1$ and $\rho=0.5$, reflecting the need for stronger correction under stronger guidance.
We estimate the target spectrum using the same inference-time sampling schedule as the one used for generation. For the DiT-XL/2 class-conditional experiments, we use 32 radial bins and set $\lambda=0.01$ across all sampling budgets. We add a small constant $\varepsilon$ to the numerator and denominator of the spectrum ratio for numerical stability. 

To model the $z_0$ space of each model, we encode COCO val 50K images~\citep{lin2014microsoft} into the corresponding latent space and fit the power-law curve based on Eq.~\ref{eq:Ruderman}.
We note that this fitting step is not essential to the method. It is a practical choice for setting the reference spectrum in our experiments.
Since the reference spectrum is parameterized by only three scalar values, the corresponding power-law parameters can also be tuned without using external data. 

For the conceptual validation in Section~\ref{sec:spectral_dynamics}, we use 1,000 image--caption pairs from COCO val 50K. We use SDXL with 50 denoising steps, resolution $1024\times1024$, and fp16 inference. For each sample, we encode the real image into the SDXL VAE latent space and measure the spectrum along the forward noising trajectory. We then run reverse denoising from random Gaussian noise using the corresponding caption as the prompt, and measure the spectrum of the generated trajectory. We evaluate CFG scales $\{2.5, 5.0, 7.5, 15.0\}$, together with unconditional and conditional baselines. To obtain the reported curve, we first average the spectra over the 1,000 samples for each trajectory type, and then compute the mean squared error between the log-power spectra. Thus, the comparison is performed between aggregate spectra rather than as a pairwise per-sample distance. 

Our method is training-free and does not require additional model training.
All reported experiments were conducted on NVIDIA RTX 3090 GPUs.
For class-conditional ImageNet experiments, generating 50K samples was parallelized over 15 GPUs and took approximately 8 hours.
For text-to-image experiments, generating 10K samples was parallelized over 5 GPUs and took approximately 12 hours.


\begin{algorithm}[t]
\caption{Spectral Correction Guidance}
\label{alg:spectrum_ratio_simple}
\begin{algorithmic}[1]
\Require condition $y$, CFG scale $w$, denoising steps $T$, correction strength $\lambda$,
target spectra $\{T_t(\omega)\}_{t=0}^{T-1}$, radial bin count $B$
\State Initialize latent $z_T \sim \mathcal{N}(0,I)$
\For{$t = T,\dots,1$}
    \State $\epsilon_{\varnothing} \gets \epsilon_\theta(z_t,t,\varnothing)$,
    \quad
    $\epsilon_y \gets \epsilon_\theta(z_t,t,y)$
    \State $\hat{\epsilon}_t \gets \epsilon_{\varnothing}
    + w(\epsilon_y-\epsilon_{\varnothing})$
    \State $z_{t-1}
    \gets \mathrm{Update}(z_t,\hat{\epsilon}_t,t)$
    \State $\hat{z}_{t-1} \gets \mathcal{F}(z_{t-1})$
    \State Compute radial power spectrum $S_{t-1}(\omega_k)$ from $\hat{z}_{t-1}$ for $k=1,\dots,B$
    \State Compute corrected power spectrum:
    \Statex
    \[
    \tilde{S}_{t-1}(\omega_k)
    =
    S_{t-1}(\omega_k)
    \left(
    \frac{T_{t-1}(\omega_k)+\varepsilon}
         {S_{t-1}(\omega_k)+\varepsilon}
    \right)^{\lambda},
    \quad k=1,\dots,B
    \]
    \State Compute corrected Fourier magnitude:
    \Statex
    \[
    M_{t-1}(\omega_k)
    =
    \sqrt{\tilde{S}_{t-1}(\omega_k)},
    \quad k=1,\dots,B
    \]
    \State Set the Fourier magnitude to $M_{t-1}$ while preserving the phase of $\hat{z}_{t-1}$:
    \Statex
    \[
    |\hat{\tilde{z}}_{t-1}| = M_{t-1},
    \quad
    \angle \hat{\tilde{z}}_{t-1} = \angle \hat{z}_{t-1}
    \]
    \State $z_{t-1} \gets \mathcal{F}^{-1}(\hat{\tilde{z}}_{t-1})$
\EndFor
\State \Return $\mathrm{Decode}(z_0)$
\end{algorithmic}
\end{algorithm}

\section{Proof of Proposition~\ref{prop:power_law_intermediate}}
\label{app:proof-proposition-1}

\begin{proof}
Starting from the forward process in Eq.~\ref{eq:zt_def},
\[
z_t=\alpha_t z_0+\sigma_t\epsilon,
\]
we apply the Fourier transform. By linearity,
\[
\hat{z}_t(\omega)
=
\alpha_t \hat{z}_0(\omega)
+
\sigma_t \hat{\epsilon}(\omega).
\]
Since the power spectrum is defined as
\(P_t(\omega)=|\hat{z}_t(\omega)|^2\), we have
\[
P_t(\omega)
=
\left|
\alpha_t \hat{z}_0(\omega)
+
\sigma_t \hat{\epsilon}(\omega)
\right|^2 .
\]
Expanding the squared magnitude gives
\[
P_t(\omega)
=
\alpha_t^2 P_0(\omega)
+
\sigma_t^2 P_\epsilon(\omega)
+
\alpha_t\sigma_t C_t(\omega),
\]
where \(C_t(\omega)\) denotes the cross term between
\(\hat{z}_0(\omega)\) and \(\hat{\epsilon}(\omega)\).
Taking expectation on both sides yields
\[
\mathbb{E}[P_t(\omega)]
=
\alpha_t^2 \mathbb{E}[P_0(\omega)]
+
\sigma_t^2 \mathbb{E}[P_\epsilon(\omega)]
+
\alpha_t\sigma_t \mathbb{E}[C_t(\omega)].
\]
Since \(z_0\) and \(\epsilon\) are independent and \(\mathbb{E}[\epsilon]=0\), the cross term vanishes in expectation.
Therefore,
\begin{equation}
\label{eq:expected_pt_decomposition}
\mathbb{E}[P_t(\omega)]
=
\alpha_t^2 \mathbb{E}[P_0(\omega)]
+
\sigma_t^2 \mathbb{E}[P_\epsilon(\omega)] .
\end{equation}

Substituting the assumed form of the clean-image spectrum,
\[
\mathbb{E}[P_0(\omega)] = C\omega^{-\beta}+c,
\]
and the flat spectrum of white Gaussian noise,
\[
\mathbb{E}[P_\epsilon(\omega)] = c_\epsilon,
\]
into Eq.~\ref{eq:expected_pt_decomposition}, we obtain
\[
\mathbb{E}[P_t(\omega)]
=
\alpha_t^2(C\omega^{-\beta}+c)
+
\sigma_t^2 c_\epsilon .
\]
Expanding the terms gives
\[
\mathbb{E}[P_t(\omega)]
=
\alpha_t^2 C\omega^{-\beta}
+
\alpha_t^2 c
+
\sigma_t^2 c_\epsilon .
\]
Finally, defining
\[
\lambda_t := \alpha_t^2 C,
\qquad
\gamma_t := \alpha_t^2 c+\sigma_t^2 c_\epsilon,
\]
gives
\[
\mathbb{E}[P_t(\omega)]
=
\lambda_t \omega^{-\beta}
+
\gamma_t .
\]
This completes the proof.
\end{proof}

\clearpage 

\section{Computational Cost}
\label{app:computational_cost}

To better understand the computational overhead of the proposed correction, we compare the peak GPU memory usage and per-image generation time of our method with those of standard CFG over 100 images.
As shown in Table~\ref{tab:computational_cost}, peak memory usage remains comparable.
The proposed correction requires no backpropagation or additional denoiser evaluations.
It adds approximately 1.7 seconds per image on SDXL and 2.6 seconds on SD3.5.
The reduced-NFE results in Table~\ref{tab:nfe_ablation} suggest that reducing the number of denoising steps could offset this overhead while maintaining generation quality.

\begin{table}[h]
\centering
\caption{Computational cost of standard CFG and our method, measured over 100 images.}
\label{tab:computational_cost}
\resizebox{0.5\linewidth}{!}{
\begin{tabular}{llcc}
\toprule
Model & Method & Peak VRAM (GB) & Time (s/image) \\
\midrule\midrule
\multirow{2}{*}{SDXL}
& \cellcolor{gray!20}Ours
& \cellcolor{gray!20}12.70
& \cellcolor{gray!20}15.442 \\
& CFG & 12.84 & 13.775 \\
\midrule
\multirow{2}{*}{SD3.5}
& \cellcolor{gray!20}Ours
& \cellcolor{gray!20}12.68
& \cellcolor{gray!20}25.236 \\
& CFG & 12.67 & 22.628 \\
\bottomrule
\end{tabular}
}
\end{table}

\section{Human Evaluation}
\label{app:human_evaluation}

We conduct a human preference study comparing our method with standard CFG on SDXL.
We randomly sample 50 prompts and ask 20 participants to compare the generated images in terms of visual quality and prompt alignment.
For each prompt, we randomize the left-right placement of the two methods to reduce positional bias.
As shown in Table~\ref{tab:human_evaluation}, our method receives 595 out of 1,000 votes for visual quality and 588 out of 1,000 votes for prompt alignment, corresponding to win rates of 59.50\% and 58.80\%, respectively.
Standard CFG already produces high-quality images, yet participants prefer the outputs of our method on both criteria.
These gains come from sampling-time correction without additional training or changes to the underlying model.

\begin{table}[h]
\centering
\caption{Human preference study comparing our method with CFG on SDXL. Each criterion receives 1,000 votes from 20 participants evaluating 50 prompts.}
\label{tab:human_evaluation}
\resizebox{0.5\linewidth}{!}{
\begin{tabular}{lccc}
\toprule
  & Ours & CFG & Win Rate (Ours) \\
\midrule\midrule
Visual Quality & \textbf{595} & 405 & \textbf{59.50\%} \\
Prompt Alignment & \textbf{588} & 412 & \textbf{58.80\%} \\
\bottomrule
\end{tabular}
}
\end{table}

\section{Robustness Across Random Seeds}
\label{app:multi_seed}

To assess robustness to random seed variation, we compare our method with CFG on SDXL across 10 seeds using 1K randomly sampled prompts per seed.
We report the mean and standard deviation of each metric across seeds.
As shown in Table~\ref{tab:multi_seed}, our method achieves higher mean HPSv3, ImageReward, PickScore, and HPSv2 scores and lower mean FID than CFG.
CFG obtains slightly better mean CLIP-T and sFID scores.
The gains in HPSv3, ImageReward, and HPSv2 are large relative to the observed standard deviations, supporting the robustness of these improvements across random seeds.

\begin{table}[h]
\centering
\caption{Multi-seed comparison of our method and CFG on SDXL using 1K prompts per seed. Results are reported as mean $\pm$ standard deviation across 10 random seeds.}
\label{tab:multi_seed}
\resizebox{0.98\linewidth}{!}{
\begin{tabular}{lccccccc}
\toprule
Method & HPSv3$\uparrow$ & ImageReward$\uparrow$ & PickScore$\uparrow$ & HPSv2$\uparrow$ & CLIP-T$\uparrow$ & FID$\downarrow$ & sFID$\downarrow$ \\
\midrule\midrule
\cellcolor{gray!20}Ours
& \cellcolor{gray!20}$\mathbf{8.375 \pm 0.045}$
& \cellcolor{gray!20}$\mathbf{0.7368 \pm 0.0088}$
& \cellcolor{gray!20}$\mathbf{22.463 \pm 0.021}$
& \cellcolor{gray!20}$\mathbf{0.2800 \pm 0.0004}$
& \cellcolor{gray!20}$0.2651 \pm 0.0004$
& \cellcolor{gray!20}$\mathbf{66.572 \pm 0.557}$
& \cellcolor{gray!20}$344.398 \pm 0.434$ \\
CFG
& $8.068 \pm 0.040$
& $0.7020 \pm 0.0137$
& $22.449 \pm 0.014$
& $0.2750 \pm 0.0003$
& $\mathbf{0.2660 \pm 0.0005}$
& $67.201 \pm 0.504$
& $\mathbf{343.793 \pm 0.861}$ \\
\bottomrule
\end{tabular}
}
\end{table}

\section{Comparison with Output-Only Correction}
\label{app:output_only_correction}

To examine whether the gains can be explained by spectral correction of the final output alone, we compare our method with an output-only correction baseline on SDXL using 1K randomly sampled prompts.
This baseline applies spectral correction only after generation, treating it as a post-processing step.
As shown in Table~\ref{tab:output_only_correction}, output-only correction yields modest improvements over CFG in preference-based metrics.
Our method achieves larger gains in HPSv3, ImageReward, PickScore, and HPSv2, together with lower FID and sFID.
These results support the benefit of correcting spectral deviations during sampling beyond the improvements obtained by adjusting the final spectrum alone.

\begin{table}[h]
\centering
\caption{Comparison with output-only spectral correction on SDXL using 1K randomly sampled prompts.}
\label{tab:output_only_correction}
\resizebox{0.9\linewidth}{!}{
\begin{tabular}{lccccccc}
\toprule
Method & HPSv3$\uparrow$ & ImageReward$\uparrow$ & PickScore$\uparrow$ & HPSv2$\uparrow$ & CLIP-T$\uparrow$ & FID$\downarrow$ & sFID$\downarrow$ \\
\midrule\midrule
\cellcolor{gray!20}Ours
& \cellcolor{gray!20}\textbf{8.437}
& \cellcolor{gray!20}\textbf{0.738}
& \cellcolor{gray!20}\textbf{22.499}
& \cellcolor{gray!20}\textbf{0.280}
& \cellcolor{gray!20}0.265
& \cellcolor{gray!20}\textbf{66.818}
& \cellcolor{gray!20}\textbf{345.093} \\
Output-only correction
& 8.218
& 0.710
& 22.482
& 0.276
& \textbf{0.267}
& 67.557
& 345.246 \\
CFG
& 8.104
& 0.708
& 22.473
& 0.275
& \textbf{0.267}
& 67.430
& 345.259 \\
\bottomrule
\end{tabular}
}
\end{table}

\section{Non-Photorealistic and Stylized Generation}
\label{app:stylized_generation}

To evaluate our method on stylized generation, we append ``anime illustration'' or ``oil painting'' to each of 1K prompts and generate images using SDXL.
As shown in Table~\ref{tab:stylized_generation}, our method improves HPSv3, ImageReward, PickScore, and HPSv2 over standard CFG for both styles.
CLIP-T improves for anime illustration and remains unchanged at the reported precision for oil painting.
These findings are consistent with the motivation for correction relaxation discussed in Section~\ref{sec:abl}.
The correction constrains spectral deviations during earlier sampling stages and relaxes toward the end, reducing spectral constraints on the formation of style-specific details.
The results support the applicability of the proposed correction beyond standard photorealistic generation.

\begin{table}[h]
\centering
\caption{Comparison of our method and CFG on SDXL for stylized generation using 1K prompts per style.}
\label{tab:stylized_generation}
\resizebox{0.9\linewidth}{!}{
\begin{tabular}{llccccc}
\toprule
Style & Method & HPSv3$\uparrow$ & ImageReward$\uparrow$ & PickScore$\uparrow$ & HPSv2$\uparrow$ & CLIP-T$\uparrow$ \\
\midrule\midrule
\multirow{2}{*}{Anime illustration}
& \cellcolor{gray!20}Ours
& \cellcolor{gray!20}\textbf{8.065}
& \cellcolor{gray!20}\textbf{1.045}
& \cellcolor{gray!20}\textbf{23.357}
& \cellcolor{gray!20}\textbf{0.317}
& \cellcolor{gray!20}\textbf{0.337} \\
& CFG
& 7.995 & 1.011 & 23.331 & 0.314 & 0.336 \\
\midrule
\multirow{2}{*}{Oil painting}
& \cellcolor{gray!20}Ours
& \cellcolor{gray!20}\textbf{8.646}
& \cellcolor{gray!20}\textbf{1.016}
& \cellcolor{gray!20}\textbf{22.822}
& \cellcolor{gray!20}\textbf{0.300}
& \cellcolor{gray!20}\textbf{0.320} \\
& CFG
& 8.473 & 1.011 & 22.775 & 0.294 & \textbf{0.320} \\
\bottomrule
\end{tabular}
}
\end{table}

\section{Additional Experiments across Guidance Settings}
\label{app:guidance-settings}

Section~\ref{sec:t2i_generation} evaluates text-to-image generation under the standard and high guidance settings.
We further examine whether the proposed correction remains effective when the guidance scale is reduced.
Table~\ref{tab:scale_abl} reports the results on SDXL at $w=2.5$, which is half of the standard guidance scale used in our main experiments.
Our method achieves the best performance on all metrics except sFID, improving preference-based scores and FID over standard CFG.
In contrast, existing guidance correction methods often degrade performance compared to CFG in this low-guidance regime. This suggests that correction strategies designed mainly to suppress over-guidance artifacts do not effectively address settings where the guidance signal is weak, whereas our spectral correction remains beneficial across guidance settings.

Figure~\ref{fig:lower_guidance_sdxl} shows qualitative examples under reduced guidance.
Generation at this guidance scale is challenging, and the images are not always fully satisfactory across methods.
Nevertheless, our method shows relatively improved prompt alignment and visual coherence.
For example, it better reflects the racket held in one hand in the first prompt and the feeding interaction with the giraffe in the second prompt.
These results further support the robustness of the proposed correction across different guidance scales.




\begin{table}[h]
\centering
\caption{Quantitative comparison under reduced guidance on SDXL.
At half of the standard guidance scale ($w=2.5$), our method improves most metrics, while existing guidance correction methods provide limited benefit.}
\label{tab:scale_abl}
\resizebox{0.85\linewidth}{!}{
\begin{tabular}{llcccccc}
\toprule
Setting & Method & ImageReward$\uparrow$ & PickScore$\uparrow$ & HPSv2$\uparrow$ & CLIP-T$\uparrow$ & FID$\downarrow$ & sFID$\downarrow$ \\
\midrule\midrule
\multirow{5}{*}{$w=2.5$}
& Ours\cellcolor{gray!20}
& \cellcolor{gray!20}\textbf{0.391}
& \cellcolor{gray!20}\textbf{21.86}
& \cellcolor{gray!20}\textbf{0.252}
& \cellcolor{gray!20}\textbf{0.257}
& \cellcolor{gray!20}\textbf{18.28}
& \cellcolor{gray!20}52.82 \\
& CFG
& 0.362 & 21.83 & 0.249 & \textbf{0.257} & 18.76 & \textbf{52.26} \\
& CFG++
& 0.329 & 21.77 & 0.247 & \textbf{0.257} & 19.30 & 52.75 \\
& TV-CFG
& 0.104 & 21.43 & 0.235 & 0.253 & 24.85 & 56.01 \\
& LF-CFG
& 0.022 & 21.23 & 0.232 & 0.248 & 26.42 & 60.18 \\
\bottomrule
\end{tabular}
}
\end{table}

\begin{figure*}[t]
    \centering
    \setlength{\tabcolsep}{2pt}
    \renewcommand{\arraystretch}{1.0}
    \scriptsize

    \begin{tabular}{c ccccc}
        & (a) Ours & (b) CFG & (c) CFG++ & (d) TV-CFG & (e) LF-CFG \\[0.25em]

        \multirow{4}{*}{\rotatebox{90}{\textbf{SDXL}}}
        & \includegraphics[width=0.18\linewidth]{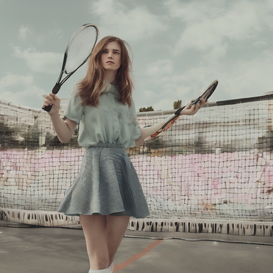}
        & \includegraphics[width=0.18\linewidth]{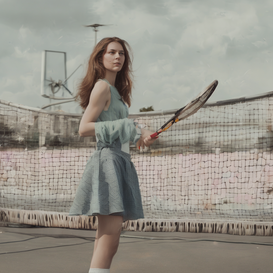}
        & \includegraphics[width=0.18\linewidth]{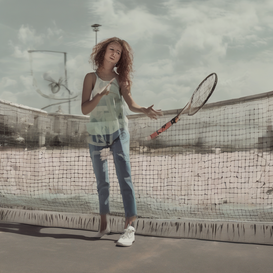}
        & \includegraphics[width=0.18\linewidth]{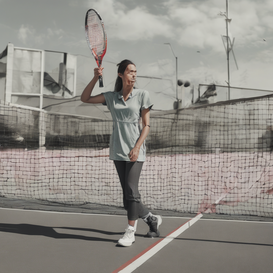}
        & \includegraphics[width=0.18\linewidth]{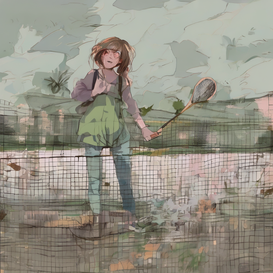} \\[-0.1em]
        & \multicolumn{5}{c}{\scriptsize \textbf{Prompt:} ``A girl with a racket in one hand.''} \\[0.45em]

        & \includegraphics[width=0.18\linewidth]{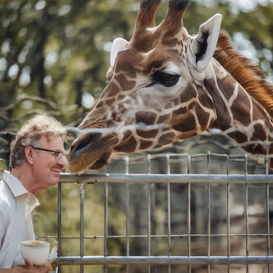}
        & \includegraphics[width=0.18\linewidth]{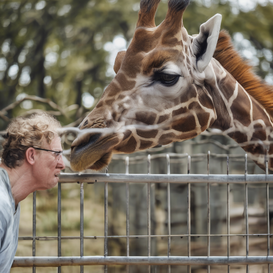}
        & \includegraphics[width=0.18\linewidth]{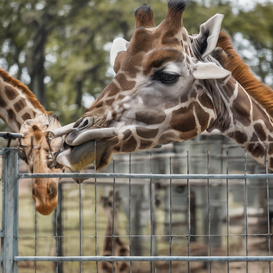}
        & \includegraphics[width=0.18\linewidth]{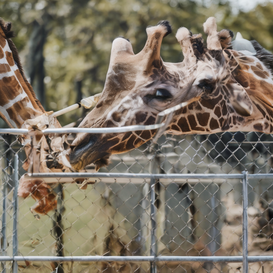}
        & \includegraphics[width=0.18\linewidth]{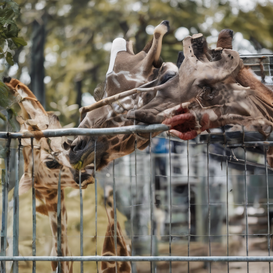} \\[-0.1em]
        & \multicolumn{5}{c}{\scriptsize \textbf{Prompt:} ``A man feeding a giraffe through a metal fence''}
    \end{tabular}

    \caption{Qualitative comparison under reduced guidance on SDXL.
    At half of the standard guidance scale ($w=2.5$), our method better preserves prompt-relevant content and visual coherence, while existing guidance correction methods provide limited qualitative improvement.}
    \label{fig:lower_guidance_sdxl}
\end{figure*}


\section{Additional Experiments on Diverse Backbone Models}
\label{app:diverse-backbones}

We further evaluate whether the proposed spectral correction can be applied beyond the DiT setting studied in Section~\ref{sec:t2i_generation}.
To this end, we conduct additional experiments on other class-conditional image generation backbones, including SiT and JiT.

\paragraph{SiT.}
We first evaluate our method on SiT-XL/2~\citep{ma2024sit}.
We use 32 radial bins for power spectrum estimation, with $\lambda=0.005$ and $\rho=0.5$.
Figure~\ref{fig:class_conditional_compare_sit} shows qualitative results.
The results suggest that our method can improve generation quality also on SiT.
Our method produces a more object-centric image for the European fire salamander class and reveals a clearer object shape for the Indian cobra class.

\paragraph{JiT.}
We also evaluate our method on JiT-L/16~\citep{li2025back}.
We use 128 radial bins for power spectrum estimation, with $\lambda=0.5$ and $\rho=0.5$.
Figure~\ref{fig:class_conditional_compare_jit} presents the qualitative comparison.
In the water ouzel example, our method produces a sharper image.
In the iguana example, class-specific visual details are more visible.
However, we observe that JiT exhibits larger sample-level variation than the other backbones.
Consistent with the ablation results in Section~\ref{sec:abl}, the effect of stronger correction is not uniform across samples, as it can introduce artifacts in some cases while producing only marginal changes in others.
This behavior may be related to JiT operating directly in the raw-pixel space, where spectral statistics can exhibit greater sample-level variability. The observation suggests that, although our method is motivated by spectral assumptions in pixel space, compressed latent spaces may provide a more stable domain for applying the correction by reducing sample-level spectral variability.





\begin{figure*}[t]
    \centering
    \setlength{\tabcolsep}{2pt}
    \renewcommand{\arraystretch}{1.0}

    \begin{tabular}{cccc}
        (a) CFG & (b) Ours & (c) CFG & (d) Ours \\[0.25em]

        \includegraphics[width=0.24\linewidth]{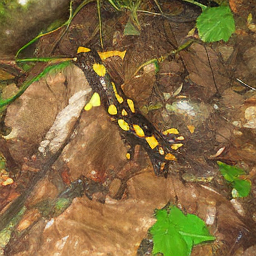} &
        \includegraphics[width=0.24\linewidth]{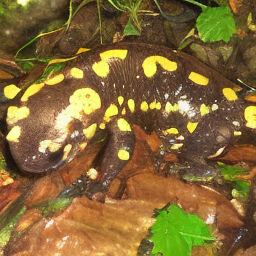} &
        \includegraphics[width=0.24\linewidth]{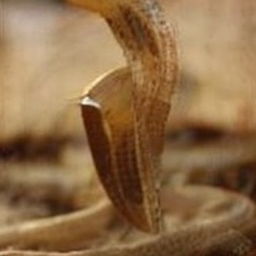} &
        \includegraphics[width=0.24\linewidth]{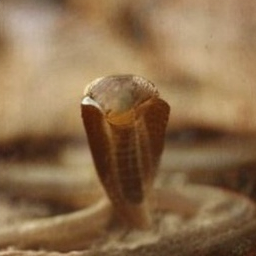} \\[-0.1em]

        \multicolumn{2}{c}{\small \textbf{Class:} ``European fire salamander''} &
        \multicolumn{2}{c}{\small \textbf{Class:} ``Indian cobra''}
    \end{tabular}

    \caption{Qualitative comparison on SiT for class-conditional generation.
Compared to CFG, our method yields clearer structures aligned with the target class and better preserves class-relevant visual attributes.}
    \label{fig:class_conditional_compare_sit}
\end{figure*}

\begin{figure*}[t]
    \centering
    \setlength{\tabcolsep}{2pt}
    \renewcommand{\arraystretch}{1.0}

    \begin{tabular}{cccc}
        (a) CFG & (b) Ours & (c) CFG & (d) Ours \\[0.25em]

        \includegraphics[width=0.24\linewidth]{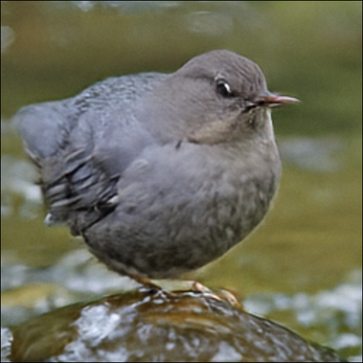} &
        \includegraphics[width=0.24\linewidth]{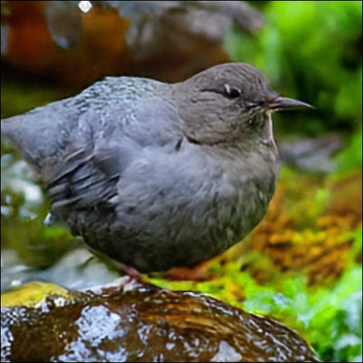} &
        \includegraphics[width=0.24\linewidth]{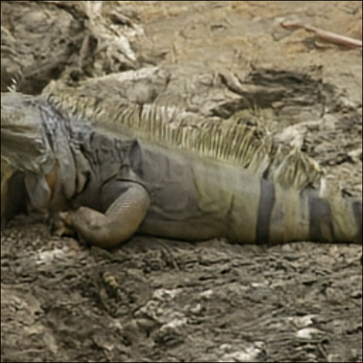} &
        \includegraphics[width=0.24\linewidth]{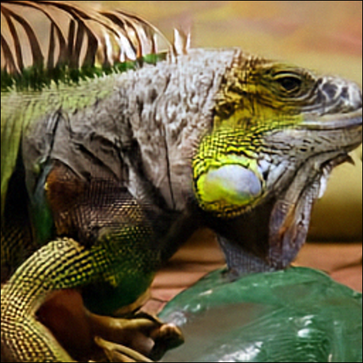} \\[-0.1em]

        \multicolumn{2}{c}{\small \textbf{Class:} ``Water ouzel''} &
        \multicolumn{2}{c}{\small \textbf{Class:} ``Iguana''}
    \end{tabular}

    \caption{Qualitative comparison on JiT for class-conditional generation.
Our method improves image sharpness and fine visual details, producing more recognizable samples than CFG.}
    \label{fig:class_conditional_compare_jit}
\end{figure*}

\section{Additional Analysis of Power Spectrum Modeling}
\label{app:power-spectrum-estimation}

Our method extends a spectral principle from pixel space to latent space. 
To empirically validate this assumption, we examine whether the power-law behavior of natural image spectra is preserved after encoding images into the latent spaces of the models used in our experiments.
Figure~\ref{fig:z0_power} shows the mean power spectrum of encoded COCO val 50K images~\citep{lin2014microsoft} in the latent space of each model, along with the fitted curve based on Eq.~\ref{eq:Ruderman}.
The fitted curves achieve $R^2 > 0.99$ across models, indicating that the $z_0$ spaces of the image generation models considered in our experiments can be well characterized by this power-law structure.
This result provides empirical justification for applying the proposed spectral correction in these spaces.




We also analyze the effect of binning resolution when estimating the power spectrum for the proposed correction.
The binning resolution determines how many frequency components are grouped into each radial bin.
To make the effect of this design choice more visible, we vary the number of radial bins under a relatively strong correction strength of $\lambda=0.5$.
Figure~\ref{fig:bin_ablation} shows the qualitative results.
Coarse binning, such as 8 radial bins, covers a broad range of frequencies within each bin and mainly produces global changes, such as changes in overall brightness.
By contrast, finer binning, such as 128 radial bins, produces more noticeable changes in image structure.
These results suggest that binning resolution is one of the design choices that can affect the behavior of spectral correction.

\begin{figure}
    \centering
    \includegraphics[width=1\linewidth]{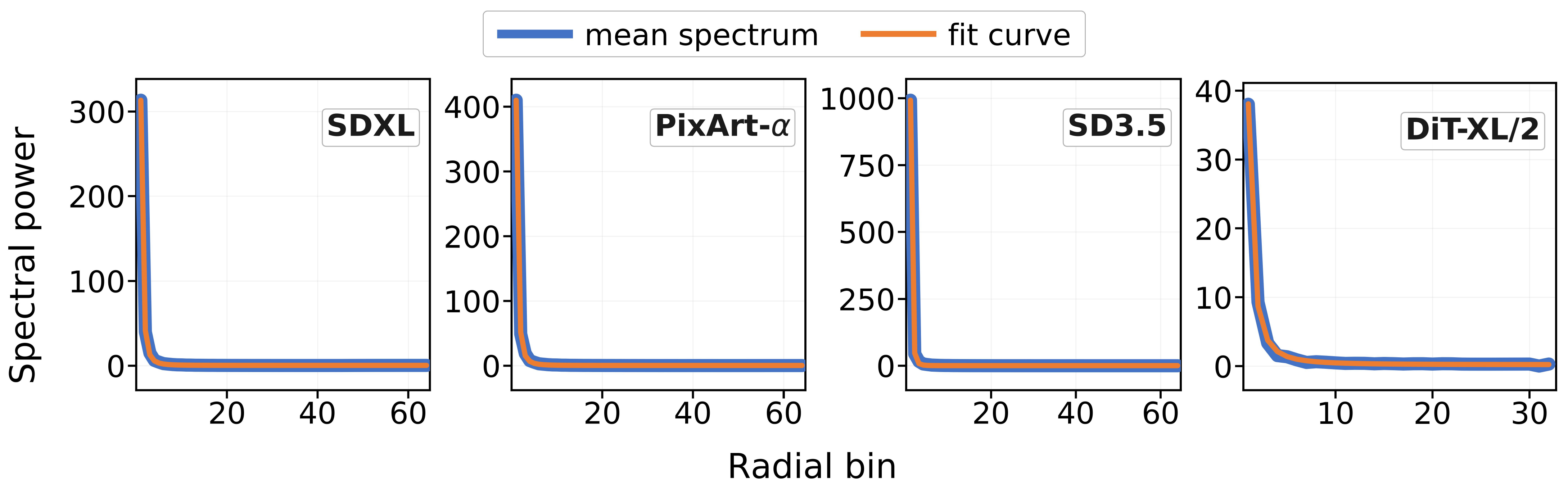}
\caption{Power spectrum of $z_0$ from baseline diffusion models with fitted curves.
Although the spectral scale varies across models, the spectra consistently follow a power-law statistical structure. The fit based on Eq.~\ref{eq:Ruderman} achieves $R^2 > 0.99$ across models.}
    \label{fig:z0_power}
\end{figure}

\begin{figure*}[t]
    \centering
    \setlength{\tabcolsep}{2pt}
    \renewcommand{\arraystretch}{1.0}
    \scriptsize

    \begin{tabular}{c cccc}
        & CFG & 8 bins & 32 bins & 128 bins \\[0.25em]

        \multirow{4}{*}{\rotatebox{90}{\textbf{SDXL}}}
        & \includegraphics[width=0.225\linewidth]{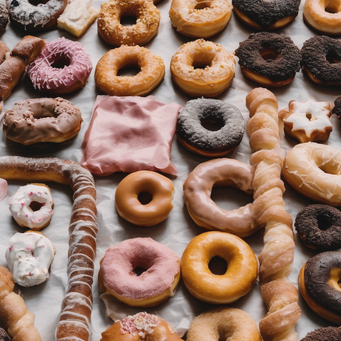}
        & \includegraphics[width=0.225\linewidth]{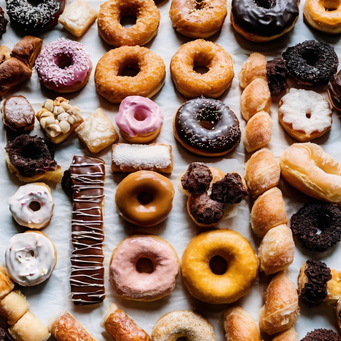}
        & \includegraphics[width=0.225\linewidth]{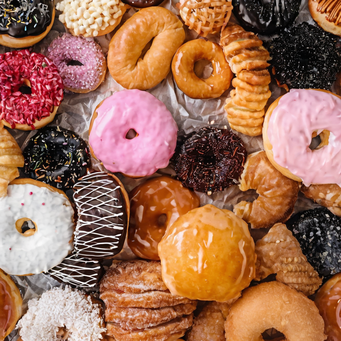}
        & \includegraphics[width=0.225\linewidth]{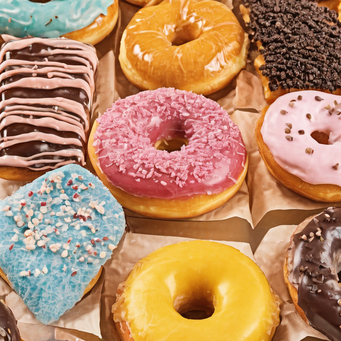} \\[-0.1em]
        & \multicolumn{4}{c}{\scriptsize \textbf{Prompt:} ``A selection of doughnuts and pastry sitting on waxed sheets''} \\[0.45em]

        & \includegraphics[width=0.225\linewidth]{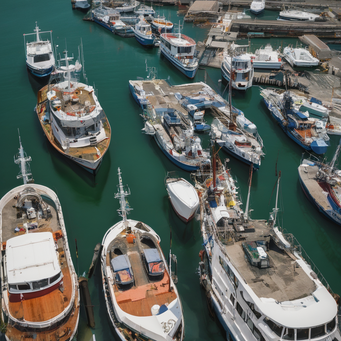}
        & \includegraphics[width=0.225\linewidth]{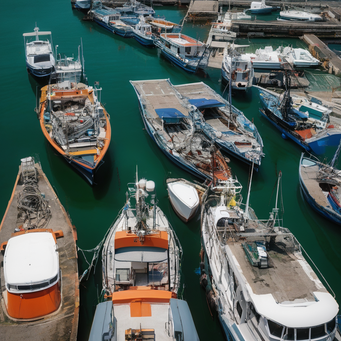}
        & \includegraphics[width=0.225\linewidth]{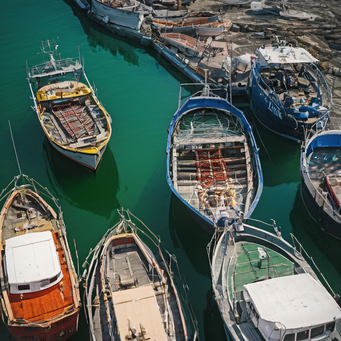}
        & \includegraphics[width=0.225\linewidth]{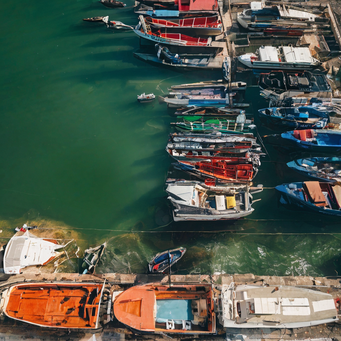} \\[-0.1em]
        & \multicolumn{4}{c}{\scriptsize \textbf{Prompt:} ``Several boats in the water docked at a port.''}
    \end{tabular}

    \caption{Ablation on the number of radial bins for power spectrum estimation on SDXL.
We use a relatively strong correction strength of $\lambda=0.5$ to better highlight the effect of binning.
}
    \label{fig:bin_ablation}
\end{figure*}

\newpage

\begin{figure}[t]
    \centering
    \begin{subfigure}[b]{0.3\linewidth}
        \centering
        \captionsetup{labelformat=empty}
        \caption{\quad (a) PickScore}
        \label{fig:rho_pickscore}
        \includegraphics[width=0.9\linewidth]{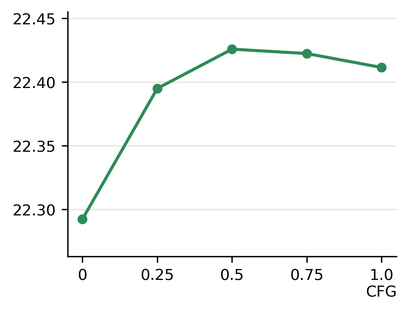}
    \end{subfigure}%
    \hfill
    \begin{subfigure}[b]{0.3\linewidth}
        \centering
        \captionsetup{labelformat=empty}
        \caption{\quad (b) FID}
        \label{fig:rho_fid}
        \includegraphics[width=0.9\linewidth]{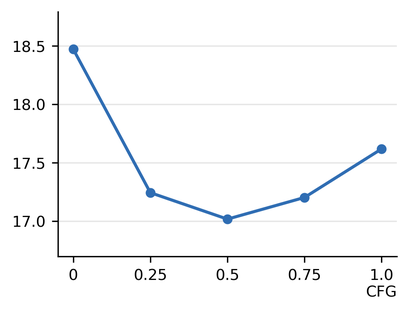}
    \end{subfigure}%
    \hfill
    \begin{subfigure}[b]{0.3\linewidth}
        \centering
        \captionsetup{labelformat=empty}
        \caption{\quad (c) sFID}
        \label{fig:rho_sfid}
        \includegraphics[width=0.9\linewidth]{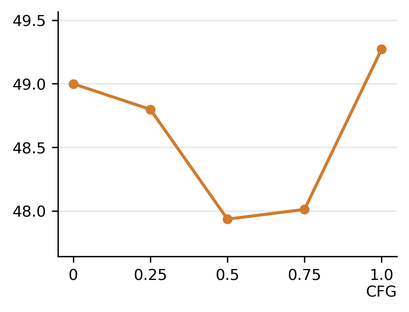}
    \end{subfigure}
    \caption{Effect of the relaxation factor $\rho$ on PickScore, FID, and sFID.
    The x-axis is the relaxation factor and the y-axis is the metric value.
    Overall, $\rho=0.5$ yields the most balanced performance across the three metrics.}
    \vspace{-0.2em}
    \label{fig:rho_ablation_metrics}
\end{figure}

\section{Additional Analysis of Correction Relaxation}
\label{app:correction-relaxation}
\vspace{-0.2em}
We further analyze the effect of the relaxation factor $\rho$ used in the proposed correction.
Figure~\ref{fig:rho_ablation_metrics} shows the results on SDXL, where we vary $\rho$ and evaluate PickScore, FID, and sFID.
The value $\rho=0.5$ provides the most balanced results across the three metrics.
This provides quantitative support for the role of correction relaxation in balancing alignment and fidelity.
We also observe that the benefits of relaxation are often more apparent in visual quality than in metrics, with corrected samples showing a more natural appearance and fewer visible artifacts.

\vspace{-0.2em}
\section{Additional Qualitative Results}
\label{app:qualitative-results}
\vspace{-0.2em}
Figures~\ref{fig:additional_qualitative_sdxl}, \ref{fig:additional_qualitative_sdxl_2}, 
\ref{fig:additional_qualitative_pixart_2}, \ref{fig:additional_qualitative_pixart}, 
\ref{fig:additional_qualitative_sd35}, and~\ref{fig:additional_qualitative_sd35_2} 
show additional qualitative results on SDXL, PixArt-$\alpha$, and SD3.5.
Figures~\ref{fig:additional_high_guidance_sdxl}, \ref{fig:additional_high_guidance_sdxl_2}, 
\ref{fig:additional_high_guidance_pixart}, \ref{fig:additional_high_guidance_pixart_2}, 
\ref{fig:additional_high_guidance_sd35}, and~\ref{fig:additional_high_guidance_sd35_2} 
show additional results under high guidance.
Across these examples, our method produces more favorable outputs than the baselines.
It mitigates artifacts observed in CFG, improves brightness and sharpness, and better preserves text-aligned visual content. Some outputs remain visually close to those of CFG. When intermediate spectra are already close to the analytic reference, our method applies only small spectral adjustments, helping preserve the quality already achieved by CFG. We recommend zooming in on the figures to examine fine-grained differences in visual quality and text alignment.


\paragraph{Analysis of Failure Cases.}
Despite the overall effectiveness of the proposed correction, the additional qualitative results also reveal two common failure patterns.
These include cases where the correction only partially mitigates artifacts observed in CFG, and cases where improvements in visual quality or prompt alignment are accompanied by undesired changes in other aspects of the image.
First, in the first example of Figure~\ref{fig:additional_qualitative_sdxl_2}, our method produces a cat with a more natural composition, but the generated cat still exhibits an abnormal tail structure.
Similarly, in the fifth example of Figure~\ref{fig:additional_qualitative_pixart}, the face of the person near the center is improved, while some details such as the hands remain unnatural. In these examples, our method retains much of the quality already achieved by CFG while refining some fine-grained details, though certain artifacts persist.
In the last example of Figure~\ref{fig:additional_qualitative_sd35}, our method improves visual quality over CFG, but some artifacts remain.
Second, some cases show a trade-off between different aspects of generation quality.
In the last example of Figure~\ref{fig:additional_high_guidance_sdxl_2}, the overall color tone is improved, but the number of people increases, making the result less consistent with the text condition.
In the last example of Figure~\ref{fig:additional_high_guidance_pixart_2}, the luggage shape becomes more natural, but the interaction between objects is weakened during the correction.
These examples suggest that, while the proposed method can improve generation quality in many cases, there remains room for further improving its robustness.



\vspace{-0.2em}
\section{Additional Discussion}
\label{app:additional-discussion}
\vspace{-0.2em}
The proposed method improves guided generation by normalizing spectral deviations from the expected denoising dynamics.
This correction keeps the guided trajectory closer to the spectral statistics induced by the forward process, which can improve image quality and prompt alignment.
At the same time, the correction does not provide explicit feedback toward a desired semantic direction.
As a result, its ability to control generation toward a specific target remains limited.
More direct semantic control may require additional learning signals, such as external rewards or task-specific supervision, but the use of such signals introduces separate challenges. These observations suggest that more fundamental improvements may require aligning the denoising dynamics from the training stage.
This direction calls for a deeper understanding of the representation space and generation dynamics formed by diffusion models.


\newpage

\begin{figure*}[!t]
    \centering
    \setlength{\tabcolsep}{2pt}
    \renewcommand{\arraystretch}{1.0}
    \scriptsize

    \begin{tabular}{c ccccc}
        & (a) Ours & (b) CFG & (c) CFG++ & (d) TV-CFG & (e) LF-CFG \\[0.25em]

        \multirow{4}{*}
        & \includegraphics[width=0.18\linewidth]{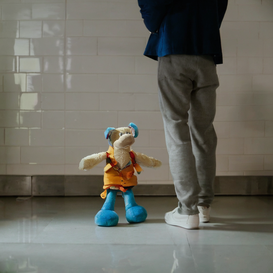}
        & \includegraphics[width=0.18\linewidth]{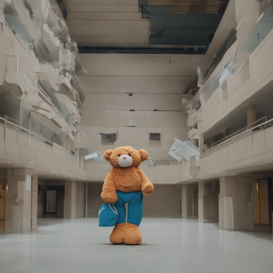}
        & \includegraphics[width=0.18\linewidth]{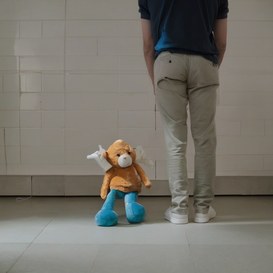}
        & \includegraphics[width=0.18\linewidth]{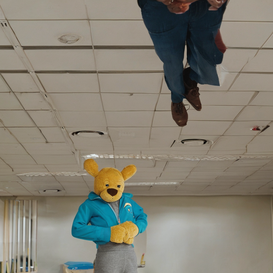}
        & \includegraphics[width=0.18\linewidth]{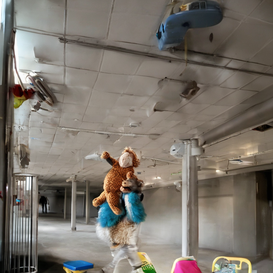} \\[-0.1em]
        & \multicolumn{5}{c}{\scriptsize \textbf{Prompt:} ``a man standing in a building with a stuffed animal''} \\[0.45em]

        & \includegraphics[width=0.18\linewidth]{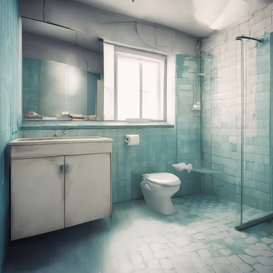}
        & \includegraphics[width=0.18\linewidth]{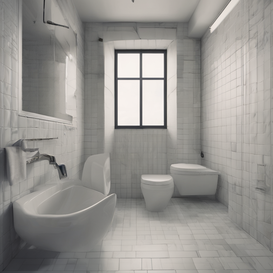}
        & \includegraphics[width=0.18\linewidth]{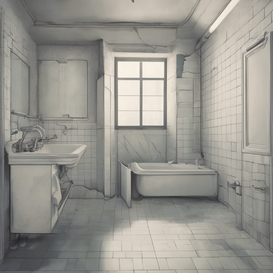}
        & \includegraphics[width=0.18\linewidth]{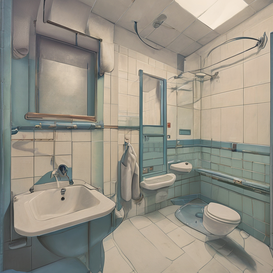}
        & \includegraphics[width=0.18\linewidth]{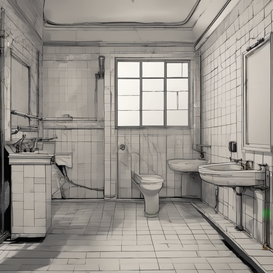} \\[-0.1em]
        & \multicolumn{5}{c}{\scriptsize \textbf{Prompt:} ``A bathroom that is empty and not in use.''} \\[0.8em]

        \multirow{4}{*}
        & \includegraphics[width=0.18\linewidth]{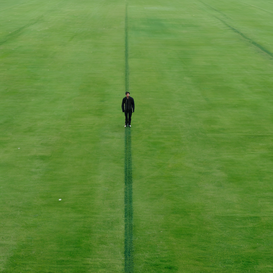}
        & \includegraphics[width=0.18\linewidth]{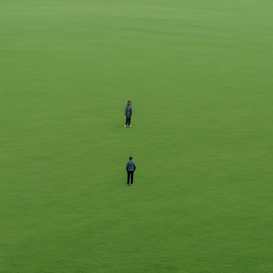}
        & \includegraphics[width=0.18\linewidth]{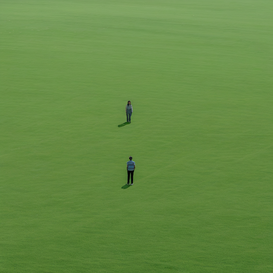}
        & \includegraphics[width=0.18\linewidth]{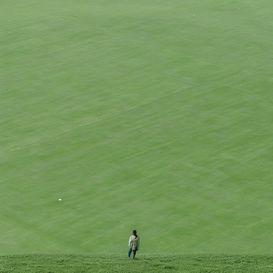}
        & \includegraphics[width=0.18\linewidth]{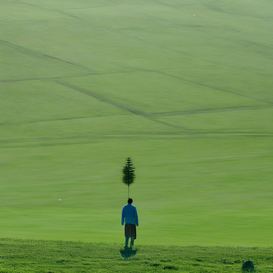} \\[-0.1em]
        & \multicolumn{5}{c}{\scriptsize \textbf{Prompt:} ``a person standing on a large green field.''} \\[0.45em]

        & \includegraphics[width=0.18\linewidth]{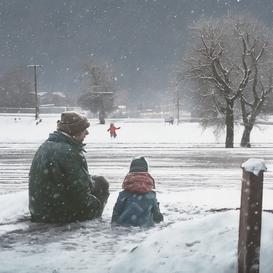}
        & \includegraphics[width=0.18\linewidth]{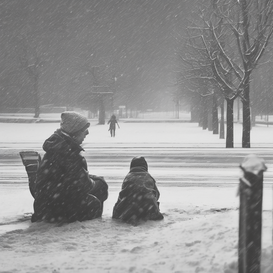}
        & \includegraphics[width=0.18\linewidth]{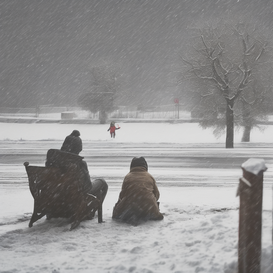}
        & \includegraphics[width=0.18\linewidth]{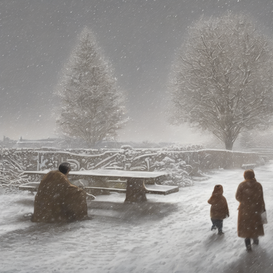}
        & \includegraphics[width=0.18\linewidth]{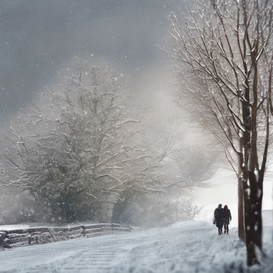} \\[-0.1em]
        & \multicolumn{5}{c}{\scriptsize \textbf{Prompt:} ``There are two people who are out in the snow.''} \\[0.8em]

        \multirow{4}{*}
        & \includegraphics[width=0.18\linewidth]{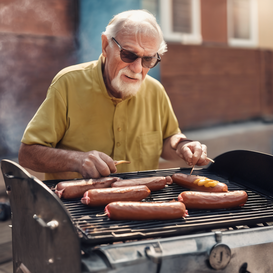}
        & \includegraphics[width=0.18\linewidth]{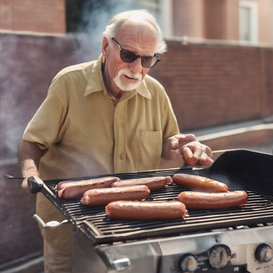}
        & \includegraphics[width=0.18\linewidth]{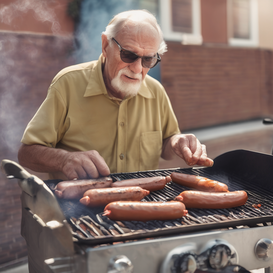}
        & \includegraphics[width=0.18\linewidth]{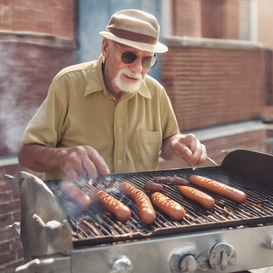}
        & \includegraphics[width=0.18\linewidth]{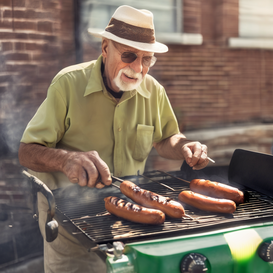} \\[-0.1em]
        & \multicolumn{5}{c}{\scriptsize \textbf{Prompt:} ``The elderly man is roasting hotdogs on a grill.''} \\[0.45em]

        & \includegraphics[width=0.18\linewidth]{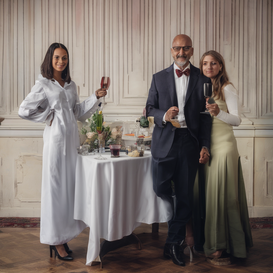}
        & \includegraphics[width=0.18\linewidth]{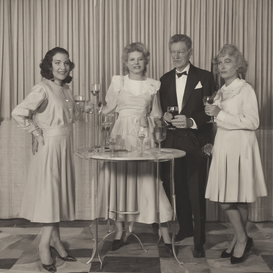}
        & \includegraphics[width=0.18\linewidth]{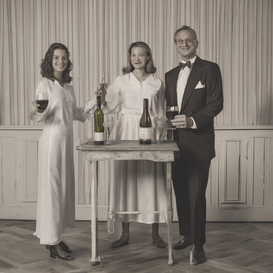}
        & \includegraphics[width=0.18\linewidth]{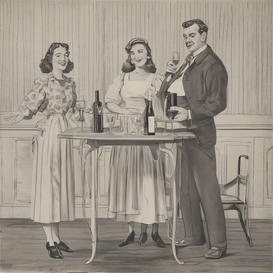}
        & \includegraphics[width=0.18\linewidth]{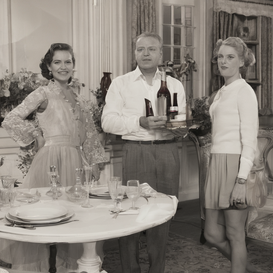} \\[-0.1em]
        & \multicolumn{5}{c}{\scriptsize \textbf{Prompt:} ``A man and two women standing near a table with wine glasses.''}
    \end{tabular}

    \caption{Additional qualitative results on SDXL.}
    \label{fig:additional_qualitative_sdxl}
\end{figure*}

\begin{figure*}[!t]
    \centering
    \setlength{\tabcolsep}{2pt}
    \renewcommand{\arraystretch}{1.0}
    \scriptsize

    \begin{tabular}{c ccccc}
        & (a) Ours & (b) CFG & (c) CFG++ & (d) TV-CFG & (e) LF-CFG \\[0.25em]

        \multirow{4}{*}
        & \includegraphics[width=0.18\linewidth]{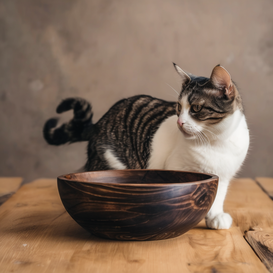}
        & \includegraphics[width=0.18\linewidth]{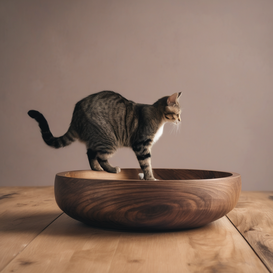}
        & \includegraphics[width=0.18\linewidth]{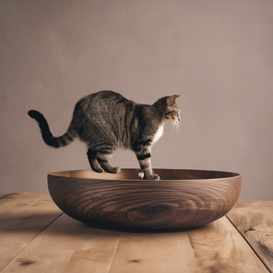}
        & \includegraphics[width=0.18\linewidth]{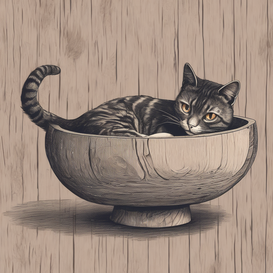}
        & \includegraphics[width=0.18\linewidth]{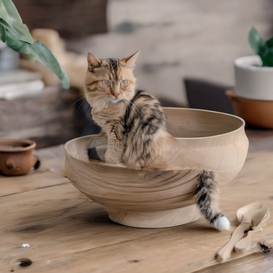} \\[-0.1em]
        & \multicolumn{5}{c}{\scriptsize \textbf{Prompt:} ``A CAT ON A WOODEN SURFACE IS LOOKING AT A WOODEN BOWL''} \\[0.45em]

        & \includegraphics[width=0.18\linewidth]{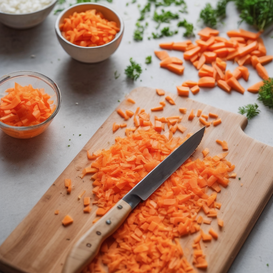}
        & \includegraphics[width=0.18\linewidth]{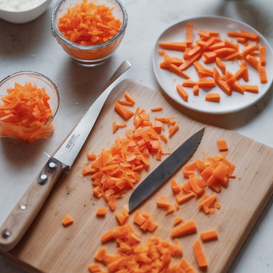}
        & \includegraphics[width=0.18\linewidth]{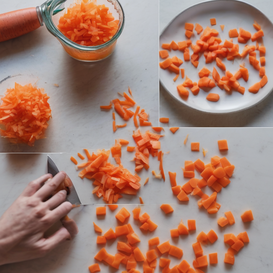}
        & \includegraphics[width=0.18\linewidth]{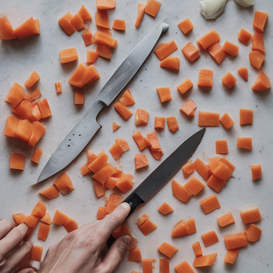}
        & \includegraphics[width=0.18\linewidth]{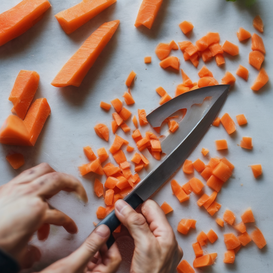} \\[-0.1em]
        & \multicolumn{5}{c}{\scriptsize \textbf{Prompt:} ``A knife and progressively finely chopped up carrots.''} \\[0.8em]

        \multirow{4}{*}
        & \includegraphics[width=0.18\linewidth]{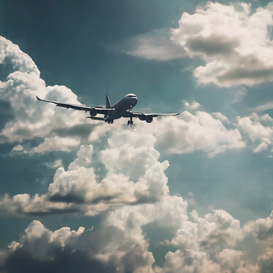}
        & \includegraphics[width=0.18\linewidth]{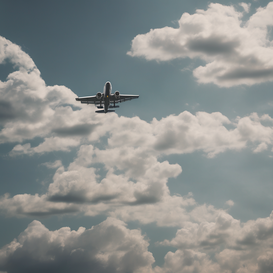}
        & \includegraphics[width=0.18\linewidth]{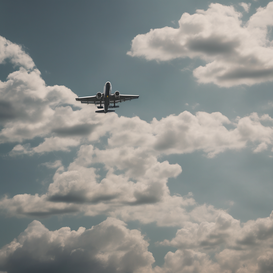}
        & \includegraphics[width=0.18\linewidth]{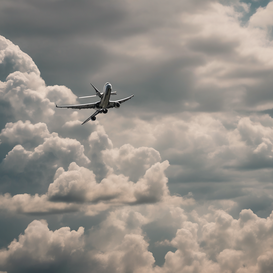}
        & \includegraphics[width=0.18\linewidth]{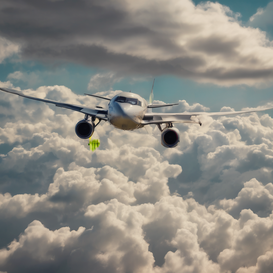} \\[-0.1em]
        & \multicolumn{5}{c}{\scriptsize \textbf{Prompt:} ``An airplane flies in a cloudy sky as a backdrop.''} \\[0.45em]

        & \includegraphics[width=0.18\linewidth]{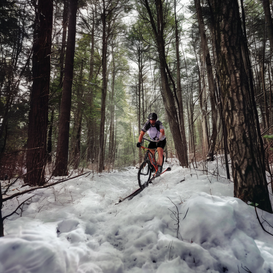}
        & \includegraphics[width=0.18\linewidth]{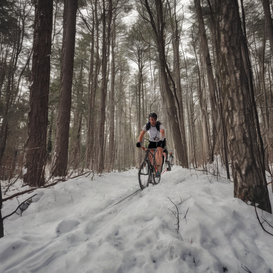}
        & \includegraphics[width=0.18\linewidth]{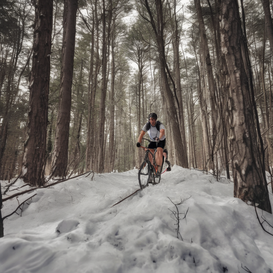}
        & \includegraphics[width=0.18\linewidth]{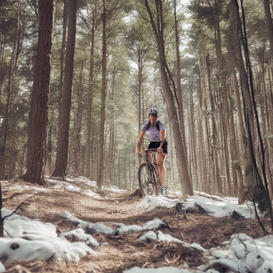}
        & \includegraphics[width=0.18\linewidth]{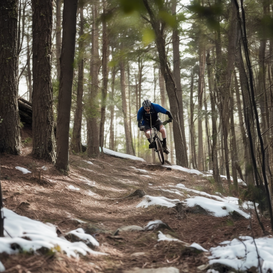} \\[-0.1em]
        & \multicolumn{5}{c}{\scriptsize \textbf{Prompt:} ``Short traing coming down the tacks in the middle of the woods.''} \\[0.8em]

        \multirow{4}{*}
        & \includegraphics[width=0.18\linewidth]{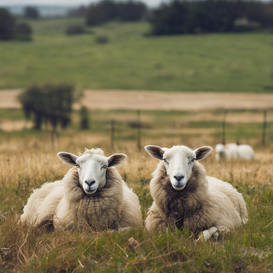}
        & \includegraphics[width=0.18\linewidth]{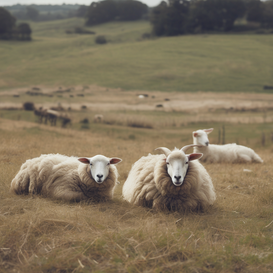}
        & \includegraphics[width=0.18\linewidth]{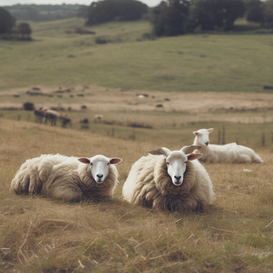}
        & \includegraphics[width=0.18\linewidth]{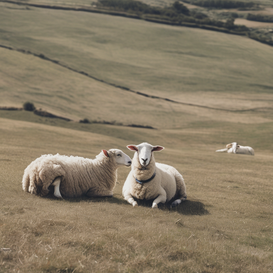}
        & \includegraphics[width=0.18\linewidth]{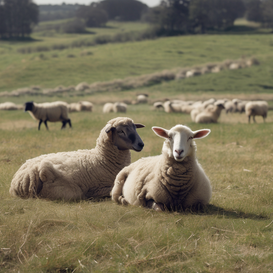} \\[-0.1em]
        & \multicolumn{5}{c}{\scriptsize \textbf{Prompt:} ``Two sheep laying in th middle of an open field.''} \\[0.45em]

        & \includegraphics[width=0.18\linewidth]{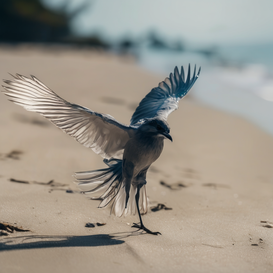}
        & \includegraphics[width=0.18\linewidth]{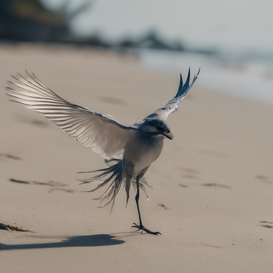}
        & \includegraphics[width=0.18\linewidth]{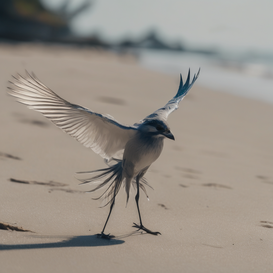}
        & \includegraphics[width=0.18\linewidth]{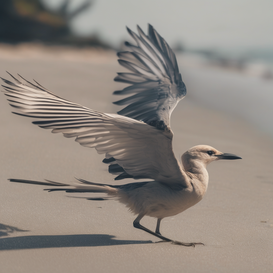}
        & \includegraphics[width=0.18\linewidth]{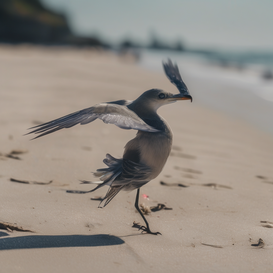} \\[-0.1em]
        & \multicolumn{5}{c}{\scriptsize \textbf{Prompt:} ``A bird that has its wings stretched out on a beach.''}
    \end{tabular}

    \caption{Additional qualitative results on SDXL.}
    \label{fig:additional_qualitative_sdxl_2}
\end{figure*}


\begin{figure*}[!t]
    \centering
    \setlength{\tabcolsep}{2pt}
    \renewcommand{\arraystretch}{1.0}
    \scriptsize

    \begin{tabular}{c ccccc}
        & (a) Ours & (b) CFG & (c) CFG++ & (d) TV-CFG & (e) LF-CFG \\[0.25em]

        \multirow{4}{*}
        & \includegraphics[width=0.18\linewidth]{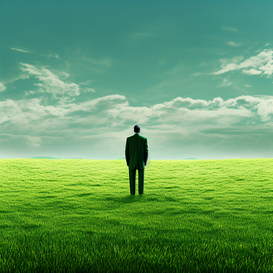}
        & \includegraphics[width=0.18\linewidth]{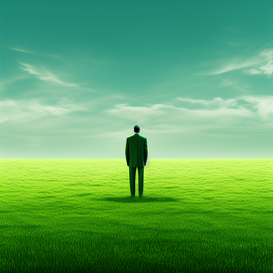}
        & \includegraphics[width=0.18\linewidth]{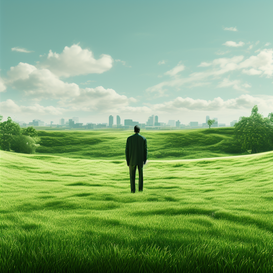}
        & \includegraphics[width=0.18\linewidth]{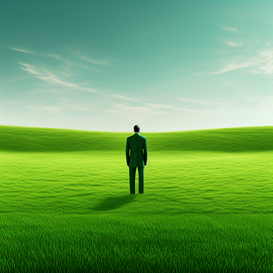}
        & \includegraphics[width=0.18\linewidth]{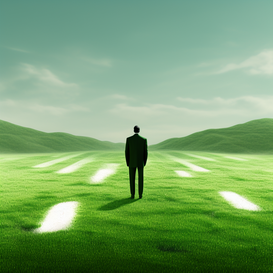} \\[-0.1em]
        & \multicolumn{5}{c}{\scriptsize \textbf{Prompt:} ``a person standing on a large green field.''} \\[0.45em]

        & \includegraphics[width=0.18\linewidth]{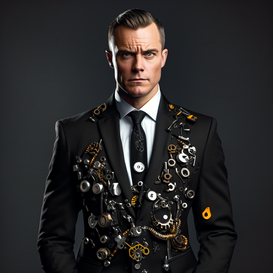}
        & \includegraphics[width=0.18\linewidth]{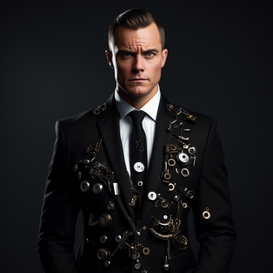}
        & \includegraphics[width=0.18\linewidth]{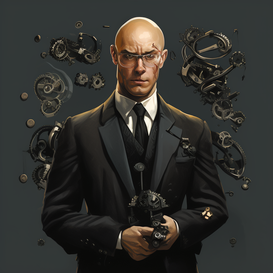}
        & \includegraphics[width=0.18\linewidth]{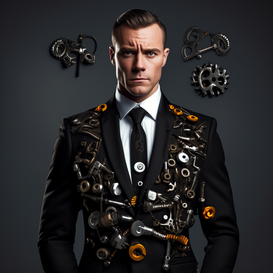}
        & \includegraphics[width=0.18\linewidth]{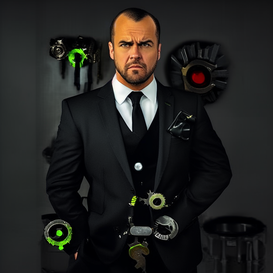} \\[-0.1em]
        & \multicolumn{5}{c}{\scriptsize \textbf{Prompt:} ``A black jacket and a tie with nuts and bolts on it.''} \\[0.8em]

        \multirow{4}{*}
        & \includegraphics[width=0.18\linewidth]{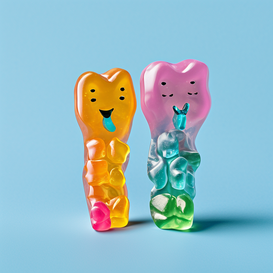}
        & \includegraphics[width=0.18\linewidth]{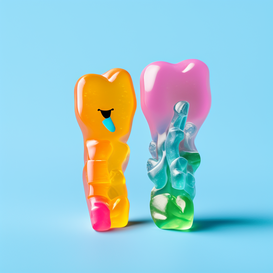}
        & \includegraphics[width=0.18\linewidth]{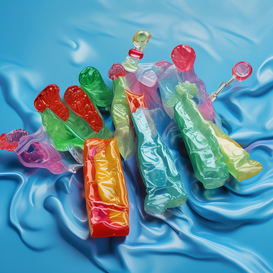}
        & \includegraphics[width=0.18\linewidth]{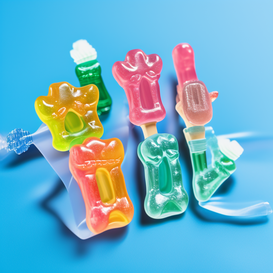}
        & \includegraphics[width=0.18\linewidth]{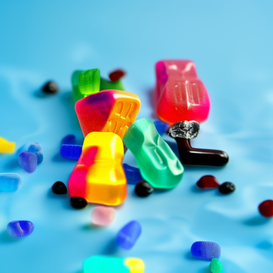} \\[-0.1em]
        & \multicolumn{5}{c}{\scriptsize \textbf{Prompt:} ``Some gummy teeth and gummy toothbrushes on a wrapper.''} \\[0.45em]

        & \includegraphics[width=0.18\linewidth]{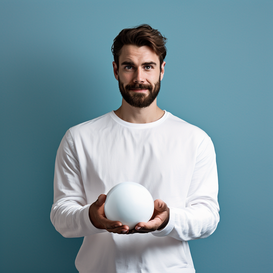}
        & \includegraphics[width=0.18\linewidth]{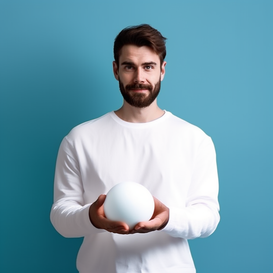}
        & \includegraphics[width=0.18\linewidth]{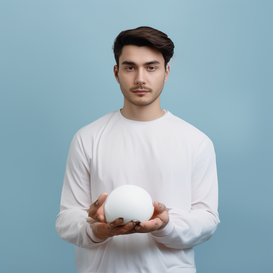}
        & \includegraphics[width=0.18\linewidth]{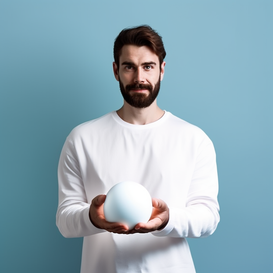}
        & \includegraphics[width=0.18\linewidth]{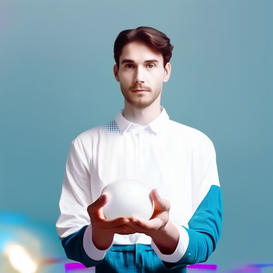} \\[-0.1em]
        & \multicolumn{5}{c}{\scriptsize \textbf{Prompt:} ``A man holding a white object in his hand.''} \\[0.8em]

        \multirow{4}{*}
        & \includegraphics[width=0.18\linewidth]{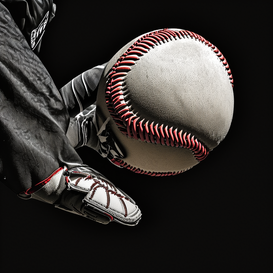}
        & \includegraphics[width=0.18\linewidth]{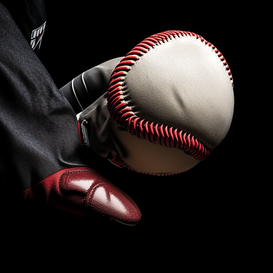}
        & \includegraphics[width=0.18\linewidth]{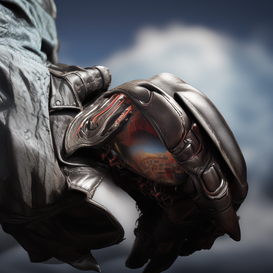}
        & \includegraphics[width=0.18\linewidth]{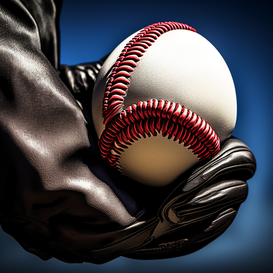}
        & \includegraphics[width=0.18\linewidth]{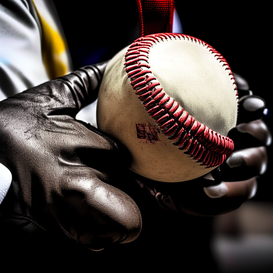} \\[-0.1em]
        & \multicolumn{5}{c}{\scriptsize \textbf{Prompt:} ``The baseball was caught in the umpires glove.''} \\[0.45em]

        & \includegraphics[width=0.18\linewidth]{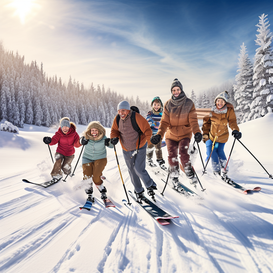}
        & \includegraphics[width=0.18\linewidth]{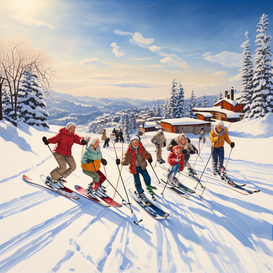}
        & \includegraphics[width=0.18\linewidth]{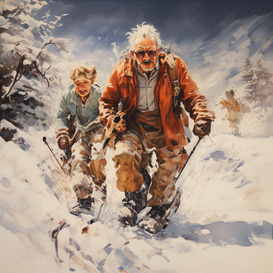}
        & \includegraphics[width=0.18\linewidth]{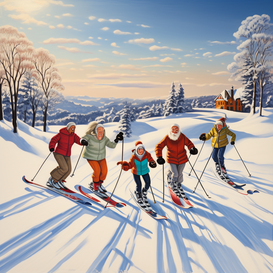}
        & \includegraphics[width=0.18\linewidth]{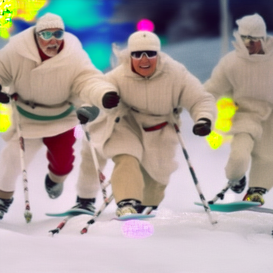} \\[-0.1em]
        & \multicolumn{5}{c}{\scriptsize \textbf{Prompt:} ``A group of young and old are skiing on the snow.''}
    \end{tabular}

    \caption{Additional qualitative results on PixArt-$\alpha$.}
    \label{fig:additional_qualitative_pixart_2}
\end{figure*}

\begin{figure*}[!t]
    \centering
    \setlength{\tabcolsep}{2pt}
    \renewcommand{\arraystretch}{1.0}
    \scriptsize

    \begin{tabular}{c ccccc}
        & (a) Ours & (b) CFG & (c) CFG++ & (d) TV-CFG & (e) LF-CFG \\[0.25em]

        \multirow{4}{*}
        & \includegraphics[width=0.18\linewidth]{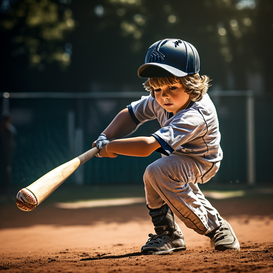}
        & \includegraphics[width=0.18\linewidth]{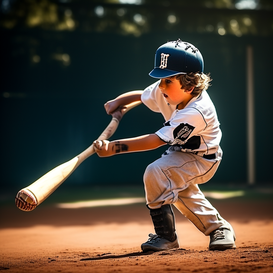}
        & \includegraphics[width=0.18\linewidth]{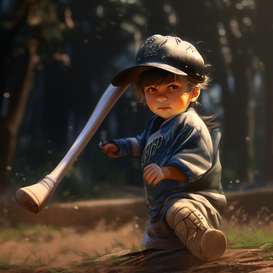}
        & \includegraphics[width=0.18\linewidth]{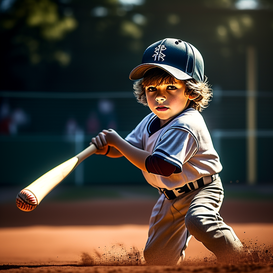}
        & \includegraphics[width=0.18\linewidth]{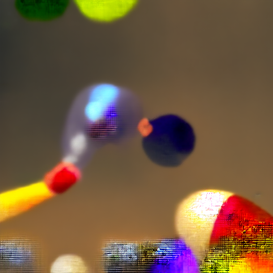} \\[-0.1em]
        & \multicolumn{5}{c}{\scriptsize \textbf{Prompt:} ``a boy swinging his baseball bat at a baseball''} \\[0.45em]

        & \includegraphics[width=0.18\linewidth]{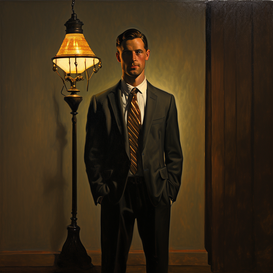}
        & \includegraphics[width=0.18\linewidth]{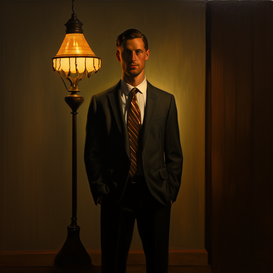}
        & \includegraphics[width=0.18\linewidth]{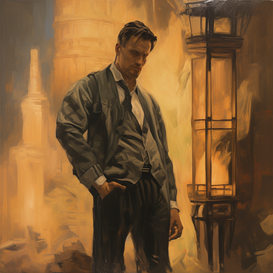}
        & \includegraphics[width=0.18\linewidth]{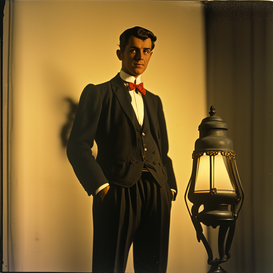}
        & \includegraphics[width=0.18\linewidth]{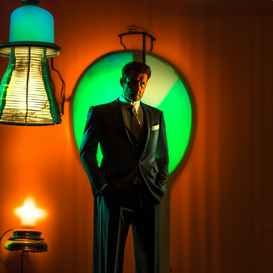} \\[-0.1em]
        & \multicolumn{5}{c}{\scriptsize \textbf{Prompt:} ``A man wearing only a tie standing next to a lamp.''} \\[0.8em]

        \multirow{4}{*}
        & \includegraphics[width=0.18\linewidth]{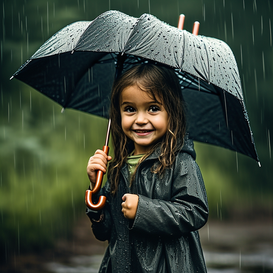}
        & \includegraphics[width=0.18\linewidth]{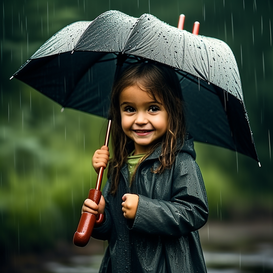}
        & \includegraphics[width=0.18\linewidth]{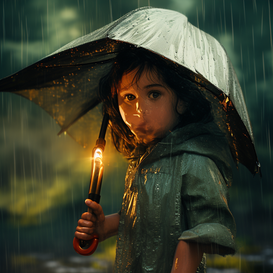}
        & \includegraphics[width=0.18\linewidth]{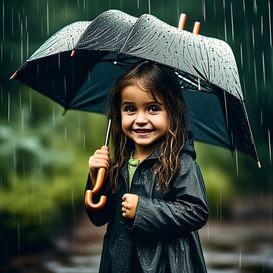}
        & \includegraphics[width=0.18\linewidth]{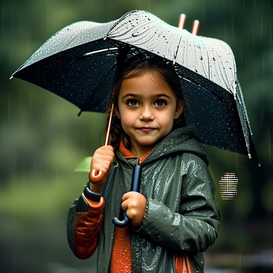} \\[-0.1em]
        & \multicolumn{5}{c}{\scriptsize \textbf{Prompt:} ``A little girl is holding an umbrella on a wet day.''} \\[0.45em]

        & \includegraphics[width=0.18\linewidth]{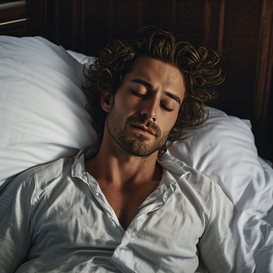}
        & \includegraphics[width=0.18\linewidth]{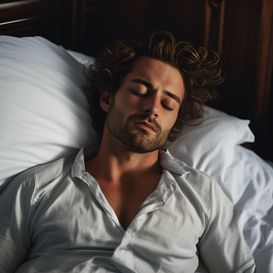}
        & \includegraphics[width=0.18\linewidth]{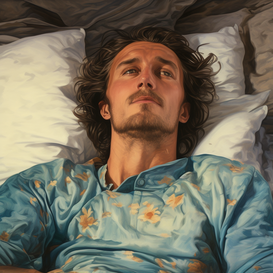}
        & \includegraphics[width=0.18\linewidth]{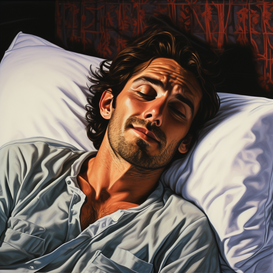}
        & \includegraphics[width=0.18\linewidth]{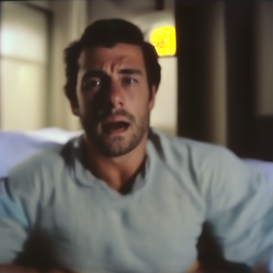} \\[-0.1em]
        & \multicolumn{5}{c}{\scriptsize \textbf{Prompt:} ``A man with his shirt open lying in bed.''} \\[0.8em]

        \multirow{4}{*}
        & \includegraphics[width=0.18\linewidth]{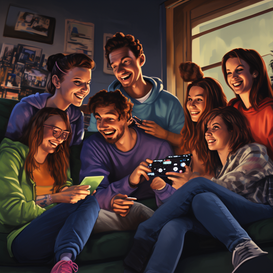}
        & \includegraphics[width=0.18\linewidth]{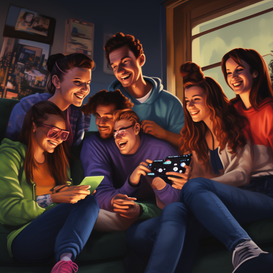}
        & \includegraphics[width=0.18\linewidth]{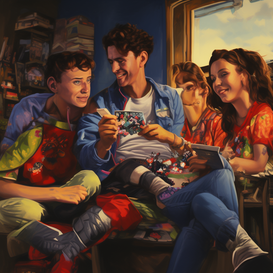}
        & \includegraphics[width=0.18\linewidth]{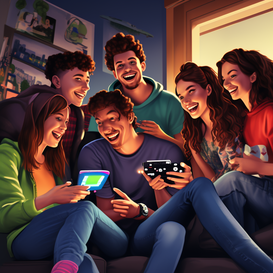}
        & \includegraphics[width=0.18\linewidth]{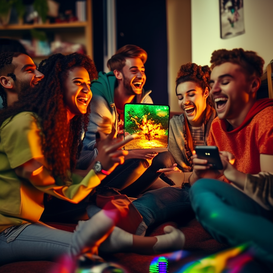} \\[-0.1em]
        & \multicolumn{5}{c}{\scriptsize \textbf{Prompt:} ``A group of young people sitting on a couch next to a guy playing a Nintendo Wii.''} \\[0.45em]

        & \includegraphics[width=0.18\linewidth]{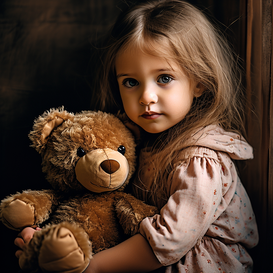}
        & \includegraphics[width=0.18\linewidth]{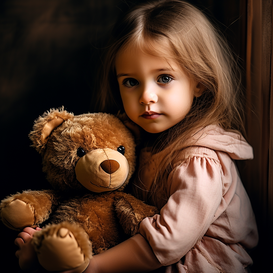}
        & \includegraphics[width=0.18\linewidth]{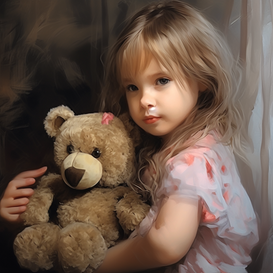}
        & \includegraphics[width=0.18\linewidth]{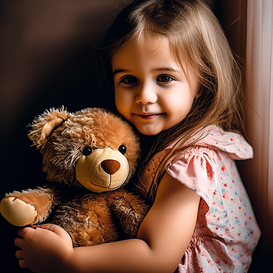}
        & \includegraphics[width=0.18\linewidth]{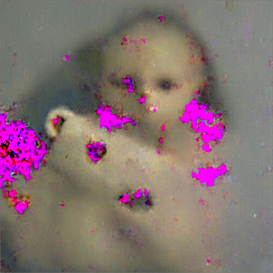} \\[-0.1em]
        & \multicolumn{5}{c}{\scriptsize \textbf{Prompt:} ``A little girl is holding a teddy bear.''}
    \end{tabular}

    \caption{Additional qualitative results on PixArt-$\alpha$.}
    \label{fig:additional_qualitative_pixart}
\end{figure*}


\begin{figure*}[!t]
    \centering
    \setlength{\tabcolsep}{2pt}
    \renewcommand{\arraystretch}{1.0}
    \scriptsize

    \begin{tabular}{c ccccc}
        & (a) Ours & (b) CFG & (c) CFG++ & (d) TV-CFG & (e) LF-CFG \\[0.25em]

        \multirow{4}{*}
        & \includegraphics[width=0.18\linewidth]{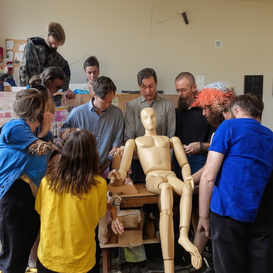}
        & \includegraphics[width=0.18\linewidth]{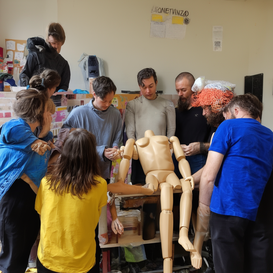}
        & \includegraphics[width=0.18\linewidth]{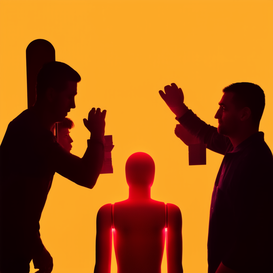}
        & \includegraphics[width=0.18\linewidth]{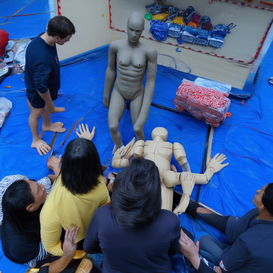}
        & \includegraphics[width=0.18\linewidth]{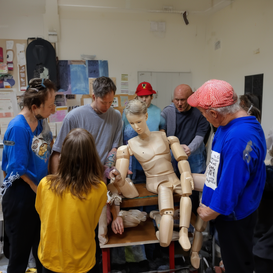} \\[-0.1em]
        & \multicolumn{5}{c}{\scriptsize \textbf{Prompt:} ``Some people gathered together around a dummy on a board.''} \\[0.45em]

        & \includegraphics[width=0.18\linewidth]{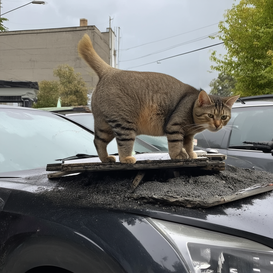}
        & \includegraphics[width=0.18\linewidth]{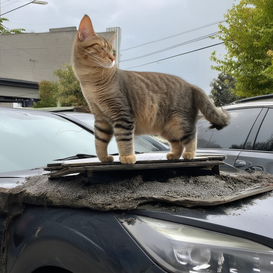}
        & \includegraphics[width=0.18\linewidth]{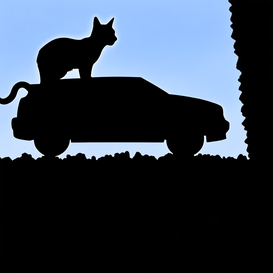}
        & \includegraphics[width=0.18\linewidth]{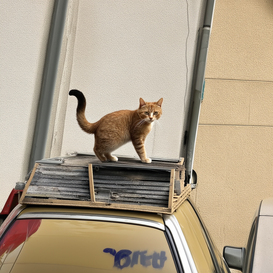}
        & \includegraphics[width=0.18\linewidth]{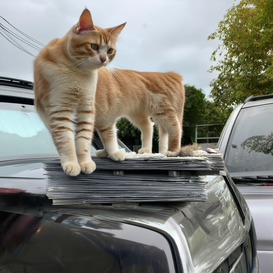} \\[-0.1em]
        & \multicolumn{5}{c}{\scriptsize \textbf{Prompt:} ``a cat standing on top of a cars ash board''} \\[0.8em]

        \multirow{4}{*}
        & \includegraphics[width=0.18\linewidth]{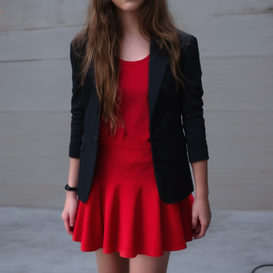}
        & \includegraphics[width=0.18\linewidth]{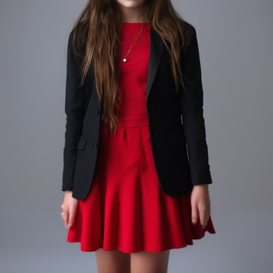}
        & \includegraphics[width=0.18\linewidth]{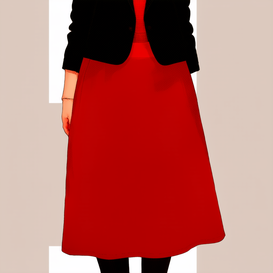}
        & \includegraphics[width=0.18\linewidth]{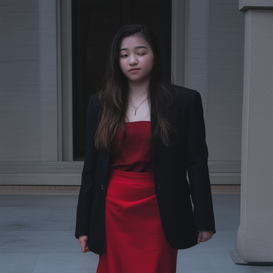}
        & \includegraphics[width=0.18\linewidth]{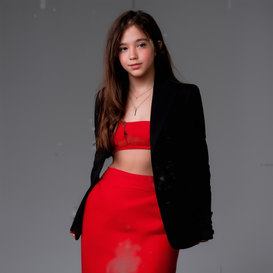} \\[-0.1em]
        & \multicolumn{5}{c}{\scriptsize \textbf{Prompt:} ``A girl in a red skirt and red tye with a blazer.''} \\[0.45em]

        & \includegraphics[width=0.18\linewidth]{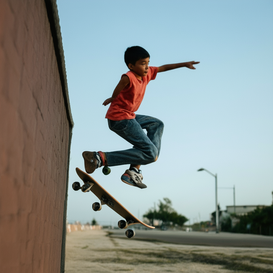}
        & \includegraphics[width=0.18\linewidth]{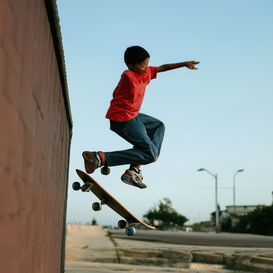}
        & \includegraphics[width=0.18\linewidth]{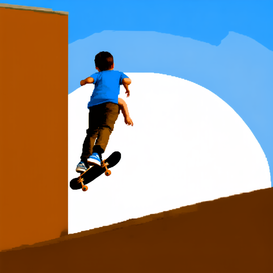}
        & \includegraphics[width=0.18\linewidth]{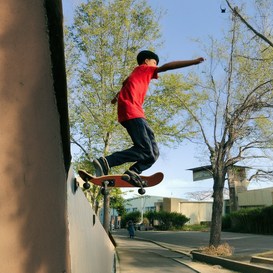}
        & \includegraphics[width=0.18\linewidth]{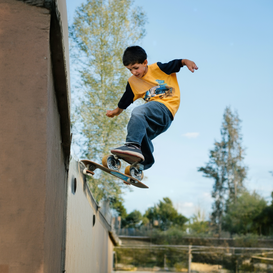} \\[-0.1em]
        & \multicolumn{5}{c}{\scriptsize \textbf{Prompt:} ``A boy doing a jump off a wall with a skateboard.''} \\[0.8em]

        \multirow{4}{*}
        & \includegraphics[width=0.18\linewidth]{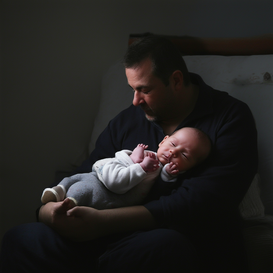}
        & \includegraphics[width=0.18\linewidth]{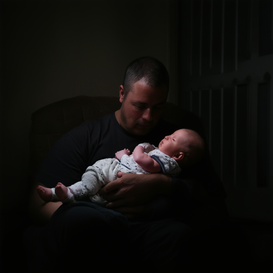}
        & \includegraphics[width=0.18\linewidth]{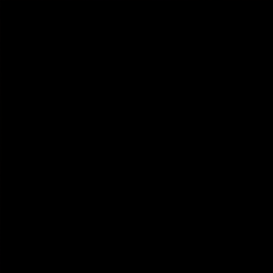}
        & \includegraphics[width=0.18\linewidth]{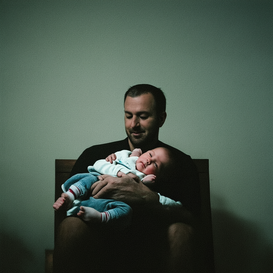}
        & \includegraphics[width=0.18\linewidth]{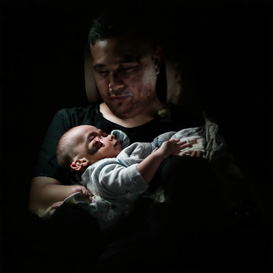} \\[-0.1em]
        & \multicolumn{5}{c}{\scriptsize \textbf{Prompt:} ``A man sitting in the dark holding a baby.''} \\[0.45em]

        & \includegraphics[width=0.18\linewidth]{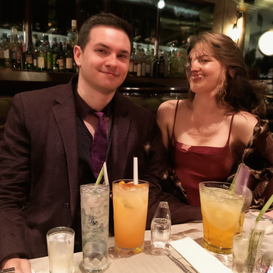}
        & \includegraphics[width=0.18\linewidth]{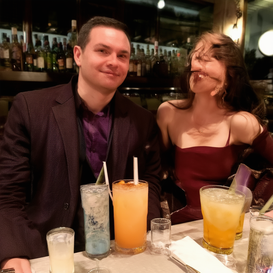}
        & \includegraphics[width=0.18\linewidth]{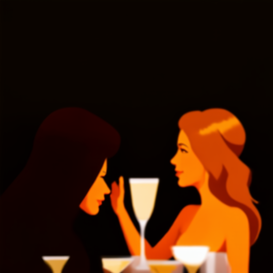}
        & \includegraphics[width=0.18\linewidth]{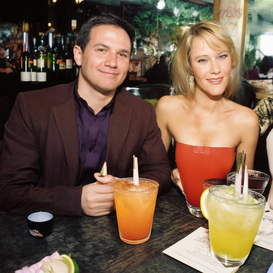}
        & \includegraphics[width=0.18\linewidth]{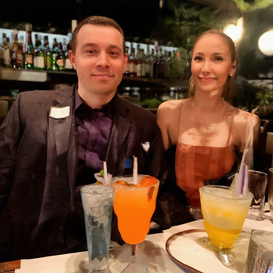} \\[-0.1em]
        & \multicolumn{5}{c}{\scriptsize \textbf{Prompt:} ``A man and woman at a table with several drinks.''}
    \end{tabular}

    \caption{Additional qualitative results on SD3.5.
    }
    \label{fig:additional_qualitative_sd35}
\end{figure*}

\begin{figure*}[!t]
    \centering
    \setlength{\tabcolsep}{2pt}
    \renewcommand{\arraystretch}{1.0}
    \scriptsize

    \begin{tabular}{c ccccc}
        & (a) Ours & (b) CFG & (c) CFG++ & (d) TV-CFG & (e) LF-CFG \\[0.25em]

        \multirow{4}{*}
        & \includegraphics[width=0.18\linewidth]{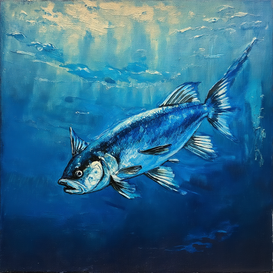}
        & \includegraphics[width=0.18\linewidth]{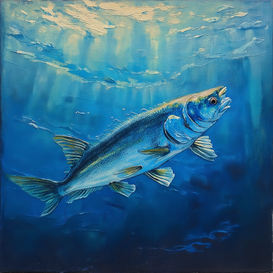}
        & \includegraphics[width=0.18\linewidth]{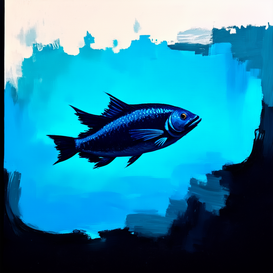}
        & \includegraphics[width=0.18\linewidth]{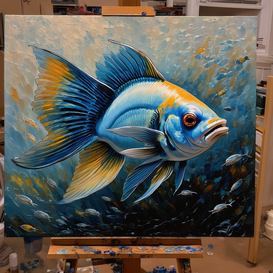}
        & \includegraphics[width=0.18\linewidth]{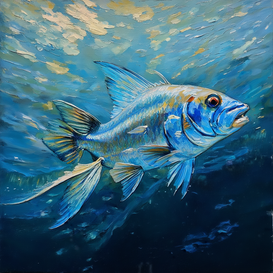} \\[-0.1em]
        & \multicolumn{5}{c}{\scriptsize \textbf{Prompt:} ``A painting of a blue fish flying through the canvas''} \\[0.45em]

        & \includegraphics[width=0.18\linewidth]{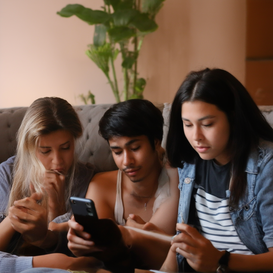}
        & \includegraphics[width=0.18\linewidth]{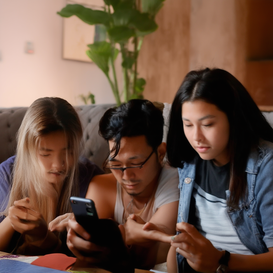}
        & \includegraphics[width=0.18\linewidth]{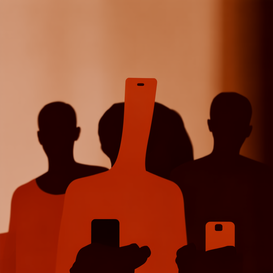}
        & \includegraphics[width=0.18\linewidth]{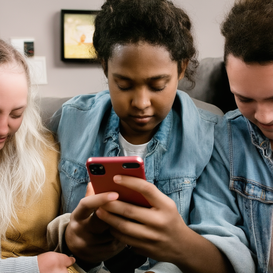}
        & \includegraphics[width=0.18\linewidth]{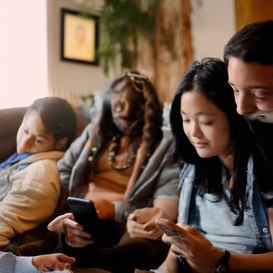} \\[-0.1em]
        & \multicolumn{5}{c}{\scriptsize \textbf{Prompt:} ``three friends watching something attentively on a phone''} \\[0.8em]

        \multirow{4}{*}
        & \includegraphics[width=0.18\linewidth]{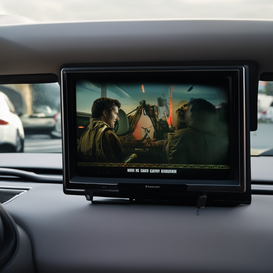}
        & \includegraphics[width=0.18\linewidth]{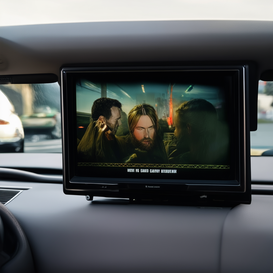}
        & \includegraphics[width=0.18\linewidth]{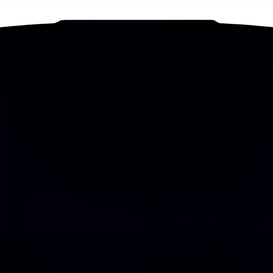}
        & \includegraphics[width=0.18\linewidth]{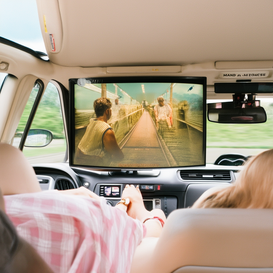}
        & \includegraphics[width=0.18\linewidth]{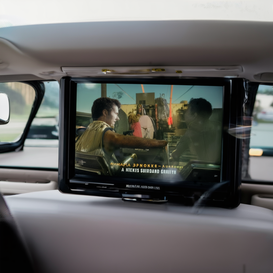} \\[-0.1em]
        & \multicolumn{5}{c}{\scriptsize \textbf{Prompt:} ``A tv mounted inside of a vehicle displaying a movie.''} \\[0.45em]

        & \includegraphics[width=0.18\linewidth]{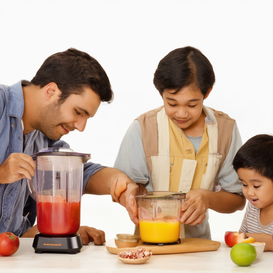}
        & \includegraphics[width=0.18\linewidth]{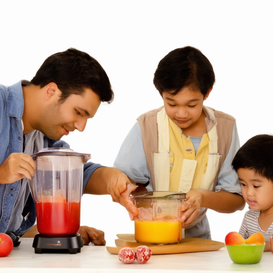}
        & \includegraphics[width=0.18\linewidth]{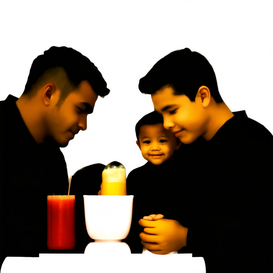}
        & \includegraphics[width=0.18\linewidth]{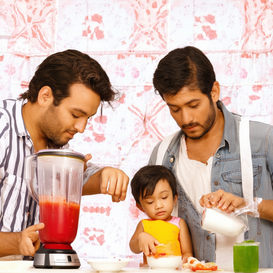}
        & \includegraphics[width=0.18\linewidth]{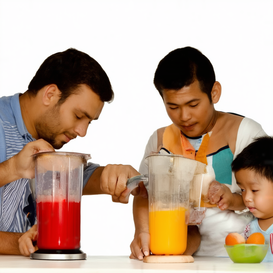} \\[-0.1em]
        & \multicolumn{5}{c}{\scriptsize \textbf{Prompt:} ``Two men and two children under a flap, and one of the men is mixing juice in a blender.''} \\[0.8em]

        \multirow{4}{*}
        & \includegraphics[width=0.18\linewidth]{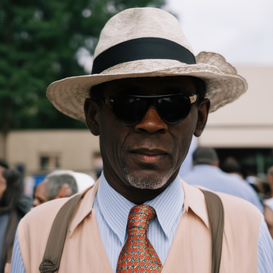}
        & \includegraphics[width=0.18\linewidth]{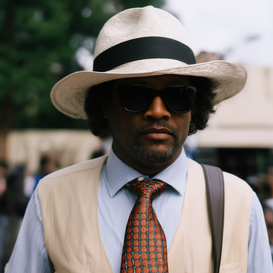}
        & \includegraphics[width=0.18\linewidth]{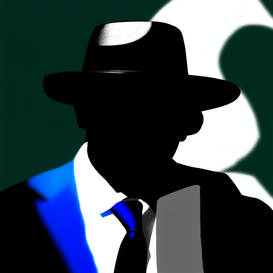}
        & \includegraphics[width=0.18\linewidth]{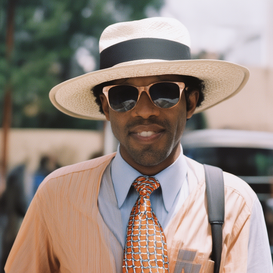}
        & \includegraphics[width=0.18\linewidth]{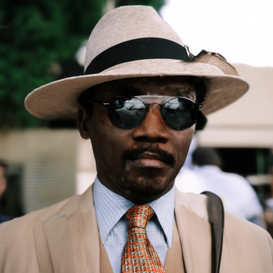} \\[-0.1em]
        & \multicolumn{5}{c}{\scriptsize \textbf{Prompt:} ``A man with a large tie, a hat, and sunglasses.''} \\[0.45em]

        & \includegraphics[width=0.18\linewidth]{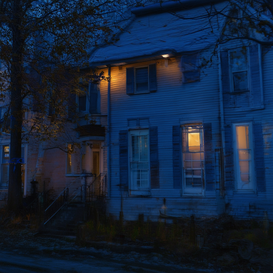}
        & \includegraphics[width=0.18\linewidth]{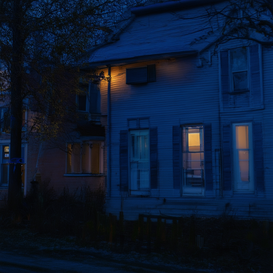}
        & \includegraphics[width=0.18\linewidth]{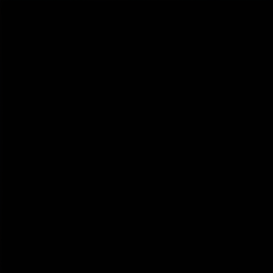}
        & \includegraphics[width=0.18\linewidth]{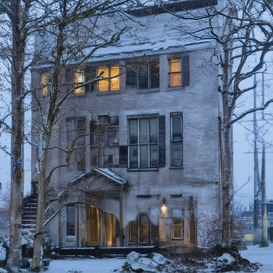}
        & \includegraphics[width=0.18\linewidth]{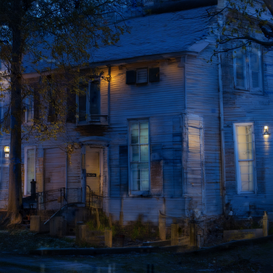} \\[-0.1em]
        & \multicolumn{5}{c}{\scriptsize \textbf{Prompt:} ``At almost 6:30 the place is almost dark''}
    \end{tabular}

    \caption{Additional qualitative results on SD3.5.
    }
    \label{fig:additional_qualitative_sd35_2}
\end{figure*}


\begin{figure*}[!t]
    \centering
    \setlength{\tabcolsep}{2pt}
    \renewcommand{\arraystretch}{1.0}
    \scriptsize

    \begin{tabular}{c cc}
        & (a) Ours & (b) CFG \\[0.25em]

        \multirow{6}{*}
        & \includegraphics[width=0.42\linewidth]{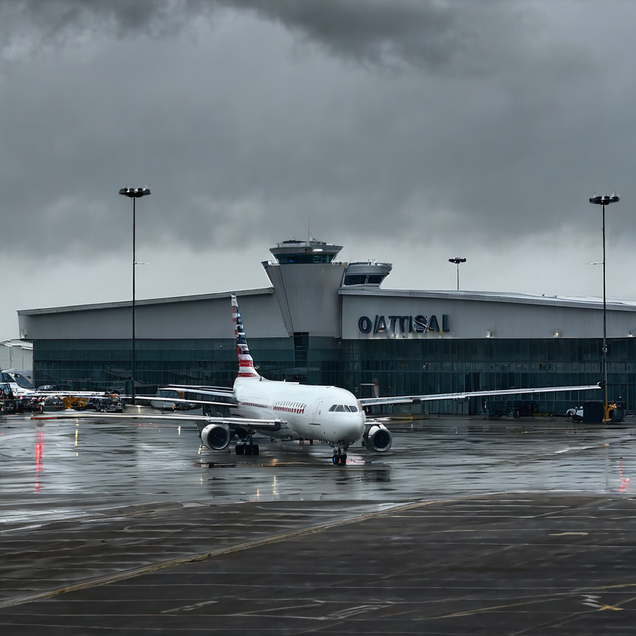}
        & \includegraphics[width=0.42\linewidth]{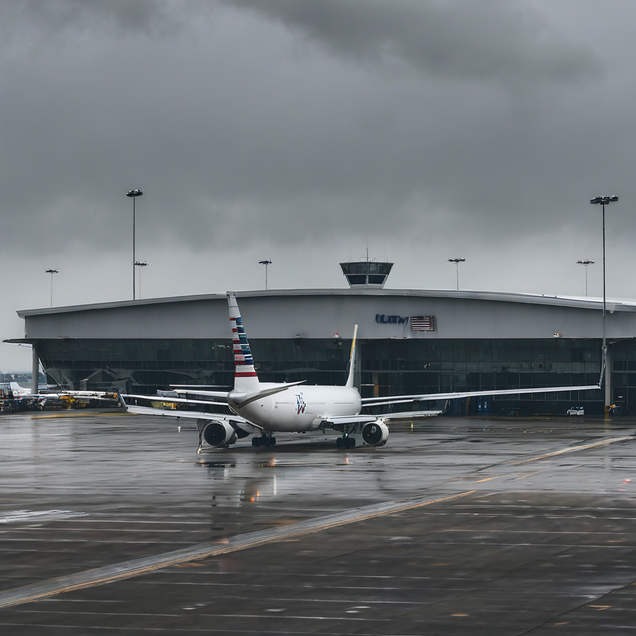} \\[-0.1em]
        & \multicolumn{2}{c}{\scriptsize \textbf{Prompt:} ``An airplane parked at a airport terminal under gray skies.''} \\[0.45em]

        & \includegraphics[width=0.42\linewidth]{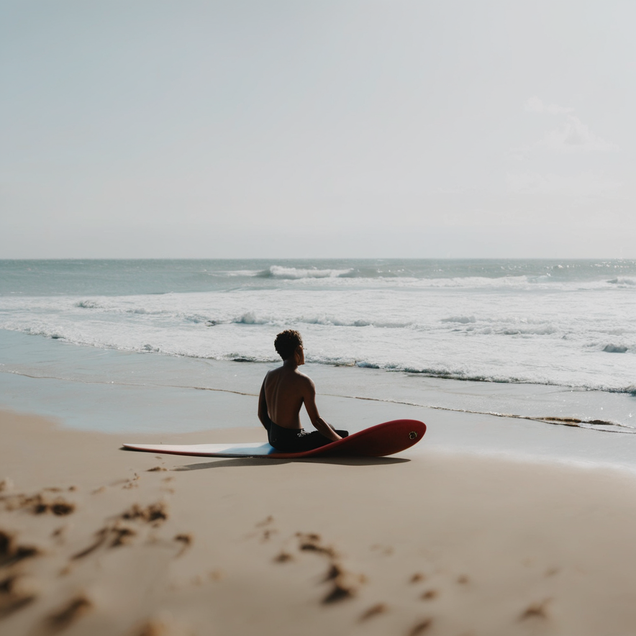}
        & \includegraphics[width=0.42\linewidth]{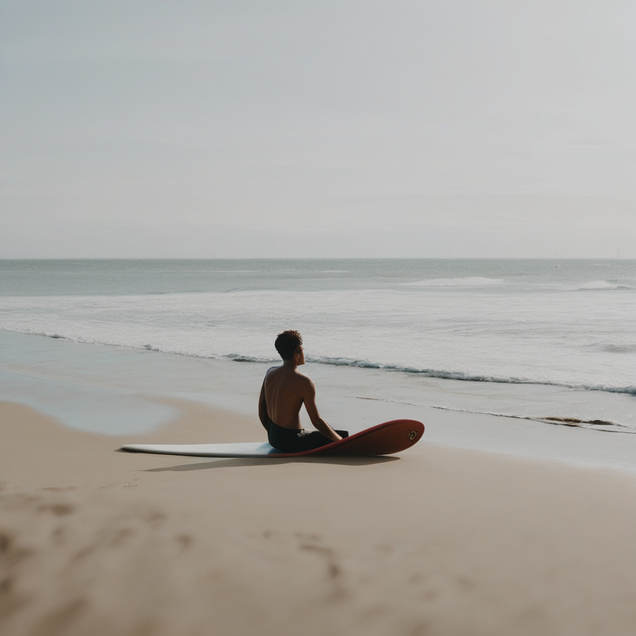} \\[-0.1em]
        & \multicolumn{2}{c}{\scriptsize \textbf{Prompt:} ``A person sitting on a surfboard on a beach.''} \\[0.45em]

        & \includegraphics[width=0.42\linewidth]{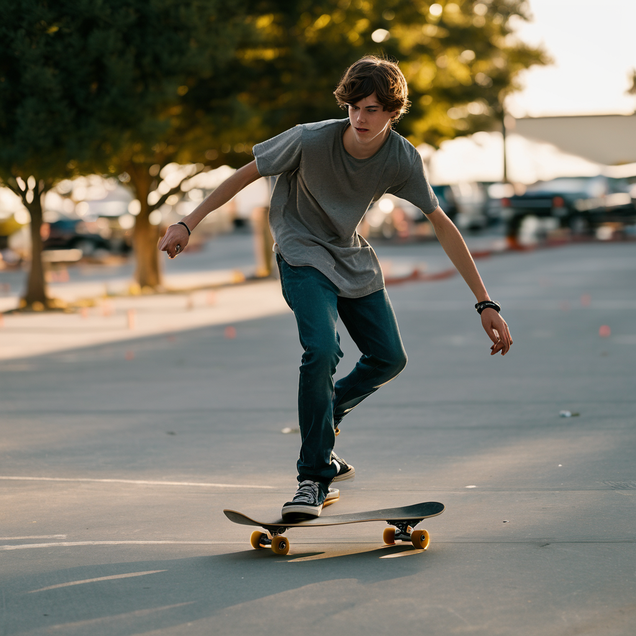}
        & \includegraphics[width=0.42\linewidth]{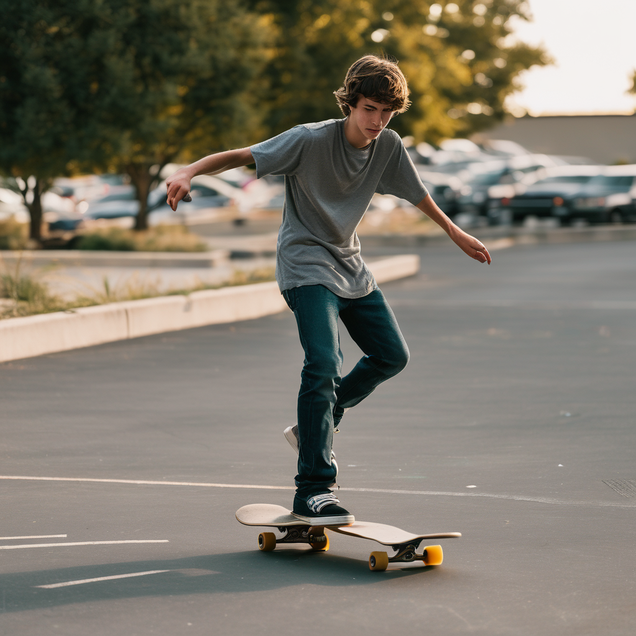} \\[-0.1em]
        & \multicolumn{2}{c}{\scriptsize \textbf{Prompt:} ``A young man skate boarding on a curbed edge in a parking lot.''}
    \end{tabular}

    \caption{Additional qualitative results under high guidance on SDXL.}
    \label{fig:additional_high_guidance_sdxl}
\end{figure*}

\begin{figure*}[!t]
    \centering
    \setlength{\tabcolsep}{2pt}
    \renewcommand{\arraystretch}{1.0}
    \scriptsize

    \begin{tabular}{c cc}
        & (a) Ours & (b) CFG \\[0.25em]

        \multirow{6}{*}
        & \includegraphics[width=0.42\linewidth]{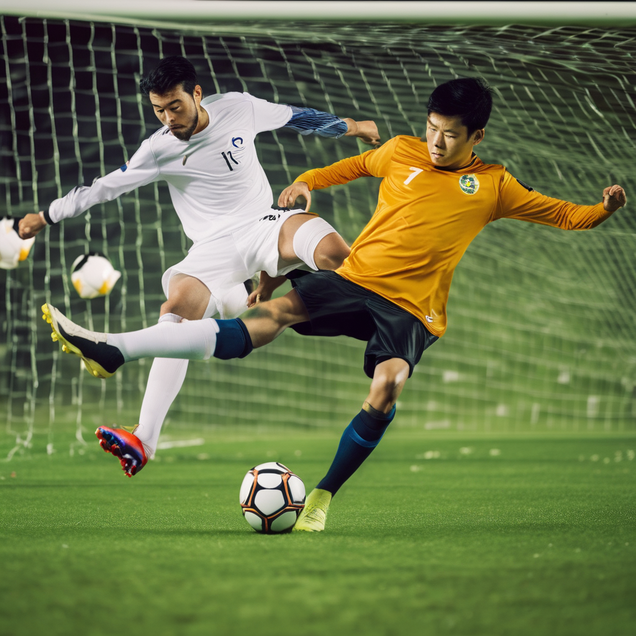}
        & \includegraphics[width=0.42\linewidth]{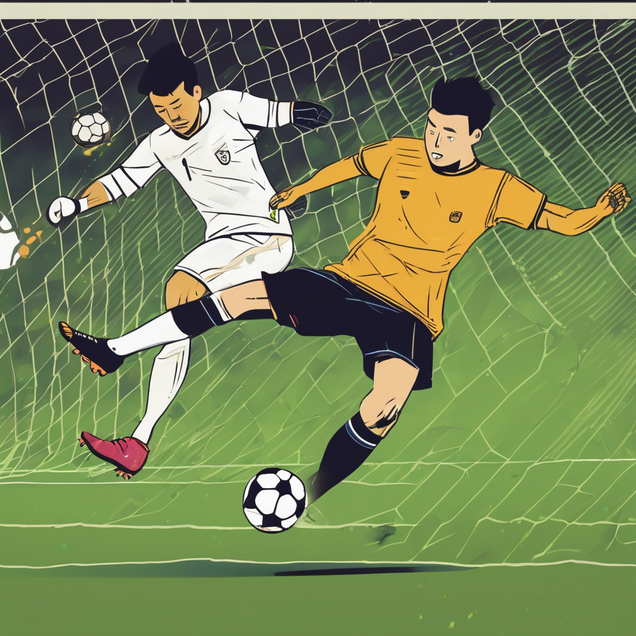} \\[-0.1em]
        & \multicolumn{2}{c}{\scriptsize \textbf{Prompt:} ``a soccer player kicks a ball directly at the goalkeeper.''} \\[0.45em]

        & \includegraphics[width=0.42\linewidth]{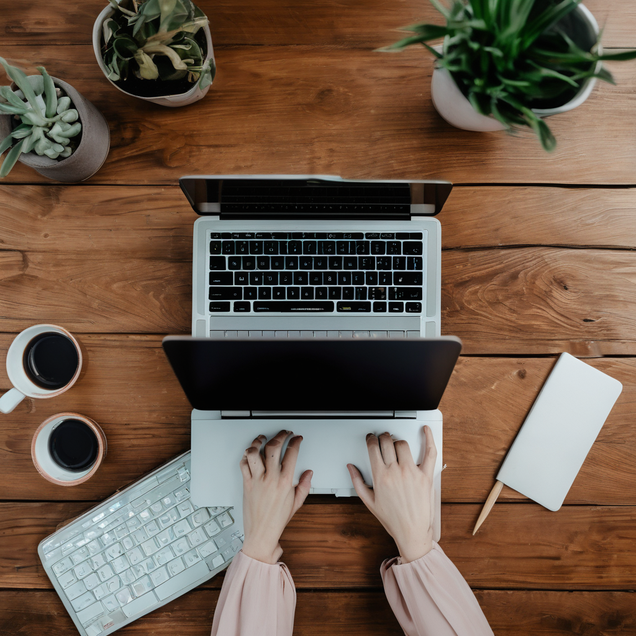}
        & \includegraphics[width=0.42\linewidth]{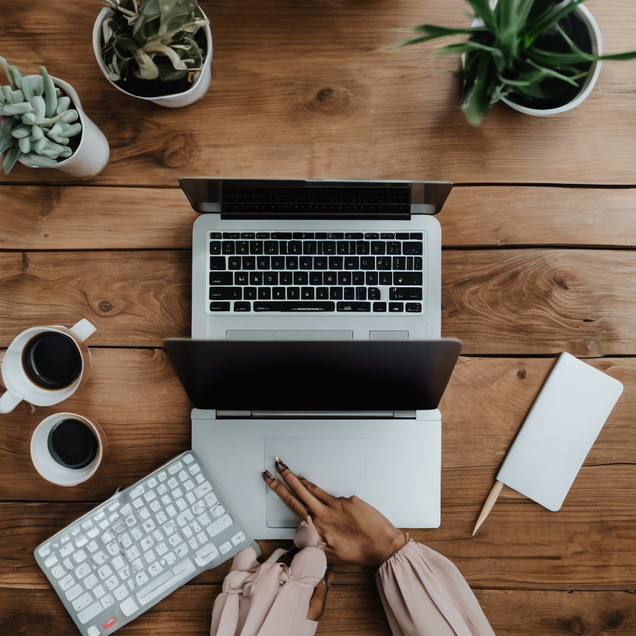} \\[-0.1em]
        & \multicolumn{2}{c}{\scriptsize \textbf{Prompt:} ``A woman is typing on a laptop on a wooden table.''} \\[0.45em]

        & \includegraphics[width=0.42\linewidth]{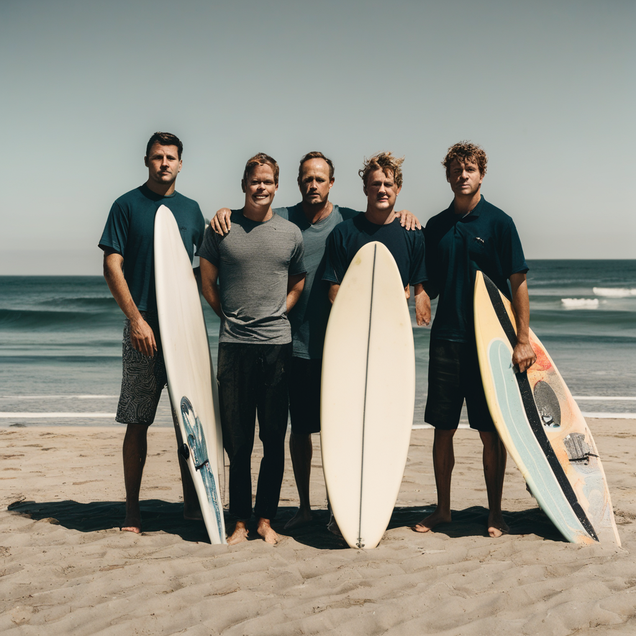}
        & \includegraphics[width=0.42\linewidth]{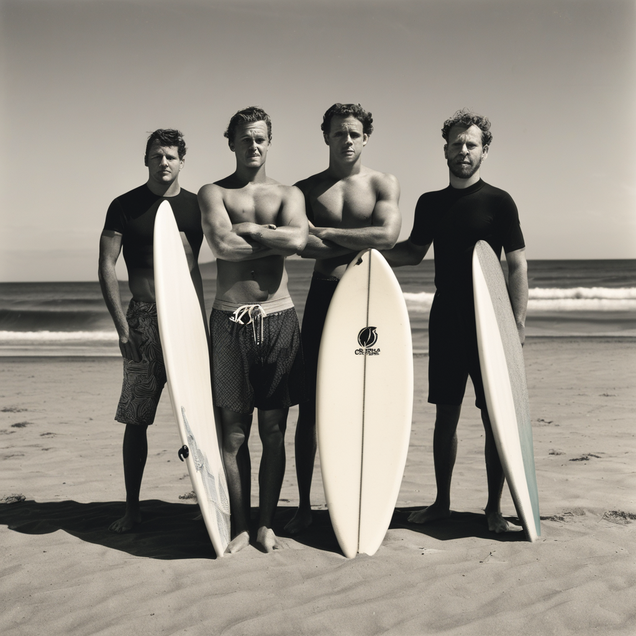} \\[-0.1em]
        & \multicolumn{2}{c}{\scriptsize \textbf{Prompt:} ``Four men with surf boards standing in the front of the beach side.''}
    \end{tabular}

    \caption{Additional qualitative results under high guidance on SDXL.}
    \label{fig:additional_high_guidance_sdxl_2}
\end{figure*}


\begin{figure*}[!t]
    \centering
    \setlength{\tabcolsep}{2pt}
    \renewcommand{\arraystretch}{1.0}
    \scriptsize

    \begin{tabular}{c cc}
        & (a) Ours & (b) CFG \\[0.25em]

        \multirow{6}{*}
        & \includegraphics[width=0.42\linewidth]{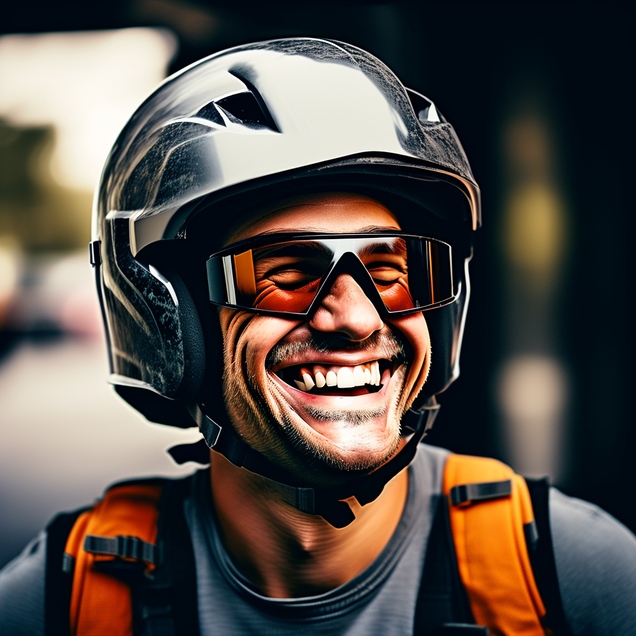}
        & \includegraphics[width=0.42\linewidth]{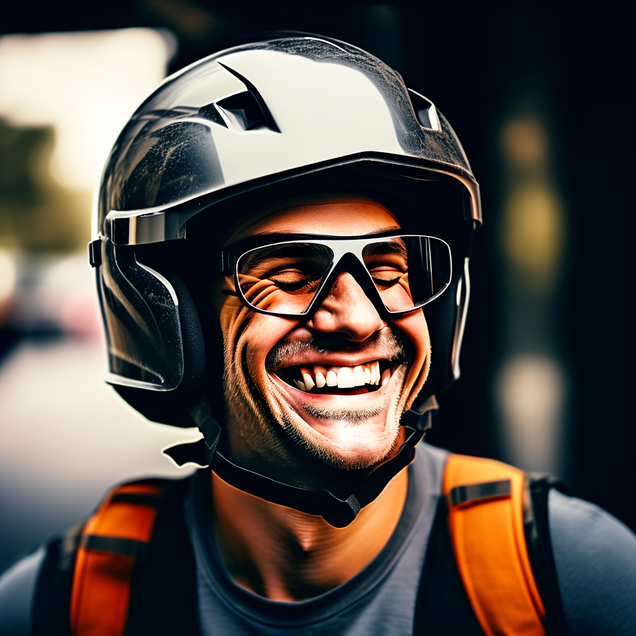} \\[-0.1em]
        & \multicolumn{2}{c}{\scriptsize \textbf{Prompt:} ``A man smiles wearing a helmet and goggles.''} \\[0.45em]

        & \includegraphics[width=0.42\linewidth]{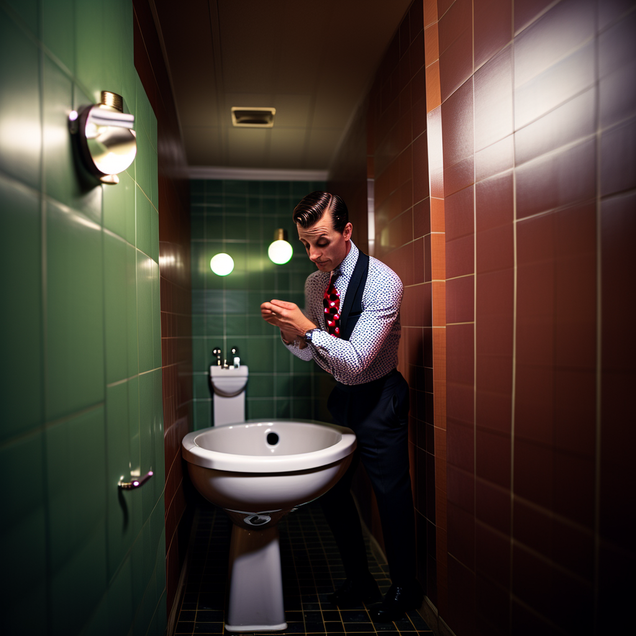}
        & \includegraphics[width=0.42\linewidth]{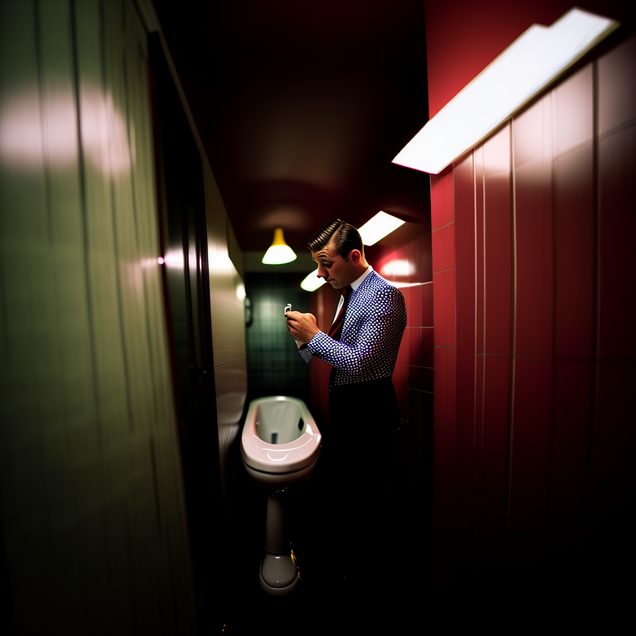} \\[-0.1em]
        & \multicolumn{2}{c}{\scriptsize \textbf{Prompt:} ``A man putting on a polka dot tie in the bathroom.''} \\[0.45em]

        & \includegraphics[width=0.42\linewidth]{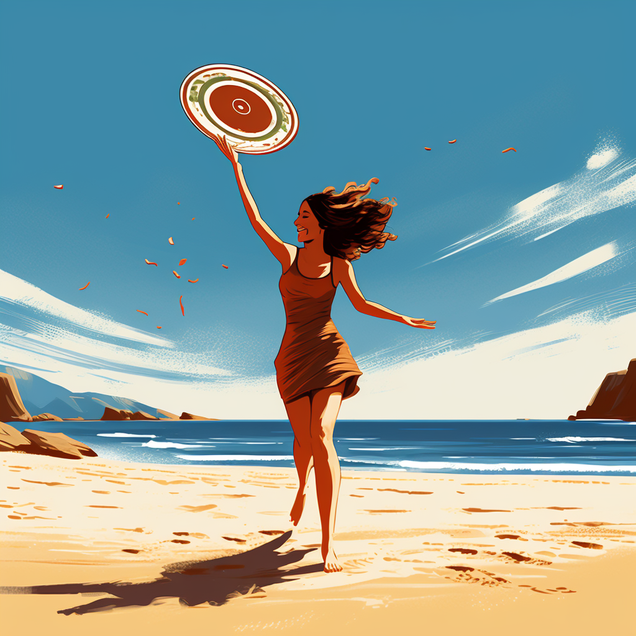}
        & \includegraphics[width=0.42\linewidth]{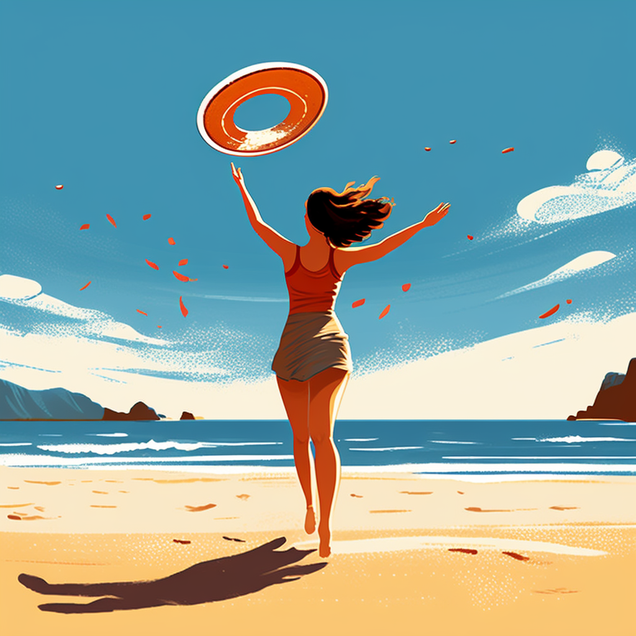} \\[-0.1em]
        & \multicolumn{2}{c}{\scriptsize \textbf{Prompt:} ``A woman is throwing a frisbee on the beach.''}
    \end{tabular}

    \caption{Additional qualitative results under high guidance on PixArt-$\alpha$.}
    \label{fig:additional_high_guidance_pixart}
\end{figure*}

\begin{figure*}[!t]
    \centering
    \setlength{\tabcolsep}{2pt}
    \renewcommand{\arraystretch}{1.0}
    \scriptsize

    \begin{tabular}{c cc}
        & (a) Ours & (b) CFG \\[0.25em]

        \multirow{6}{*}
        & \includegraphics[width=0.42\linewidth]{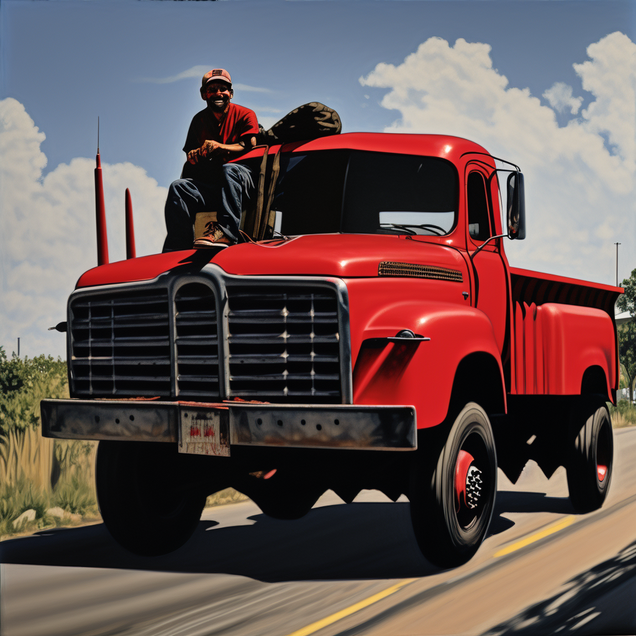}
        & \includegraphics[width=0.42\linewidth]{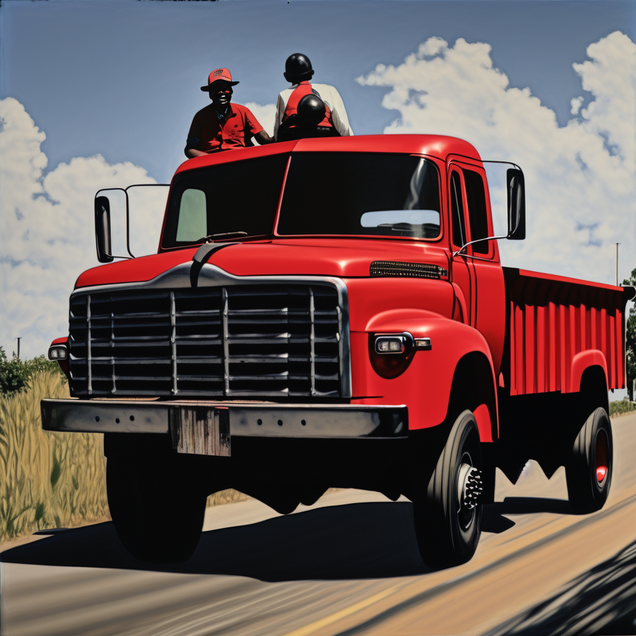} \\[-0.1em]
        & \multicolumn{2}{c}{\scriptsize \textbf{Prompt:} ``A red and black truck with a man riding on back of it.''} \\[0.45em]

        & \includegraphics[width=0.42\linewidth]{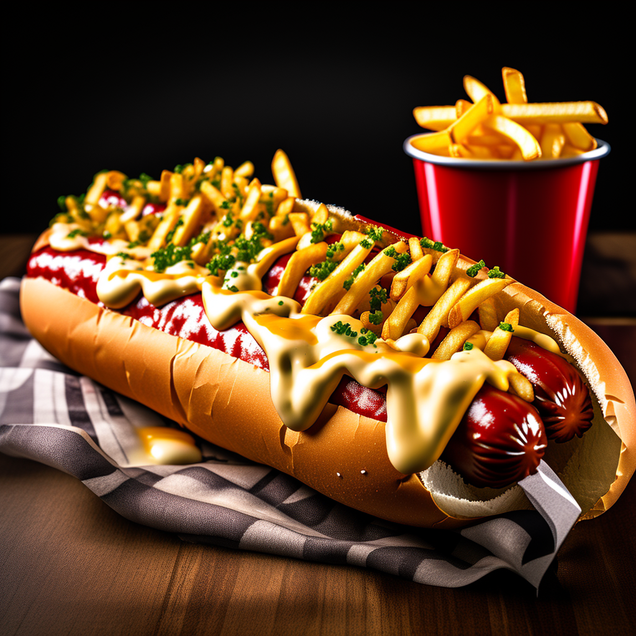}
        & \includegraphics[width=0.42\linewidth]{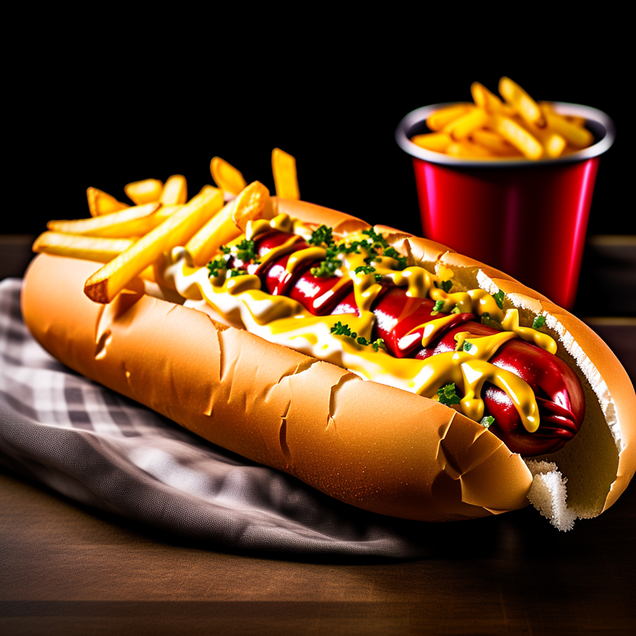} \\[-0.1em]
        & \multicolumn{2}{c}{\scriptsize \textbf{Prompt:} ``Hot dog Philly cheese style next to very large order of fries.''} \\[0.45em]

        & \includegraphics[width=0.42\linewidth]{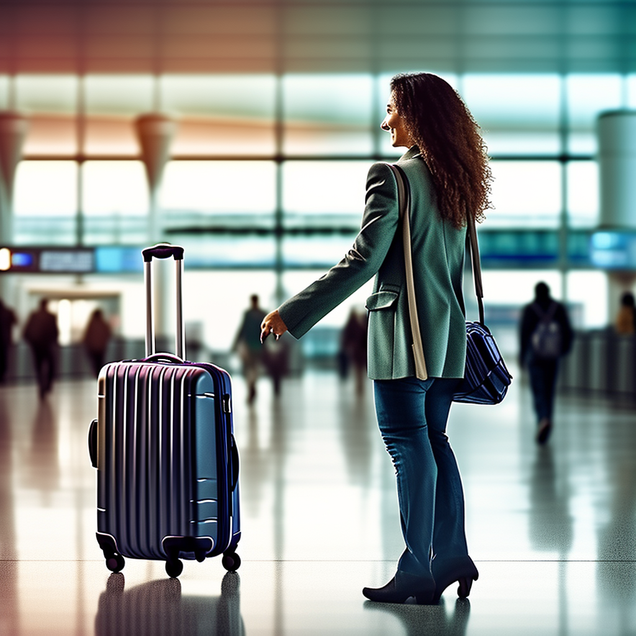}
        & \includegraphics[width=0.42\linewidth]{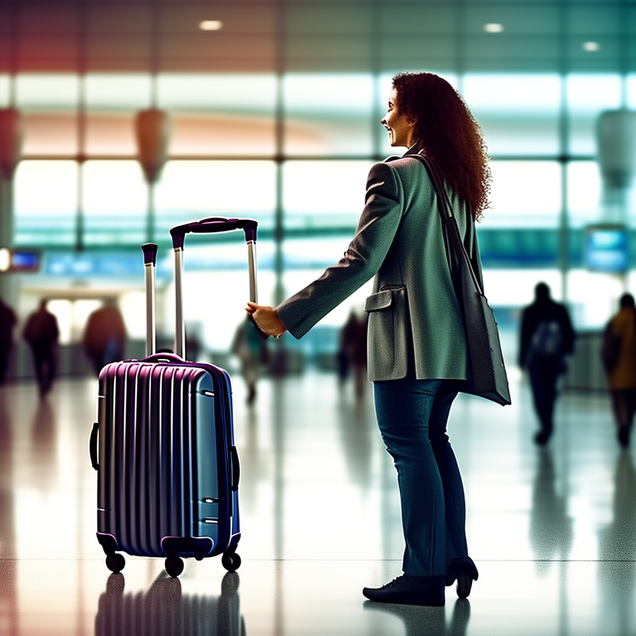} \\[-0.1em]
        & \multicolumn{2}{c}{\scriptsize \textbf{Prompt:} ``Woman grabbing her luggage at an airport terminal.''}
    \end{tabular}

    \caption{Additional qualitative results under high guidance on PixArt-$\alpha$.}
    \label{fig:additional_high_guidance_pixart_2}
\end{figure*}


\begin{figure*}[!t]
    \centering
    \setlength{\tabcolsep}{2pt}
    \renewcommand{\arraystretch}{1.0}
    \scriptsize

    \begin{tabular}{c cc}
        & (a) Ours & (b) CFG \\[0.25em]

        \multirow{6}{*}
        & \includegraphics[width=0.42\linewidth]{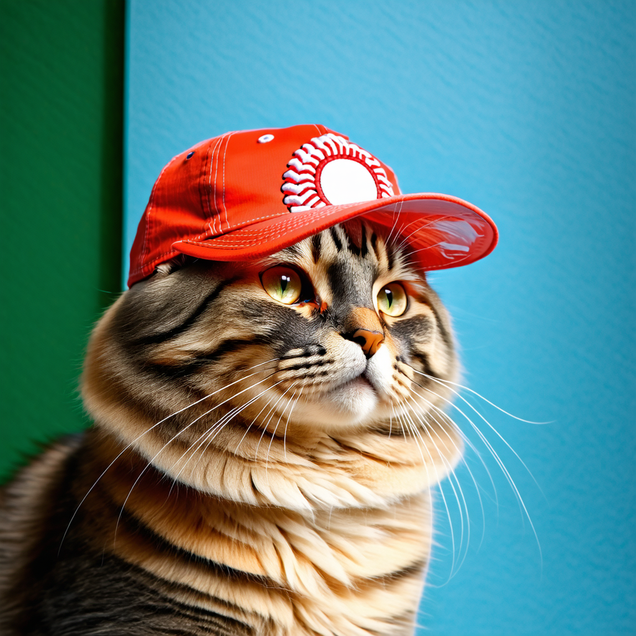}
        & \includegraphics[width=0.42\linewidth]{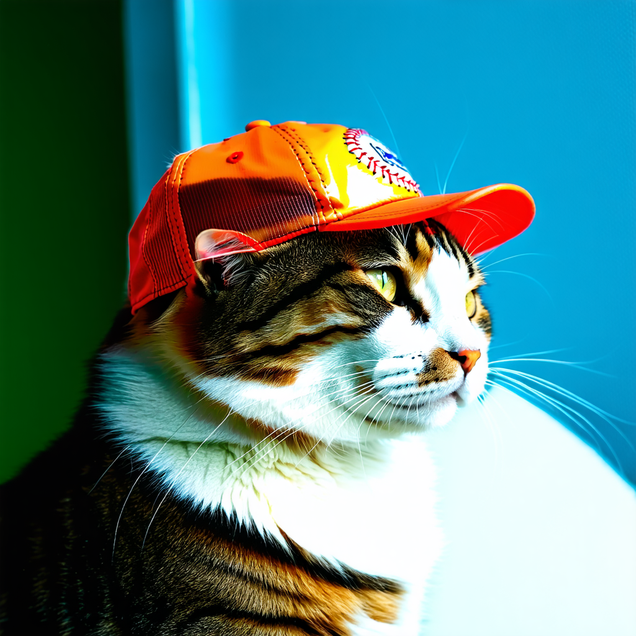} \\[-0.1em]
        & \multicolumn{2}{c}{\scriptsize \textbf{Prompt:} ``A cat is wearing a plastic baseball cap.''} \\[0.45em]

        & \includegraphics[width=0.42\linewidth]{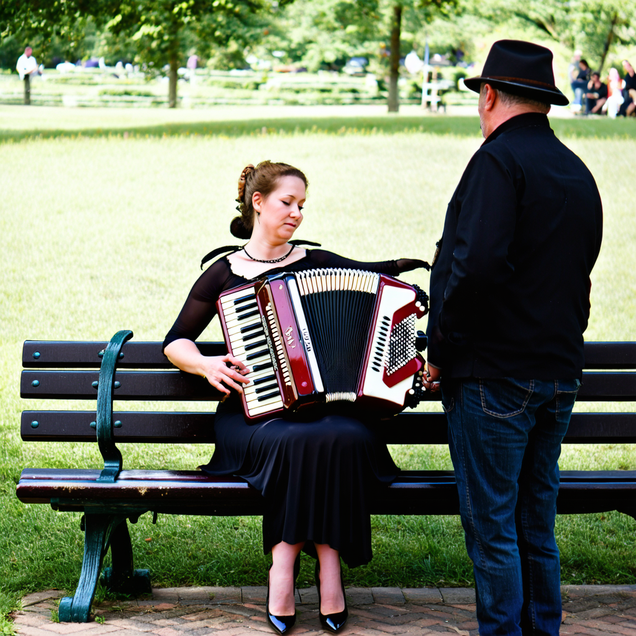}
        & \includegraphics[width=0.42\linewidth]{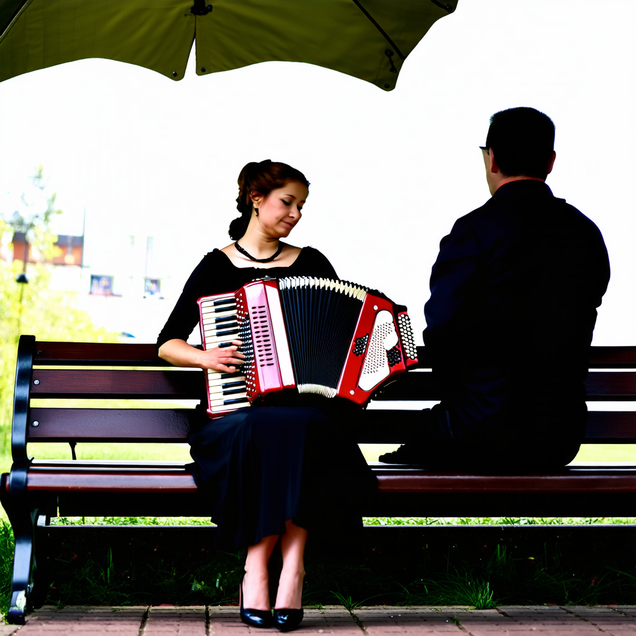} \\[-0.1em]
        & \multicolumn{2}{c}{\scriptsize \textbf{Prompt:} ``A woman on a bench plays an accordion as a man looks on.''} \\[0.45em]

        & \includegraphics[width=0.42\linewidth]{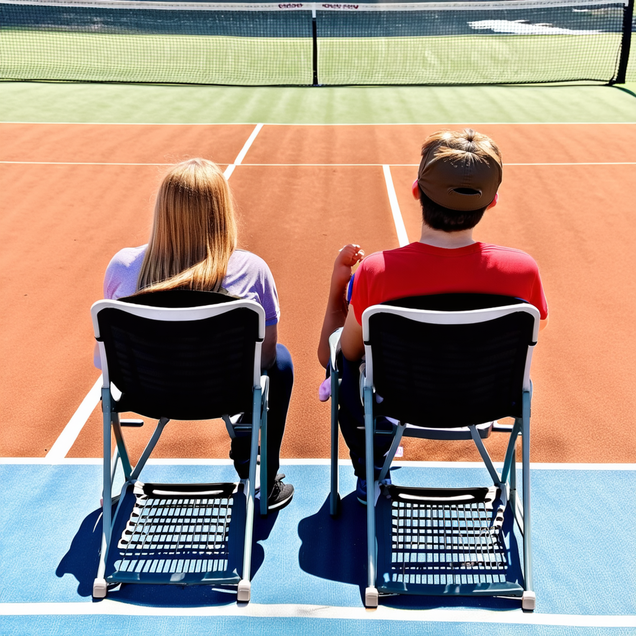}
        & \includegraphics[width=0.42\linewidth]{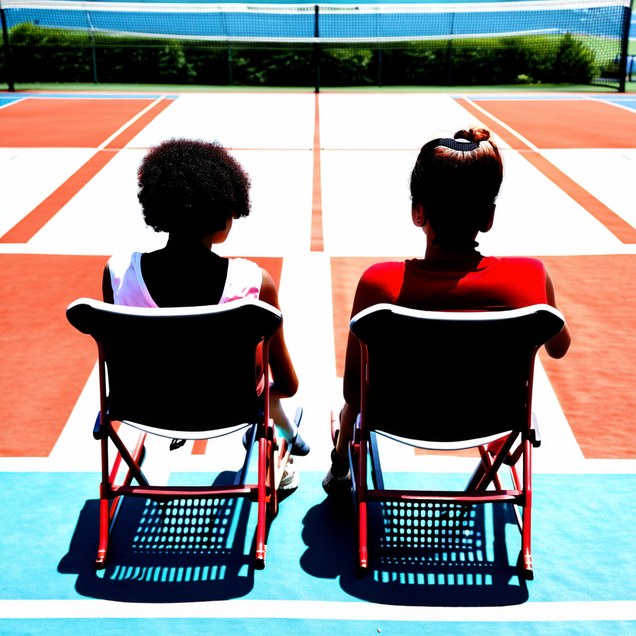} \\[-0.1em]
        & \multicolumn{2}{c}{\scriptsize \textbf{Prompt:} ``two people sitting next to each other on folding chairs on a tennis court''}
    \end{tabular}

    \caption{Additional qualitative results under high guidance on SD3.5.}
    \label{fig:additional_high_guidance_sd35}
\end{figure*}

\begin{figure*}[!t]
    \centering
    \setlength{\tabcolsep}{2pt}
    \renewcommand{\arraystretch}{1.0}
    \scriptsize

    \begin{tabular}{c cc}
        & (a) Ours & (b) CFG \\[0.25em]

        \multirow{6}{*}
        & \includegraphics[width=0.42\linewidth]{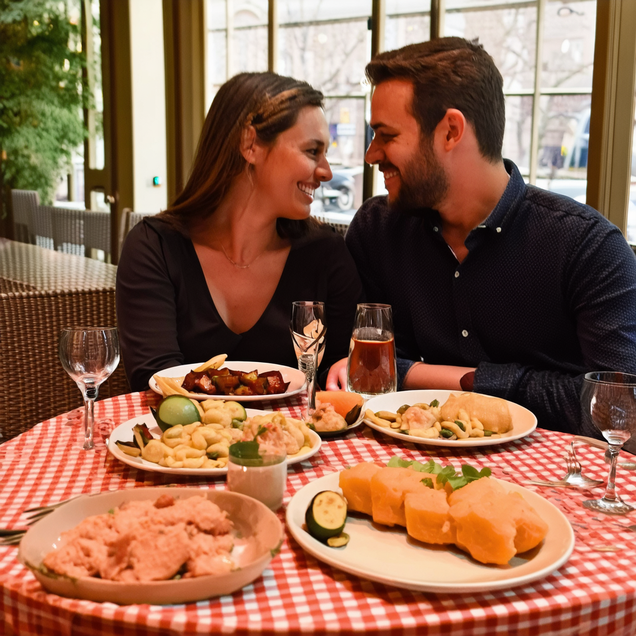}
        & \includegraphics[width=0.42\linewidth]{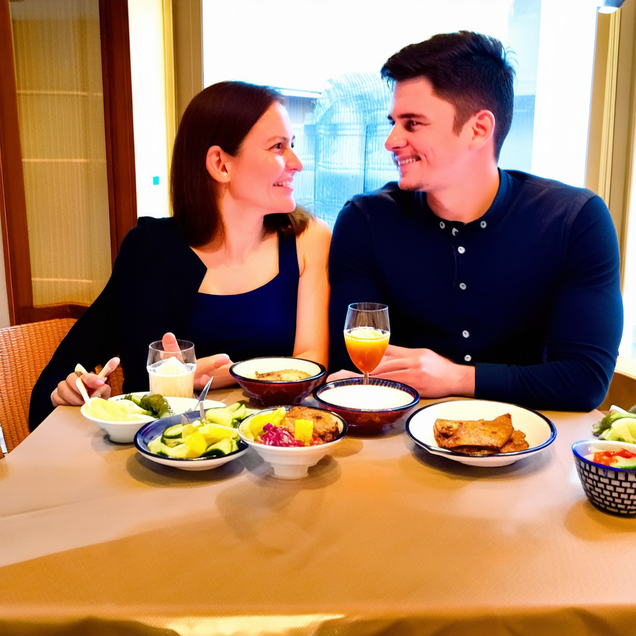} \\[-0.1em]
        & \multicolumn{2}{c}{\scriptsize \textbf{Prompt:} ``A couple is sitting at a table with food.''} \\[0.45em]

        & \includegraphics[width=0.42\linewidth]{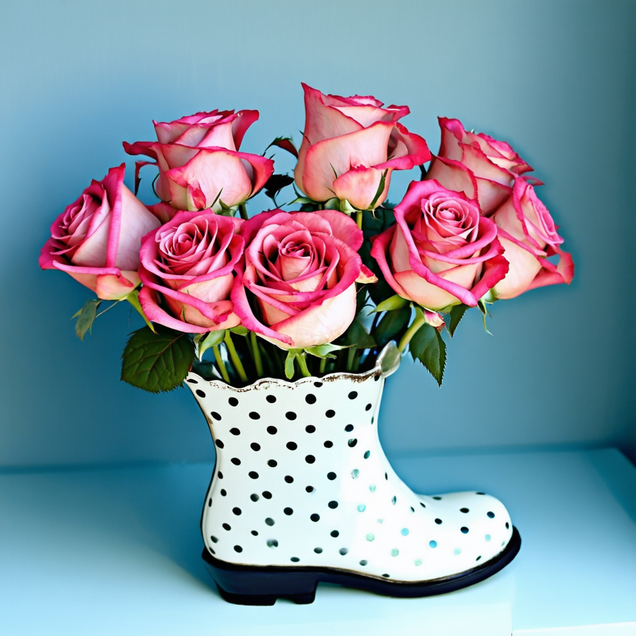}
        & \includegraphics[width=0.42\linewidth]{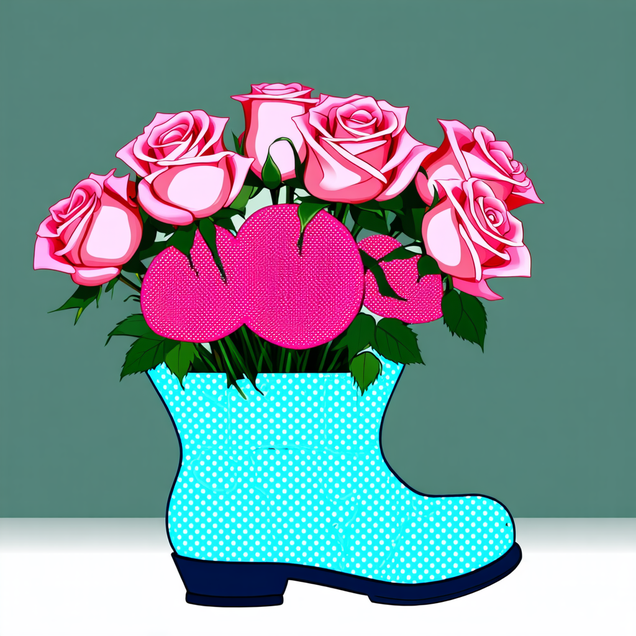} \\[-0.1em]
        & \multicolumn{2}{c}{\scriptsize \textbf{Prompt:} ``Pink roses in a little polka-dotted boot-shaped vase''} \\[0.45em]

        & \includegraphics[width=0.42\linewidth]{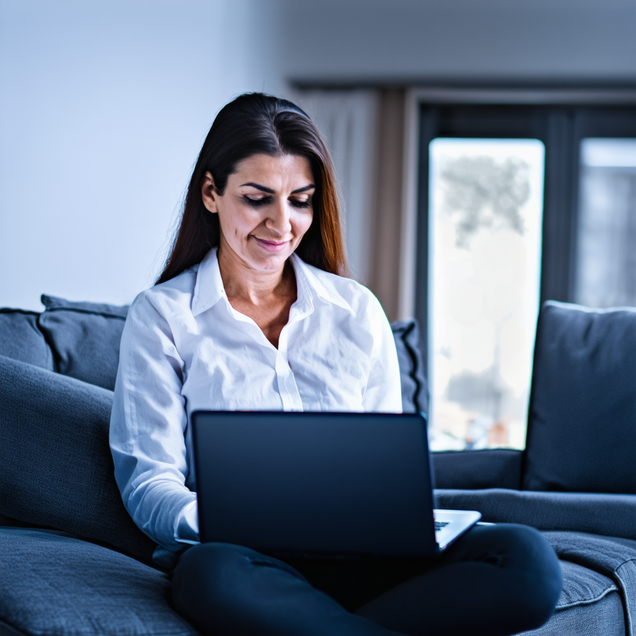}
        & \includegraphics[width=0.42\linewidth]{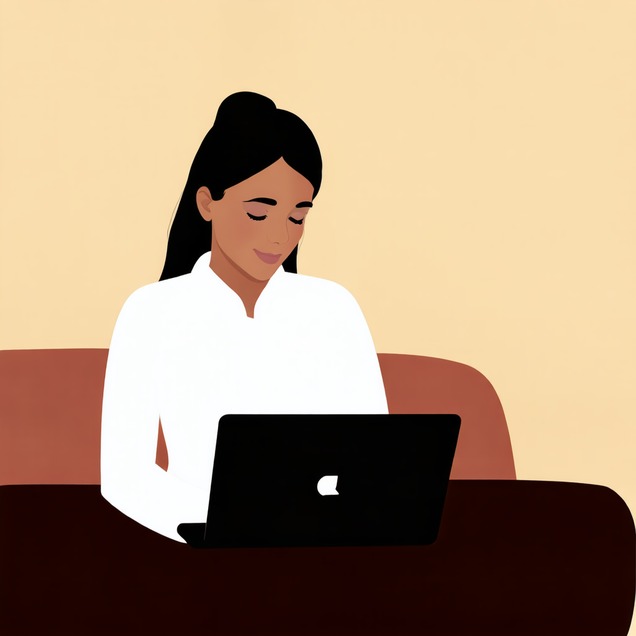} \\[-0.1em]
        & \multicolumn{2}{c}{\scriptsize \textbf{Prompt:} ``A woman in white shirt sitting at couch with a laptop.''}
    \end{tabular}

    \caption{Additional qualitative results under high guidance on SD3.5.}
    \label{fig:additional_high_guidance_sd35_2}
\end{figure*}

\end{document}